\documentclass{article}

\usepackage[nonatbib,preprint]{neurips_2020}

\usepackage[utf8]{inputenc} % allow utf-8 input
\usepackage[T1]{fontenc}    % use 8-bit T1 fonts
\usepackage{hyperref}       % hyperlinks
\usepackage{url}            % simple URL typesetting
\usepackage{booktabs}       % professional-quality tables
\usepackage{amsfonts}       % blackboard math symbols
\usepackage{nicefrac}       % compact symbols for 1/2, etc.
\usepackage{microtype}      % microtypography
\usepackage{comment}
\usepackage{thmtools}
\usepackage{thm-restate}
\usepackage{amsmath,amsthm, amssymb, latexsym,graphicx}
\usepackage[linesnumbered,lined,ruled,noend]{algorithm2e}
\usepackage{papercommands}
\usepackage{multicol, multirow, tabularx}
\usepackage{subcaption}
\usepackage{float}
\usepackage{mathtools}
\usepackage{longtable}
\usepackage{hyperref}
\usepackage{dcolumn}
\usepackage[inline]{enumitem}
\usepackage{thmtools}
\usepackage{thm-restate}
\usepackage{diagbox}
\usepackage{adjustbox}

\newcolumntype{2}{D{.}{}{2.0}}

\newcolumntype{L}{>{\centering\arraybackslash}m{3cm}}
\date{}

\author{
  Murad Tukan \\
  \texttt{muradtuk@gmail.com}
  \And
  Alaa Maalouf \\
  \texttt{alaamalouf12@gmail.com}
  \And
  Dan Feldman \\
  \texttt{dannyf.post@gmail.com}
  \AND
  \normalfont{The Robotics and Big Data Lab,}\\Department of Computer Science,\\University of Haifa,\\Haifa, Israel
}

\allowdisplaybreaks

\newcounter{myCounter}

\begin{document}

\title{Coresets for Near-Convex Functions}
\maketitle
\begin{abstract}

Coreset is usually a small weighted subset of $n$ input points in $\mathbb{R}^d$, that provably approximates their loss function for a given set of queries (models, classifiers, etc.). Coresets become increasingly common in machine learning since existing heuristics or inefficient algorithms may be improved by running them possibly many times on the small coreset that can be maintained for streaming distributed data. Coresets can be obtained by sensitivity (importance) sampling, where its size  is proportional to the total sum of sensitivities. Unfortunately, computing the sensitivity of each point is problem dependent and may be harder to compute than the original optimization problem at hand. 

We suggest a generic framework for computing sensitivities (and thus coresets) for wide family of loss functions which we call near-convex functions.
This is by suggesting the $f$-SVD factorization that generalizes the SVD factorization of matrices to functions. Example applications include coresets that are either new or significantly improves previous results, such as SVM, Logistic regression, M-estimators, and $
\ell_z$-regression. Experimental results and open source are also provided.

\label{abstract}

\end{abstract} 
\section{Introduction} \label{introduction}
In common machine learning problems, we are given a set of input points $P \subseteq \REAL^d$ (training data), and a loss function $f:P\times \Q \to [0, \infty)$, where the goal is to solve the optimization problem of finding a query (model, classifiers, centers) $x^*$ that minimizes the sum of fitting errors $\sum_{p\in P} f(p,x)$ over every query $x$ in a given (usually infinite) set. For example, in $k$-median (or $k$-mean) clustering, each query is a set of $k$ centers and the loss function is the distance (or squared distance) of a point to its nearest center. In linear regression or SVM, every input point includes a label, and the loss function is the fitting error between the classification of $p$ via a given query to the actual label of $p$. Empirical risk minimization (ERM) may be used to generalize the result from train to test data.

%\comment{from a finite or inifnite set}

%that minimizes the total loss $\sum_{p\in P} f(p,x)$ over every $x \in \REAL^d$.
%In machine learning, training a model requires solving an optimization problem
%In machine learning, optimization problems play a crucial role for both theoretical and practical aspects, as every trainable machine learning model boils down to solving an optimization problem.

%With the rise of availability of (and need for) big data-sets, consisting of hundreds of millions of points, existing algorithms that were mainly designed to solve such optimization problems face difficulties, for instance, insufficient storage, high computational complexity, etc. Thus the challenge becomes designing new algorithms with better running time, storage or other features. Among common solutions are
%%With the rise of need and availability of big data-sets, consisting of hundreds of millions of points, traditional approaches of algorithm designing, that were mainly tailored towards offline datasets (i.e., the data-set is stored locally on the machine), had failed to capture and achieve their purposes
%%. Such issues would be dealt with by either
%\begin{enumerate*}[label=(\roman*)]
%    \item extending the machine's computational capabilities and its memory, e.g., more GPUs, CPUs, etc.,
%    \item redesigning existing algorithms to cope with parallelism and such task is not always easy to grasp, or simply
%    \item applying ad hoc heuristics that work in practice yet has no theoretical guarantees.
%\end{enumerate*}
%
%
%\subsection{Coresets}
\paragraph{Modern machine learning. }In practice, many of these optimization or learning problems are usually hard even to approximate. Instead, practical heuristics with no provable guarantees may be used to solve them. Even for well understood problems, which have close optimal solution, such as linear regression or classes of convex optimization, in the era of big data we may wish to maintain the solution in other computation models such as: streaming input data (``on-the-fly") that provably uses small memory, parallel computations on distributed data (on the cloud, network or GPUs) as well as deletion of points, constrained optimization (e.g. sparse classifiers). Cross validation~\cite{hastie2009elements} or hyper-parameter tuning techniques such as AutoML~\cite{feurer2015efficient,guyon2015design} need to evaluate many queries for different subsets of the data, and different constraints.
%With the rise of availability of (and need for) big data-sets, consisting of hundreds of millions of points, existing algorithms that were mainly designed to solve such optimization problems face difficulties, for instance, insufficient storage, high computational complexity, etc. Thus the challenge becomes designing new algorithms with better running time, storage or other features. Among common solutions are

\paragraph{Coresets. }One approach is to redesign existing machine learning algorithms for faster, approximate solutions and these new computation models. A different approach that is to use data summarization techniques. \emph{Coresets} in particular were first used to solve problems in computational geometry~\cite{AHV04} and got increasing attention in both the industry~\cite{bachem2018scalable,pmlr-v84-bachem18a,buadoiu2008optimal,curtain2019coresets,karnin2019discrepancy} and academy~\cite{balcan2013distributed, pmlr-v97-braverman19a,feldman2010coresets, feldman2014coresets} over the recent years; see surveys in~\cite{feldman2020core,munteanu2018coresets,phillips2016coresets}. Informally, coreset is a small weighted subset of the input points (unlike e.g. sketches, or dimension-reduction techniques) that approximates the loss of the input set $P$ for \emph{every} feasible query $x$, up to a provable bound of $1\pm\eps$ for a given error parameter $\eps\in(0,1)$. The size of the coreset is usually polynomial in $1/\eps$ but independent or near-logarithmic in the size of the input. Since such a coreset approximates every query (and not just the optimal one), it supports constraint optimization, and the above computation models using merge-and-reduce trees; see details in~\cite{feldman2020core}. Moreover, coresets may be computed in time that is near-linear in the input, even for NP-hard optimization problems. Existing heuristic or inefficient algorithms may then be applied many times on the small coreset to obtain improved or faster models in such cases.

Example coresets in machine learning include~\emph{SVM}~\cite{har2007maximum,tsang2006generalized,tsang2005core,tsang2005very,tukan2020coresets}, $\ell_z$-regression~\cite{cohen2015lp, dasgupta2009sampling, sohler2011subspace}, clustering~\cite{bachem2015coresets, cohen2015dimensionality, feldman2011scalable,gu2012coreset,jubran2020sets,lucic2015strong, schmidt2019fair}, logistic regression~\cite{huggins2016coresets,munteanu2018coresets}, LMS solvers and \emph{SVD}~\cite{feldman2013turning, maalouf2019fast,maalouf2019tight, sarlos2006improved},
where all of these works have been dedicated to suggest a coreset for a specific problem.

%A common technique, both in theory and practice, that yields fast and provably good coresets is the approach of non-uniform sampling, sensitivity sampling which was first suggest by~\cite{langberg2010universal}.

\paragraph{A generic framework} for constructing coresets was suggested in~\cite{feldman2011unified, langberg2010universal}. It states that, with high probability, non-uniform sampling from the input set yields a coreset. Each point should be sampled i.i.d. with a probability that is proportional to its importance or sensitivity, and assigned a multiplicative weight which is inverse proportional to this probability, so that the expected original sum of losses over all the points will be preserved. Here, the sensitivity of an input point $p\in P$ is defined to be the maximum of its relative fitting loss $s(p)=f(p,x)/\sum_{q\in P}f(q,x)$ over every possible query $x$. The size of the coreset is near-linear in the total (sum) $t$ of these sensitivities; see Theorem~\ref{thm:coreset} for details. It turns out in the recent years that many classical and hard machine learning problems~\cite{braverman2016new, lucic2017training, sohler2018strong} have total sensitivity that is near-logarithmic or independent of the input size $\abs{P}$ which implies small coresets via sensitivity sampling.

\paragraph{Paper per problem.} The main disadvantage of this framework is that the sensitivity $s(p)$, as defined above, is problem dependent: namely on the loss function $f$ and the feasible set of queries. Moreover, maximizing $s(p)=f(p,x)/\sum_{q\in P}f(q,x)$ is equivalent to minimizing the inverse $\sum_{q\in P}f(q,x)/f(p,x)$. Unfortunately, minimizing the enumerator is usually the original optimization problem which motivated the coreset in the first place. The denominator may make the problem harder, in addition to the fact that now we need to solve this optimization problem for each and every input point in $P$. While approximations of the sensitivities usually suffice, sophisticated and different approximation techniques are frequently tailored in papers of recent machine learning conferences for each and every problem.

%The size of the (sampled) coreset in~\cite{feldman2011unified} depends quadratically on the sum $t$ of these sensitivity bounds, called total sensitivity bound. Later on,~\cite{braverman2016new} reduced the size of the coreset that are constructed by this framework to be only near-linear ($t\log{t}$) in the total sensitivity bound.

\subsection{Problem Statement}
To this end, the goal of this paper is to suggest a framework for sensitivity bounding of a \emph{family} of functions, and not for a specific optimization problem. This approach is inspired by convex optimization: while we do not have a single algorithm to solve any convex optimization, we do have generic solutions for family of convex functions. E.g., linear programming, Semi-definite programming, and so on.

%
% (which implies a framework for coreset construction using ~\cite{braverman2016new}) to a wide family of problems (see Lemma~\ref{thm:sensitivityBound}), which is of interest since many machine learning models rely on optimizing such functions, e.g., support vector machines, logistic regression, linear regression, $\ell_z$-regression for any $z \in (0, \infty)$, and outlier resistant functions. We call this family, the family of \emph{near-convex functions}. For simplicity, the following Definition gives a simplified version of this family. However, in the appendix we show a more generalized version that handles a bigger family of functions; See Definition~\ref{def:familyOfConvexFuncs_Gen} at the Appendix.

We choose the following family of near-convex loss functions, with example supervised and unsupervised applications that include support vector machines, logistic regression, $\ell_z$-regression for any $z \in (0, \infty)$, and functions that are robust to outliers. In the Supplementary Material we suggest a more generalized version that handles a bigger family of functions; see Definition~\ref{def:familyOfConvexFuncs_Gen}, and hope that this paper will inspire the research of more and larger families.

\begin{Definition}[Near-convex functions]
\label{def:familyOfConvexFuncs}
Let $P \subseteq \REAL^d$ be a set of $n$ points, and let $f : P \times \Q \to [0,\infty)$ be a loss function. We call $f$ a near-convex function if there are a convex loss function $g : P \times \Q \to [0, \infty)$ (see Definition~\ref{def:conv} at Supplementary Material), a function $h : P \times \Q \to [0,\infty)$, and a scalar $z > 0$ satisfying:
\begin{enumerate}[label=(\roman*), nolistsep]
\item \label{prop:reductionToFam} There exist $c_1,c_2 > 0$ such that for every $p \in P$, and $x \in \Q$,
\[
c_1 \term{g(p,x)^z + h(p,x)^z} \leq f(p,x) \leq c_2 \term{ g(p,x)^z+ h(p,x)^z}.
\]
\item \label{prop:def} For every $p \in P$, $x \in \Q$ and $b > 0$, we have $
g(p,bx) = b \cdot g(p,x)$.
\item \label{prop:h_trait} For every $p \in P$ and $x \in \Q$, we have $
\frac{h(p,x)^z}{\sum_{q \in P} h(p,x)^z} \leq \frac{2}{n}$.
\item \label{prop:defLvl} The set $\mathcal{X}_g = \br{x \in \Q \middle| \sum_{p \in P } g(p,x)^{\maxArgs{1,z}} \leq 1}$ is centrally symmetric, i.e., for every $x \in \mathcal{X}_g$ we have $-x \in \mathcal{X}_g$, and there exist $R,r \in (0,\infty)$ such that
$
B(0_d,r) \subset \mathcal{X}_g \subset B(0_d,R),
$
where $B(0_d,y)$ denotes a ball of radius $y > 0$, centered at $0_d$.
\end{enumerate}
We denote by $\set{F}$, the union of all functions $f$ with the above properties.
\end{Definition}

We are interested in a generic algorithm that would get a set of input points, and a loss function as above, and compute a sensitivity for each point, based on the parameters of the given loss function. In addition, we wish to use worst-case analysis and prove that for every input the total sensitivity (and thus size of coreset) would be small, depending on the ``hardness" of the loss function that is encapsulated in the above parameters $z, R$, etc.

\section{Related Work}
\label{sec:applications}

%In what follows, we give a set of functions which are widely used in the literature of machine learning, that are covered by our framework.
\paragraph{Logistic Regression.}
A coreset construction algorithms for the problem of logistic regression were suggested by~\cite{huggins2016coresets},~\cite{ tolochinsky2018generic}, and~\cite{munteanu2018coresets}. All of these works handled variations of the problem, e.g., they all lack the incorporation of the bias term (intercept) in their loss function.
Specifically speaking, both \cite{huggins2016coresets} and \cite{munteanu2018coresets} didn't account for the regularization term and its parameter. Furthermore, the coreset's size established by~\cite{munteanu2018coresets}, was dependant on the structure of the input data. As for~\cite{tolochinsky2018generic}, the coreset only succeed for a small subset of queries (a ball in $\REAL^d$ of radius $r$, where the coreset's size is near linear in $r$).
Contrary to previous works, our coreset approximates the logistic regression loss function including the bias parameter (intercept) and the regularization term for every possible query. This is the loss function that is usually used in practice, e.g., see Sklearn library in~\cite{scikit-learn}. Finally, our coreset's size is independent of the structure of the data.

% Our coreset approximates the SVM formulation which relies on the hinge loss, that is used widely in practice; See Sklearn library at Pedregosa et al. (2011). Both related papers, use the squared hinge loss which is also used in practice, to enforce the SVM cost function to be strongly convex.

\paragraph{SVM.} \cite{clarkson2010coresets, tsang2006generalized, tsang2005core} addressed the problem of coreset construction for SVM, yet they used squared hinge loss to enforce the SVM cost function to be strongly convex. At \cite{tukan2020coresets}, the coreset is constructed with respect to the hinge loss which most used form of SVM in practice (see Sklearn library at~\cite{scikit-learn}). However for the coreset to be constructed, a (sub-)optimal solution was required for the problem itself. In addition, the coreset size depended heavily on on the ratio between the variance of each class of points.
In this paper, we also address a coreset with respect to the hinge loss, yet we don't require any (sub-)optimal solution to construct the coreset, and our coreset's size depends on the ratio between the number of points of each class (see Corollary~\ref{lem:SVM}).

\paragraph{$\ell_z$-Regression.} A notable line of work~\cite{clarkson2005subgradient, cohen2015lp, dasgupta2009sampling, sohler2011subspace, woodruff2013subspace} addressed the construction of coresets and sketches in this area, however, all such papers addressed the case of $z \geq 1$. Most of these works used tools similar to the well-conditioned basis which was first suggested at~\cite{dasgupta2009sampling} to compute such coresets. Intuitively it can be thought of as a generalization of the \emph{SVD} factorization of an input set with respect to the loss function of $\ell_z$-regression for any $z \geq 1$. In our framework we generalize this factorization in order to compute coresets for the near-convex functions. To our knowledge, we suggest the first coreset for the problem of $\ell_z$-regression for any $z \in (0,1)$.% and our $f$-SVD is a generalization of it.

\paragraph{Outlier resistant functions. (similar to $M$-estimators)} To our knowledge, we present the first coreset for such problem; see Corollary~\ref{thm:restrictedLp}.

\section{Our contribution}
% \textbf{Will be added to the introduction as a subsection.}
In this paper, we suggest an $\eps$-coreset construction algorithm with respect to any near-convex function. Specifically speaking, we provide:
\begin{enumerate}[label=(\roman*)]
  \item A generalization of the well conditioned bases of~\cite{dasgupta2009sampling} to a broader family of functions, i.e., not just for $\ell_z$-Regression problems where $z \geq 1$. This informally describes a factorization of the input data with respect to a given near-convex loss function. We call such factorization the $f$-SVD of $P$ (see Definition~\ref{def:svdWRTF}).
  \item A framework for bounding the sensitivity of each point in an input set with respect to any near-convex function. The heart of the framework relies on computing the $f$-SVD factorization described in~(i); see Lemma~\ref{thm:sensitivityBound} and Algorithm~\ref{alg:mainAlg}.
  \item By (ii), we provide the first $\eps$-coreset for the problem of $\ell_z$-regression where $z \in (0,1)$, and the first $\eps$-coreset for certain outlier resistant functions. We also generalize existing works of coreset construction for the problems of logistic regression and SVM; see Section~\ref{sec:analysis}.

  \item Experimental results on real-world and synthetic datasets for common machine learning solvers (supported by our framework) of Scikit-learn library~\cite{scikit-learn}, assessing the practicability and efficacy of our algorithm.
  \item An open source code implementation of our algorithm, for reproducing our results and future research.
\end{enumerate}

%  \begin{enumerate}[label=(\roman*)]
%  \item A generalization of the well conditioned bases of~\cite{dasgupta2009sampling} to a broader family of functions, i.e., not just for $\ell_z$-Regression problems where $z \geq 1$. This informally describes a factorization of the input data with respect to a given general bi-Lipschitz loss function. We call such factorization the $f$-SVD of $P$ (see Definition~\ref{def:svdWRTF}).
%  \item A framework for bounding the sensitivity of each point in an input set with respect to any generalized bi-Lipschitz function. The heart of the framework relies on computing the $f$-SVD factorization described in~(i); See Lemma~\ref{thm:sensitivityBound} and Algorithm~\ref{alg:mainAlg}.
%  \item By (ii), we provide the first $\eps$-coreset for the problem of $\ell_z$-regression where $z \in (0,1)$, and the first $\eps$-coreset for certain outlier resistant functions. We also generalize existing works of coreset construction for the problems of logistic regression and SVM.
%  \item Experimental results on real-world and synthetic datasets for common machine learning solvers (supported by our framework) of Scikit-learn library~\cite{scikit-learn}, assessing the practicability and efficacy of our algorithm.
%  \item An open source code implementation of our algorithm, for reproducing our results and future research.
%  \end{enumerate}

\subsection{Novelty}

%The $f$-SVD factorization is a generalization of the \emph{well-conditioned bases} of~\cite{dasgupta2009sampling}.
\paragraph{$f$-SVD factorization.}
In this work, we suggest a novel factorization technique of an input dataset with respect to a specific loss function $f$, we call it the $f$-SVD factorization. Roughly speaking, the heart of the $f$-SVD factorization lies in finding a diagonal matrix $D \in [0,\infty)^{d \times d}$ and an orthogonal matrix $V \in \REAL^{d \times d}$ such that the total loss $\sum_{p \in P} f(p,x)$ for any query $x \in \REAL^d$ can be bounded from above by $\sqrt{d} \norm{DV^Tx}_2$ and from below by $\norm{DV^Tx}_2$.
In some sense, this can be thought of as a $\left(1-1/\sqrt{d}\right)$-coreset (or a sketch) since it approximates the total loss for any query in $\REAL^d$ up to a multiplicative factor of $\left(1-1/\sqrt{d}\right)$.
In order to obtain such factorization, we forge a link between the L\"{o}wner ellipsoid~\cite{john2014extremum} and the properties of near-convex functions; see Fig.~\ref{fig:into} for a detailed illustrative explanation, Definition~\ref{def:svdWRTF} and Lemma~\ref{lem:LownerToF} for the formal details.

Note that \emph{SVD} factorization is a special case of $f$-SVD due to that fact that \emph{SVD} handles functions of the form $\sqrt{\sum_{p \in P} \abs{p^Tx}^2}$ and attempts to achieve the same purpose. The $f$-SVD factorization is a generalization of the \emph{well-conditioned bases} of~\cite{dasgupta2009sampling}.

%Our sensitivity bounding algorithm is based on finding a factorization of the input set of points $P$ with respect to generalized bi-Lipschitz loss function $f \in \set{F}$.
\paragraph{From $f$-SVD to sensitivity bounds.}
With the lower bound on the total loss that is guaranteed by the $f$-SVD, we show how to bound the sensitivity of each point in the dataset. On the other hand, the upper bound on the total loss provided by the $f$-SVD factorization, helps us in bounding the total sensitivity. Having this being said, we use the $f$-SVD factorization to suggest a sensitivity bounding framework for a set of points with respect to any generalized bi-Lipshcitz function $f\in \set{F}$; see Lemma~\ref{thm:sensitivityBound}.

% which in turn will be roughly polynomial in $d$ resulting from the approximation achieved by the $f$-SVD factorization

%As for the upper bound, we show how to bound the total sensitivity. Since the $f$-SVD can be computed

%The $f$-SVD factorization provides us with upper and lower bounds for the total loss, where the

%To bound the sensitivity of each point $p \in P$, we first bound $f$ by a pair of functions $g,h : P \times \REAL^d \to [0, \infty)$ satisfying Definition~\ref{def:familyOfConvexFuncs} with respect to $f$. We then use the $f$-SVD factorization of $P$ with respect to $f$ to lower bound the denominator in the sensitivity term for any query $x \in \REAL^d$, i.e., $\sum_{p \in P} f(p,x)$. Under reasonable assumption, we then can bound the sensitivity of every $p \in P$. As for the total sensitivity, we need the upper bound that is provided by the $f$-SVD factorization. Our $f$-SVD is a generalization of~\cite{dasgupta2009sampling} towards a broader family of functions which we call the generalized bi-Lipschitz loss functions.

\begin{figure}[b!]
\centering
\includegraphics[width=\textwidth, height=4cm]{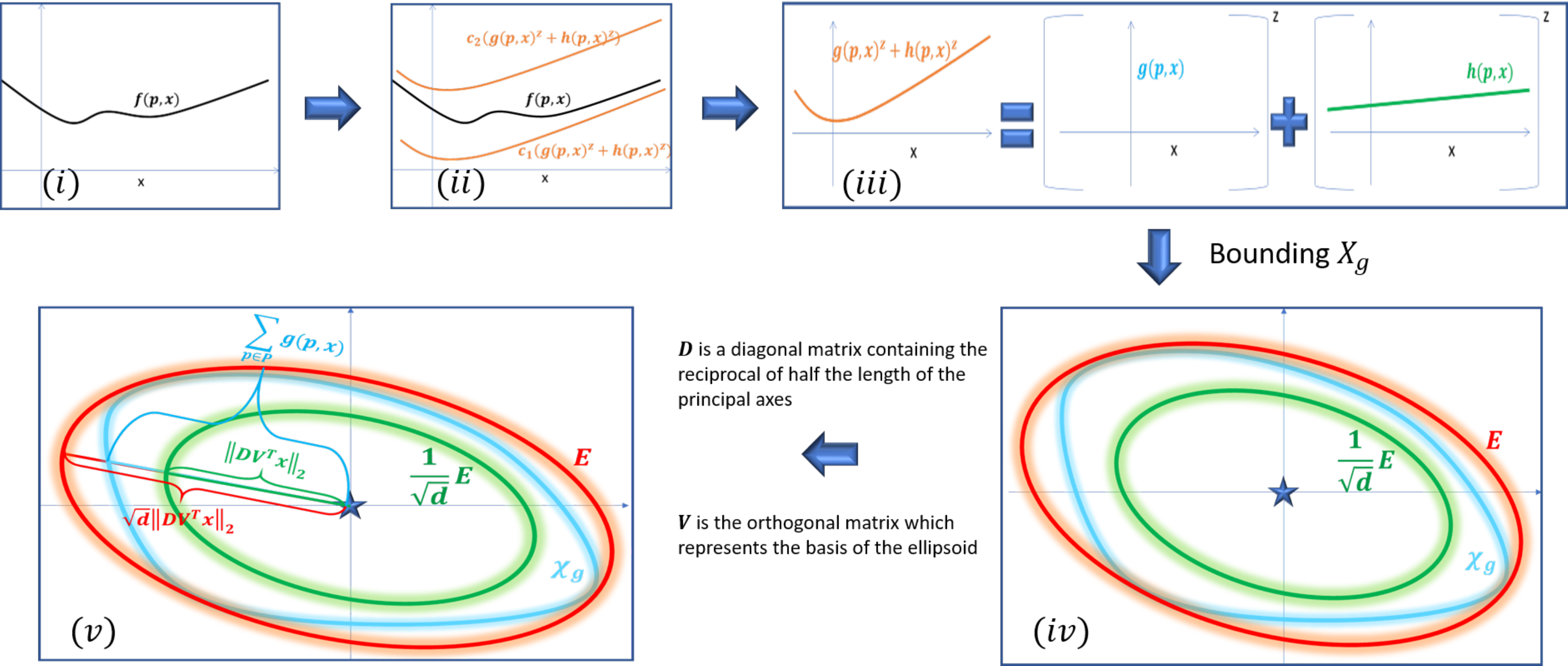}
\caption{\textbf{How to compute $f$-SVD}: (i) Given a set $P \subseteq \REAL^2$, and a function $f : P \times \REAL^2 \to [0, \infty)$, (ii) find a function which can bound $f(p, \cdot) \times \REAL^2 \to [0, \infty)$ from above and below for every $p \in P$, (iii) decompose this function into two functions $g$ and $h$ where for every $p,q \in P$ and $x \in \REAL^2$, $g(p, \cdot)$ is a convex function (e.g., $g(p,x) = |p^Tx|^4$), and $h(p,x) \approx h(q,x)$ (e.g., $h(p,x) = \norm{x}_1 + 10$), here $z = 1$. (iv) Since $g$ is convex, we find the L\"{o}wner ellipsoid $E$ which contains $\set{X}_g = \{ x \in \REAL^2 | \sum_{p \in P} g(p,x) \leq 1 \}$, and the dilated ellipsoid $1/\sqrt{d} E$ is inscribed in $\set{X}_g$.
Following this, we compute a diagonal matrix $D \in [0,\infty)^{2 \times 2}$ and an orthogonal matrix $V \in \REAL^{2 \times 2}$ such that $E = \{ x \in \REAL^2 | \norm{DV^Tx}_2 \leq 1 \}$. (v) By properties of the L\"{o}wener ellipsoid, we show that  the total loss of $g$ (cyan line) for any query $x \in \REAL^2$ is in the range  $[\norm{DV^Tx}_2, \sqrt{d} \norm{DV^Tx}_2]$ (green and red lines). When combined with the bounds on $f$, we obtain an upper bound on the sensitivity of each point in $P$ and on the total sensitivity.}
\label{fig:into}
\end{figure}

\section{Preliminaries}

\paragraph{Notations.} For integers $n,d \geq 2$, we denote by $0_d$ the origin of $\REAL^d$, and by $[n]$ the set $\br{1,\cdots,n}$. The set $\REAL^{n \times d}$ denotes the union over every $n \times d$ real matrix, and $I_d \in \REAL^{d \times d}$ denotes the identity matrix.
%A \emph{weighted set} is a pair $(P,w)$, where $P$ is a set of points in $\REAL^d$, and $w : P \to [0,\infty)$ is a non-negative weight function.
We say that a matrix $A \in \REAL^{d \times d}$ is \emph{orthogonal} if and only if $A^TA = AA^T = I_d$.
Finally, throughout the paper, vectors are addressed as column vectors.

% \paragraph{Coresets.}
% In common machine learning problems, we are given a set of input points $P \subseteq \REAL^d$ (training data), and a loss function $f:P\times \Q \to [0, \infty)$, where the goal is to find a query $x\in \Q$, that minimizes the total cost $\sum_{p\in P} f(p,x)$. In the face of big datasets, such task becomes onerous due to either insufficient memory (the dataset is too large to be stored at the machine), or computational impracticability (the algorithms used to train the machine learning models require too much time).
% As a solution to the problems above, we use the notion of $\eps$-coresets.

In what follows, we provide formally the notion of $\eps$-coreset in our context.

%In common machine learning problems, we are given a set of input points $P \subseteq \REAL^d$ (training data), and a loss function $f:P\times \Q \to [0, \infty)$, where the goal is to find a query $x\in \Q$, that minimizes the total cost $\sum_{p\in P} f(p,x)$.
%Throughout the paper we define an optimization problem by a pair $P$ and $f$, and for a given problem we wish to compute an $\eps$-coreset for the input set $P$ with respect to the loss function $f$. We define the concept of $\eps$-coreset.
\begin{Definition}[$\eps$-coreset]
\label{def:epsCore}
Let $P \subseteq \REAL^d$ be a set of $n$ points, $f : P \times \Q \to [0, \infty)$ be a near-convex function, and let $\eps \in (0,1)$. An $\eps$-coreset for $P$ with respect to $f$, is a pair $(S,v)$ where $S \subseteq P$, $v:S\to (0,\infty)$ is a weight function, such that for every $x \in \Q$,
$$\abs{1 -  \frac{\sum_{q \in S} v(q) f(q,x)}{\sum_{p \in P} f(p,x)}} \leq \eps.$$
% \begin{equation*}
% \abs{\sum_{p \in P} f(p,x) - \sum_{q \in S} v(q) f(q,x)} \leq \eps \sum_{p \in P}  f(p,x).
% \end{equation*}
\end{Definition}

%The following theorem is used to construct an $\eps$-coreset for a given set of points $P$ and a loss function $f$, it is based on what is known as sensitivity sampling; See~\cite{langberg2010universal,braverman2016new}.
%Intuitively, the sensitivity of a point $p \in P$ is a number that corresponds to the importance of $p$ with respect to the other points in $P$, and the optimization problem at hand. Suppose we computed an upper bound on the sensitivity of each point. Then, sampling a sufficiently large number of points from $P$ based on those bounds yields an $\eps$-coreset with high probability. The size of the coreset is proportional to the total sum of the sensitivity bounds.

The following theorem formally describes how to construct an $\eps$-coreset based on the sensitivity sampling framework.

\begin{theorem}[Restatement of Theorem 5.5 in~\cite{braverman2016new}]
\label{thm:coreset}
Let $P \subseteq \REAL^d$ be a set of $n$ points, and let $f : P \times \Q \to [0, \infty)$ be a loss function. For every $p \in P$ define the \emph{sensitivity} of $p$ as
\[
\sup_{x \in \Q} \frac{f(p,x)}{\sum_{q \in P} f(q,x)},
\]
where the sup is over every $x \in \Q$ such that the denominator is non-zero.
Let $s: P \to [0,1]$ be a function such that $s(p)$ is an upper bound on the sensitivity of $p$.
Let $t = \sum_{p \in P} s(p)$ and $d'$ be the~\emph{VC dimension} of the triplet $(P,f,\Q)$; see Definition~\ref{def:dimension}. Let $c \geq 1$ be a sufficiently large constant, $\varepsilon, \delta \in (0,1)$, and let $S$ be a random sample of $\abs{S} \geq \frac{ct}{\varepsilon^2}\left(d'\log{t}+\log{\frac{1}{\delta}}\right)$
i.i.d points from $P$, such that every $p \in P$ is sampled with probability $s(p)/t$. Let $v(p) = \frac{t}{s(p)\abs{S}}$ for every $p \in S$. Then, with probability at least $1-\delta$, $(S,v)$ is an $\varepsilon$-coreset for $P$ with respect to $f$.
% $t$ denote the total sensitivity of $(P,w,\REAL^d,f)$ (see Definition~\ref{def:sensitivity}). Let $\eps,\delta \in (0,1)$ and $S$ be a random sample of
% \[
% \left| S \right| \in \Omega{\left( \frac{t}{\eps^2} \left( d' \log{t} + \log{\left(\frac{1}{\delta}\right)}\right) \right)},
% \]
% i.i.d points from $P$, such that every $p \in S$, is sampled with probability $s(p) / t$.  Let $v(q) = \frac{t w(q)}{s(q) \abs{S}}$ for every $q \in S$. Then with probability at least $1-\delta$, $(S,v)$ is an $\eps$-coreset of $(P,w,\REAL^d,f)$; See Definition~\ref{def:epsCore}.
\end{theorem}

\section{Coreset for near-convex loss functions}
\label{sec:method}
For brevity purposes, proofs of the technical results have been omitted from this manuscript; we refer the reader to the supplementary material for the proofs. In addition, for simplicity of notation, we assume that the weight of each point in the input set is $1$, while in the supplementary material, we handle the general case where each point may have any nonnegative weight. We also discuss generalized versions of Definition~\ref{def:familyOfConvexFuncs} and Definition~\ref{def:svdWRTF}.
%\begin{restatable}{lemma}{bodyBoundary}
%\label{lem:bodyBoundary}
%Let $X_{0,1}$ be a level set of $f$ which satisfies Definition~\ref{def:propertiesOfF} and let $\partial\left( X_{0,1} \right)$ be it's boundary as defined in Definition~\ref{def:boundaryOfSet}. Then $X_{1,1} = \partial \left( X_{0,1} \right)$.
%\end{restatable}

\subsection{Bounding the sensitivity}
%Given a set of $n$ points $P \subseteq \REAL^d$, and a generalized bi-Lipschitz loss function $f \in \set{F}$ as in Definition~\ref{def:familyOfConvexFuncs}, we would like to compute a sensitivity bound for every $p\in P$, in order to construct an $\eps$-coreset for $P$ with respect to $f$.
The following provides the generalization of the well-conditioned basis of~\cite{dasgupta2009sampling}, which will be used to bound the sensitivities.

\begin{Definition}[$f$-SVD of $P$]
\label{def:svdWRTF}
Let $P \subseteq \REAL^d$ be a set of $n$ points, $f \in \set{F}$ be a near-convex loss function (see Definition~\ref{def:familyOfConvexFuncs}), and let $g,h,c_1,z$ be defined as in the context of Definition~\ref{def:familyOfConvexFuncs} with respect to $f$. Let $D \in [0, \infty)^{d \times d}$ be a diagonal matrix, and let $V \in \REAL^{d \times d}$ be an orthogonal matrix, such that for every $x \in \Q$,
\[
c_1 \term{\term{\norm{DV^Tx}_2}^z + \sum_{p \in P}  h(p,x)^z} \leq \sum_{p \in P}  f(p,x),
\]
and let $\alpha \in \Theta\left( \sqrt{d} \right)$ such that for every $x \in \Q$,
\[
\sum_{p \in P} g(p,x)^{\max\br{1, z}} \leq \term{\alpha \norm{DV^Tx}_2}^{\max\br{1, z}}.
\]
Define $U:P\to \Q$ such that $U(p) = \left( V D \right)^{-1} p$ for every $p\in P$. The tuple $\left( U,D,V\right)$ is the $f$-SVD of $P$.
\end{Definition}

\textbf{Note} that (i) such factorization exists for any set of points $P$ and any near-convex loss function $f : P \times \Q \to [0, \infty)$ satisfying Definition~\ref{def:familyOfConvexFuncs}, and (ii) the matrix $VD$ is invertible due to the fact that $D$ is of full rank which is a result of Property~\ref{prop:defLvl} of Definition~\ref{def:svdWRTF}. Both (i)-(ii) hold by using L\"{o}wner ellipsoid; see Fig.~\ref{fig:into} for intuitive explanation, and Lemma~\ref{lem:LownerToF} at the Supplementary Material for formal proof.

In what follows, we proceed to bound the sensitivity of each point and the total sensitivity, with respect to a loss function $f \in \set{F}$. This is by using the $f$-SVD of $P$.
%talk more

\begin{restatable}{lemma}{sensitivityBound}
\label{thm:sensitivityBound}
Let $P \subseteq \REAL^d$ be a set of $n$ points, and let $f \in \set{F}$ be a near-convex loss function as in Definition~\ref{def:familyOfConvexFuncs}. Let $g,h,c_1,c_2,z$ be defined as in the context of Definition~\ref{def:familyOfConvexFuncs} with respect to $f$, $(U,D,V)$ be the $f$-SVD of $P$, and let $\alpha \in \Theta\term{\sqrt{d}}$ which satisfies the conditions in Definition~\ref{def:svdWRTF}. Suppose that there exists a set $\br{v_j}_{j=1}^d \subseteq \REAL^d$ of $d$ unit vectors and $c > 0$, such that for every unit vector $y \in \REAL^d$ and $p \in P$,
\[
g\term{p,(D V^T)^{-1}y}^z \leq c \sum_{j=1}^{d} g\term{p,(D V^T)^{-1}v_j}^z.
\]
Then, for every $p\in P$, the sensitivity of $p$ is bounded by
\[
s(p) \leq \frac{2c_2}{c_1 n} + \frac{cc_2}{c_1}\sum_{j=1}^d \term{g\left( p, \left( DV^T\right)^{-1} v_j\right)}^z,
\]
and the total sensitivity is bounded by
\[
\sum_{p \in P} s(p) \leq  \frac{2c_2}{c_1}+ \frac{c c_2}{c_1}\maxArgs{n^{1-z}, 1}\alpha^z d.
\]
\end{restatable}

\subsection{The coreset construction}
%In this section we suggest an algorithm that computes an $\eps$-coreset for every set $P \subseteq \REAL^d$ with respect to any generalized bi-Lipschitz loss function $f \in \set{F}$; see Definition~\ref{def:familyOfConvexFuncs}.

Algorithm~\ref{alg:mainAlg} receives as input, a set $P$ of $n$ points in $\REAL^d$, a loss function $f \in \set{F}$ (see Definition~\ref{def:familyOfConvexFuncs}), and a sample size $m > 0$. As Theorem~\ref{thm:runTime} states, if the sample size $m$ is sufficiently large, then Algorithm~\ref{alg:mainAlg} outputs a pair $(S,v)$ that is with high probability, an $\eps$-coreset for $P$ with respect to $f$.

First, we set $d^\prime$ to be VC dimension of the triplet $\term{P, f, \Q}$; See Definition~\ref{def:dimension}.
The crux of our algorithm lies in generating the importance sampling distribution via efficiently computing upper bound on the sensitivity of each point (Lines~\ref{algLine:4}--\ref{alg2:L5}).
To do so we, we compute the $f$-SVD of $P$ at Lines~\ref{algLine:v1vd}--\ref{algLine:1}, and we use it to bound the sensitivity of each $p\in P$ as stated in Lemma~\ref{thm:sensitivityBound}; see Line~\ref{algLine:5}. Now we have all the needed ingredients to use Theorem~\ref{thm:coreset} in order to obtain an $\eps$-coreset, i.e., we sample i.i.d $m$ points from $P$ based on their sensitivity bounds (see Line~\ref{algLine:donesamle}), and assign a new weight for every sampled point at Line~\ref{algLine:10}.  %The size of the sample is proportional to $\eps,\delta,d'$, and the total sensitivity bound $t$; See Lines~\ref{algLine:6}--\ref{algLine:10}.

%%%comment: add why f svd help us

% In order to do that, in Lines~\ref{algLine:1} -- \ref{algLine:2}, it first computes the $f$-SVD of $P$ with respect to $F(P,w,x) = \sum_{p \in P} w(p) f(p,x)$

%Algorithm~\ref{alg:mainAlg} entails an importance sampling scheme where the importance of each point is denoted by the ratio of its sensitivity and the total sensitivity. For each $p \in P$, the sensitivity of each point $s(p)$ is given by Corollary~\ref{coro:sensitivityBoundGeneral}; see Section~\ref{sec:analysis}.

%The purpose of Line~\ref{algLine:3} is to compute the dimension of $(P,w,X,f)$. Lines~\ref{algLine:4}--\ref{algLine:6} are responsible for setting the sensitivity of each point of $P$, as well as computing the total sensitivity. In Lines~\ref{algLine:7}--\ref{algLine:8}, we define the minimal sample size needed in order to compute an $\eps$-coreset with failure probability $\delta$. Line~\ref{algLine:9} presents importance sampling where the probability of a point $p \in P$ to be sampled is $\frac{s(p)}{t}$, and the sample size is $m$. In Line~\ref{algLine:10}, we define a weight function $v$ which will guarantee alongside with $S$, an $\eps$-coreset. Finally, the sampled points and their corresponding weight function $v$ are returned at Line~\ref{algLine:11}.

\begin{algorithm}
\caption{$\coreset(P,f,m)$\label{alg:mainAlg}}
{\begin{tabbing}
\textbf{Input:} \quad\quad\= A set $P \subseteq \REAL^{d}$ of $n$ points, a near-convex loss function\\\>$f:P\times \REAL^d \to [0,\infty)$, and a sample size $m \geq 1$.\\
\textbf{Output:} \>A pair $(S,v)$ that satisfies Theorem~\ref{thm:runTime}.
\end{tabbing}}

Set $d' :=$ the VC dimension of triplet $\term{P, f, \Q}$ \label{algLine:3} \tcp{See Definition~\ref{def:dimension}}

Set $g$ and $\br{z, c_1,c_2}$ to be a function and a set of real positive numbers respectively, satisfying Property~\ref{prop:reductionToFam} and~\ref{prop:def} of Definition~\ref{def:familyOfConvexFuncs} with respect to $f$\\

%Set $g,c_1,c_2,c_3,c_4$ as in the context of Definition~\ref{def:familyOfConvexFuncs} with respect to $f$\\

Set $c > 0 $ and $\br{v_1, \cdots, v_d}$ to be  positive scalar and a a set of $d$ unit vectors in $\REAL^d$ respectively satisfying Lemma~\ref{thm:sensitivityBound}\label{algLine:v1vd}\\

Set $(U,D,V)$ to be the $f$-SVD of $(P,w)$\label{algLine:1} \tcp{See Definition~\ref{def:familyOfConvexFuncs}}
\For{every $p \in P$ \label{algLine:4}}
{
Set $s(p):= \frac{cc_2}{c_1}\sum_{j=1}^d g\term{ p, \term{ DV^T}^{-1}v_j}^z + \frac{2c_2}{c_1 n}$ \\ \tcp{ the bound of the sensitivity of $p$ as in Lemma~\ref{thm:sensitivityBound}} \label{algLine:5}
}
Set $t := \sum_{p \in P} s(p)$ \label{algLine:6} \\
Set $\tilde{c} \geq 1$ to be a sufficiently large constant \label{algLine:7} \tcp{Can be determined from Theorem~\ref{thm:runTime}}

Pick an i.i.d sample $S$ of
$m$ points from $P$, where each $p \in P$ is sampled with probability $\frac{s(p)}{t}$.  \label{algLine:donesamle}\\

set $v: \REAL^d \to [0,\infty]$ to be a weight function such that for every $q \in S$, $v(q)=\frac{t}{s(q)\cdot m}$. \label{algLine:10} \\
\Return $(S,v)$ \label{algLine:11}

\end{algorithm}

\begin{restatable}{theorem}{runTime}
\label{thm:runTime}
Let $P \subseteq \REAL^d$ be set of $n$ points, and $f \in \set{F}$ be a near-convex function. Let $R,r > 0$ be a pair of positive scalars as in Definition~\ref{def:familyOfConvexFuncs} with respect to $f$, and let $c,c_1,c_2,\alpha$ be defined as in the context of Lemma~\ref{thm:sensitivityBound} with respect to $f$. Let $\eps, \delta \in (0,1)$ be an error parameter and a probability of failure respectively, and let $d^\prime$ be the VC dimension of the triplet $\term{P, f, \REAL^d}$. Let $t= \frac{2c_2}{c_1}+ \frac{c c_2}{c_1}\maxArgs{n^{1-z}, 1}\alpha^z d $, $m \in \bigO{ \frac{t}{\eps^2} \term{d^\prime \log{\term{t}} + \log{\term{\frac{1}{\delta}}}}}$, and let $(S,v)$ be the output of a call to $\coreset(P,f,m)$. Then,
% \begin{itemize}
%     \item $P \subseteq \REAL^d$ be set of $n$ points,
%     \item $w : P \to \br{1}$ be a weight function,
%     \item $f \in \set{F}$ be a generalized bi-Lipschitz function as in Definition~\ref{def:familyOfConvexFuncs},
%     \item $R,r > 0$ be a pair of positive scalars as in Definition~\ref{def:familyOfConvexFuncs} with respect to $f$,
%     \item $\eps, \delta \in (0,1)$ be an error parameter and a probability of failure respectively,
%     \item $c,\br{c_i}_{i=1}^2,\alpha$ be defined as in the context of Lemma~\ref{thm:sensitivityBound} with respect to $f$,
%     \item $d^\prime$ be the VC dimension of the triplet $\term{P, f, \REAL^d}$,
%     \item $t= \frac{c_2c_5}{c_1}+ \frac{c c_2}{c_1}\maxArgs{n^{1-z}, 1}\alpha^z d $,
%     \item $m = \frac{t}{\eps^2} \term{d^\prime \log{\term{t}} + \log{\term{\frac{1}{\delta}}}}$,
%     \item and let $(S,v)$ be the output of a call to $\coreset(P,w,f,m)$.
% \end{itemize}
\begin{enumerate}[label=(\roman*)]
\item with probability at least $1-\delta$, $(S,v)$ is an $\eps$-coreset of size $m$ for $P$ with respect to $f$; see Definition~\ref{def:epsCore}.
\item The overall time for constructing $(S,v)$ is bounded by $\bigO{T(n,d)d^4 \log{\left( \frac{R}{r}\right)}}$, where $T(n,d)$ is a bound on the time it takes to compute a gradient of $\sum_{p \in P} f(p, x)$ with respect to any query $x \in \Q$.
\end{enumerate}
\end{restatable}

% \textbf{From sublinear to poly-logarithmic coreset size.}
% For an input set $P \subseteq \REAL^d$ of $n$ points, and a near-convex loss function $f \in \set{F}$, we provide an analysis at Section~\ref{sec:appendix_Stream} of the Appendix that shows how to obtain a coreset of size poly-logarithmic in $n$ for $P$ with respect to $f$. This is done using the merge-and-reduce tree presented at Algorithm~\ref{alg:streamAlg}; See Lemma~\ref{thm:polylog-coreset-merge-reduce} for completeness.

\textbf{Poly-logarithmic coreset size.} We provide an analysis that shows how to obtain a coreset of size poly-logarithmic in the input size $n$; see Algorithm~\ref{alg:streamAlg} and Lemma~\ref{thm:polylog-coreset-merge-reduce} at the Supplementary Material.

\section{Applications}
\label{sec:analysis}
In what follows, we provide various applications for our framework, .e.g, SVM, Logistic Regression, $\ell_z$ for $z \in (0,1)$, outlier resistant functions (similar to Tukey in behavior). For additional problems supported by our framework, we refer the reader to Section~\ref{sec:easy} at the Supplementary Material.

%Note that for each given example, the coreset were constructed using Algorithm~\ref{alg:mainAlg}.

\begin{table*}[h!]
    \caption{\textbf{Results}: The table below presents the coreset size and the time needed for constructing it with respect to a specific set of problems, where the input is a set of $n$ points in $\REAL^d$ denoted by $P$. In the table, $\boldsymbol{\mathrm{nnz}}\term{P}$ denotes the total number of nonzero entries in the set $P$, $\boldsymbol{\Tilde{\W}}$ denotes the ratio between the number of positive and negative labeled points (in practice, it's a constant number), $\boldsymbol{\C} = \sqrt{n}$ is the given regularization parameter for the problems, $\boldsymbol{\gamma} \geq 1$ is defined as in Corollary~\ref{thm:restrictedLp}, $\boldsymbol{\eps}$ is the error parameter, and $\boldsymbol{\delta}$ is the probability of failure.}
    \begin{adjustbox}{width=\textwidth}
    \centering
    \begin{tabular}{|c|c|c|}
    \hline
        Problem type &  Coreset's size & Construction time\footnotemark \\ \hline
        Logistic regression & $\bigO{\frac{d\sqrt{n}}{\eps^2} \term{d \log{\term{ d \sqrt{n}}} + \log{\term{\frac{1}{\delta}}}}}$ & $\bigO{nd^2}$ \\ \hline
        $\ell_z$-Regression for $z \in (0,1)$ & $\bigO{\frac{n^{1-z}d^{\frac{z}{2} + 1}}{\eps^2} \term{d \log{\term{n^{1-z}d^{\frac{z}{2} + 1}}}} + \log{\term{\frac{1}{\delta}}}}$ & $\bigO{\nnz{P}\log{n} + d^{\bigO{1}}}$ \\ \hline
        %$\ell_z$-Regression for $z \in [1,2)$ & $\bigO{\frac{d^{\frac{z}{2} + 1}}{\eps^2} \term{d \log{\term{d^{\frac{z}{2} + 1}}}} + \log{\term{\frac{1}{\delta}}}}$ & $\bigO{\nnz{P}\log{n} + d^{\bigO{1}}}$ \\ \hline
        %$\ell_2$-Regression& $\bigO{\frac{d}{\eps^2} \term{d \log{\term{d}} + \log{\term{\frac{1}{\delta}}}}}$ & $\bigO{nd^2}$ \\\hline
        SVM & $\bigO{\frac{d\sqrt{n}+ \frac{\Tilde{\W}^2+1}{\Tilde{\W}}}{\eps^2} \term{d \log{\term{d\sqrt{n} + \frac{\Tilde{\W}^2+1}{\Tilde{\W}}}} + \log{\term{\frac{1}{\delta}}}}}$ & $\bigO{nd^2}$ \\ \hline
        Restricted $\ell_z$-regression & $\bigO{\frac{\gamma d^{2 + \abs{\frac{1}{2} - \frac{1}{z}}}}{\eps^2} \term{d \log{\term{\gamma d^{2 + \abs{\frac{1}{2} - \frac{1}{z}}}}} + \log{\term{\frac{1}{\delta}}}}}$ & $\bigO{\nnz{P}\log{n} + d^{\bigO{1}}}$ \\ \hline
        %Least square error & $\bigO{\frac{\term{d+2}^{2}}{\eps^2} \term{d \log{\term{\term{d+2}^{2}}} + \log{\term{\frac{1}{\delta}}}}}$ & $\bigO{\nnz{P}\log{n} + d^{\bigO{1}}}$ \\ \hline
    \end{tabular}
    \end{adjustbox}
    \label{tab:my_label}

\end{table*}

\footnotetext[1]{Problems which are reduced to $\ell_z$-regression problems for any $z \geq 1$, are easier to be dealt with in term of coreset construction time due to the existence of randomized algoritm of computing the L\"{o}wner ellipsoid by \cite{clarkson2017low}; see Section~\ref{sec:sup_experimental} for detailed description.}

\begin{restatable}[Logistic Regression]{corollary}{logisticReg}
\label{thm:logisticReg}
Let $P \subseteq \REAL^d$ be a set of $n$ points such that for every $p \in P$, $\norm{p}_2 \leq 1$, $y : P \to \br{-1, 1}$ be a labeling function, $\C \geq 1$ be a regularization parameter such that for every $p \in P$, $x \in \REAL^d$ and $b \in \REAL$,
$$
\flog\term{p,\begin{bmatrix} x \\ b\end{bmatrix}}= \frac{1}{\C}\ln{\left( 1+e^{p^Tx + y(p) \cdot b}\right)} + \frac{1}{2n}\norm{x}_2^2.$$

Let $\eps, \delta \in (0,1)$ be an error parameter and a probability of failure respectively, $m \in \bigO{\frac{dn}{\C \eps^2} \term{d \log{\term{\frac{dn}{\C}}} + \log{\term{\frac{1}{\delta}}}}}$, and let $(S,v)$ be the output of a call to $\coreset\term{P,\flog, m}$. Then, with probability at least $1 - \delta$, $(S,v)$ is an $\eps$-coreset (of size $m$) for $P$ with respect to $\flog$.
\end{restatable}

% We also provide a lower bound on the coreset size for logistic regression (including the regularization parameter and term).
% there exists $S \subseteq P$ of size $\abs{S}\in\bigO{\frac{d\W}{\C \eps^2} \term{d \log{\term{\frac{dC}{\C}}} + \log{\term{\frac{1}{\delta}}}}}$ such that with probability at least $1 - \delta$, $(S,v)$ is an $\eps$-coreset.
%%%%%%%%%%%%%%%%%%%%%%%%
%
%%%%%%%%%%%%%%%%%%%%%%%%%
% For every $p \in P$, let $P_{y(p)} = \br{q \mid q \in P, y(q) = y(p)}$ denote the set of points with the same label as the label assigned to $p$.
% %Let $\flog[g]:P \times \REAL^{d+1}\to (0,\infty)$, such that for every $p \in P$, $b \in \REAL$ and $x\in \REAL^d$, $\flog[g](p,(x\mid b)) = \abs{p^Tx}^2$.
% Let $(U,D,V)$ be the $f$-SVD of $(P,w)$ with respect to $\flog[f]$. Then, claims (i) -- (ii) hold as follows:
% \begin{enumerate}[label=(\roman*)]
% \item for every $p \in P$, the sensitivity of $p$ with respect to the query space $(P,w, \REAL^{d+1},\flog)$ is bounded by
% \[
% s(p) = \frac{64 w(p)}{C} + \frac{32\sum\limits_{q \in P_{y(p)}} w(q)}{C} w(p)\norm{U(p)}_2^2 ,
% \]
% \item and the total sensitivity is bounded by
% \[
% \sum\limits_{p \in P} s(p) \leq \frac{32}{C} \left(2 + d \right) \sum\limits_{p \in P}w(p).
% \]
% \end{enumerate}

\begin{restatable}[$\ell_z$-Regression where $z \in (0,1)$]{corollary}{lpRegressionNC}
\label{thm:lpRegressionNC}
Let $P\subseteq \REAL^d$ be a set of $n$ points, $z\in (0,1)$ and let $\fnclz:P\times \REAL^d$ be a loss function such that for every $x \in \REAL^d$, and $p \in P$,
$$\fnclz(p,x) = \abs{p^Tx}^z.$$

Let $\eps, \delta \in (0,1)$, $m \in\bigO{\frac{n^{1-z}d^{\frac{z}{2} + 1}}{\eps^2} \term{d \log{\term{n^{1-z}d^{\frac{z}{2} + 1}}} + \log{\term{\frac{1}{\delta}}}}}$, and let $(S,v)$ be the output of a call to $\coreset\term{P,\fnclz, m}$. Then, with probability at least $1 - \delta$, $(S,v)$ is an $\eps$-coreset (of size $m$) for $P$ with respect to $\fnclz$.

% Let $P^\prime = \br{p^\prime = {w(p)}^{\frac{1}{z}} p \mid p \in P}$, and let $\fnclz[f^2]: P'\times \REAL^d \to (0,\infty) $ such that for every $x \in \REAL^d$ and $p' \in P^\prime$, $\fnclz[f^2](p',x)= \abs{p'^Tx}$. Let $(U,D,V)$ be the $f$-SVD of $P^\prime$ with respect to $\fnclz[f^2]$. Then, claims (i) -- (ii) hold as follows:

% \begin{enumerate}[label=(\roman*)]
% \item \label{case:nonCLz_1} for every $p \in P$, the sensitivity of $p$ with respect to the query space $(P,w, \REAL^{d},\fnclz)$ is bounded by
% $s(p) \leq \norm{U(p)}_z^z, $
% \item \label{case:nonCLz_2} and the total sensitivity is bounded by $\sum\limits_{p \in P} s(p) \leq n^{1-z} d^{\frac{z}{2} + 1}.$
% \end{enumerate}
\end{restatable}

%\subsection{Non straight forward examples covered by our framework}
We no show how our framework can be used to compute an $\eps$-coreset for some query spaces where the involved loss functions are not from the family $\set{F}$. The coreset construction algorithms are hidden in the constructive proofs of the following corollaries.

\begin{restatable}[Support Vector Machines]{corollary}{SVM}
\label{lem:SVM}
Let $P \subseteq \REAL^d$ be a set of $n$ points such that for every $p \in P$, $\norm{p} \leq 1$. Let $y : P \to \br{1,-1}$ be a labelling function, $\C \geq 1$ be a regularization parameter such that for every $p \in P$, $x \in \REAL^d$, and $b \in \REAL$,
$$\fsvm\term{p,\begin{bmatrix} x \\ b \end{bmatrix}} = \lambda \max\br{0, 1-\left( p^Tx + y(p) \cdot b \right)} + \frac{1}{2n}\norm{x}_2^2.$$ 
Let $P_+ = \br{p \middle| p \in P, y(p) = 1}$, $P_- = P \setminus P_+$, $\Tilde{\W} = \frac{\abs{P_+}}{\abs{P_{-}}}$. %$\W_- = \sum_{p \in P_-} w(p)$.

Then, there exists an algorithm that gets the set $P$ as an input, and returns a pair $(S,v)$, such that (i) with probability at least $1-\delta$, $(S,v)$ is an $\eps$-coreset for $P$ with respect to $\fsvm$ , and (ii) the size of the coreset is $\abs{S} \in \bigO{\frac{1}{\eps^2}\term{\frac{dn}{\C} + \frac{\Tilde{\W}^2+1}{\Tilde{\W}}} \term{d \log{\term{\frac{dn}{\C} + \frac{\Tilde{\W}^2+1}{\Tilde{\W}}}} + \log{\frac{1}{\delta}}}}$.
\end{restatable}

%[The restricted $\ell_z$-Regression for $z \in [1,\infty)$]
\begin{restatable}[Outlier resistant functions]{corollary}{restrictedLp}
\label{thm:restrictedLp}
Let $P \subseteq \REAL^d$ be a set of $n$ points, and let $\fmestimator : P \times \REAL^d \to [0,\infty)$ be loss function such that for every $x \in \REAL^d$, and $p \in P$, $$\fmestimator(p,x) = \min\br{\abs{p^T x}, \norm{x}_z}.$$

Then, there exists an algorithm that gets the set $P$ as an in input, and returns a pair $(S,v)$, such that (i) with probability at least $1-\delta$, $(S,v)$ is an $\eps$-coreset for $P$ with respect to $\fmestimator$, and (ii) the size of the coreset is $\bigO{\frac{\gamma d^{2 + \abs{\frac{1}{2} - \frac{1}{z}}}}{\eps^2} \term{d \log{\term{\gamma d^{2 + \abs{\frac{1}{2} - \frac{1}{z}}}}} + \log{\term{\frac{1}{\delta}}}}}$, where $\gamma$ is defined in the proof.
\end{restatable}

\section{Experimental Results}
\label{sec:results}
In what follows we evaluate our coreset against uniform sampling on real-world datasets, with respect to the SVM problem, Logistic regression problem and $\ell_z$-regression problem for $z \in (0,1)$.  Additional details of our setup can be found at Section~\ref{sec:sup_experimental} of the Supplementary Material.

\textbf{Software/Hardware. } Our algorithms were implemented in Python 3.6~\cite{10.5555/1593511} using \say{Numpy}~\cite{oliphant2006guide}, \say{Scipy}~\cite{2020SciPy-NMeth} and \say{Scikit-learn}~\cite{scikit-learn}.
Tests were performed on $2.59$GHz i$7$-$6500$U ($2$ cores total) machine with $16$GB RAM.

% Our experiments were implemented in Python (3.6) and were performed on a

%We have evaluated our coreset's provable guarantees on real-world datasets, where the performance of our coreset is compared to that of uniform random sampling. In this paper, we have provided sensitivity bounds for each input point, and considered the construction of the coreset itself to be a black box which we use for different environmental settings, e.g.,off-line, online setting (streaming), or distributed setting. Our experiments were done with respect to Logistic Regression and SVMs, implemented in Python and MATLAB, and were performed on a $3.4$GHz i7-4770 (4 cores total) machine with $16$GB RAM.

%  \textbf{Datasets.} The following datasets were used for our experiments and were downloaded from UCI machine learning repository~\cite{Dua:2019}:
% \begin{enumerate*}[label=(\roman*)]
%     \item \textit{HTRU~\cite{Dua:2019}},
%     \item \textit{Skin~\cite{Dua:2019}},
%     \item \textit{Cod-rna~\cite{uzilov2006detection}},
%     \item \textit{Web dataset~\cite{CC01a}}, and
%     \item \textit{3D spatial networks~\cite{Dua:2019}}.
% \end{enumerate*}

\textbf{Datasets.} The following datasets were used for our experiments mostly from UCI machine learning repository~\cite{Dua:2019}:
\begin{enumerate}[label=(\roman*)]
    \item \textbf{HTRU}~\cite{Dua:2019} --- $17,898$ radio emissions of the Pulsar star each consisting of $9$ features.
    \item \textbf{Skin}~\cite{Dua:2019} --- $245,057$ random samples of R,G,B from face images consisting of $4$ dimensions.
    \item \textbf{Cod-rna}~\cite{uzilov2006detection} --- consists of $59,535$ samples, $8$ features, which has two classes (i.e. labels), describing RNAs.
    \item \textbf{Web} dataset~\cite{CC01a} -- $49,749$ web pages records where each record is consists of $300$ features.
    \item \textbf{3D spatial networks}~\cite{Dua:2019} -- 3D road network with highly accurate elevation information (+-20cm) from Denmark used in eco-routing and fuel/Co2-estimation routing algorithms consisting of $434,874$ records where each record has $4$ features.
\end{enumerate}

\textbf{Evaluation against uniform sampling.}
At Fig.~\ref{fig:HTRUsvm}--\ref{fig:Codlogistic} and Fig.~\ref{fig:l05reg}--\ref{fig:l08reg}, we have chosen $20$ sample sizes, starting from $50$ till $500$, at Figures~\ref{fig:w8asvm}--\ref{fig:w8alogistic}, we have chosen $20$ sample sizes starting from $4000$ till $16,000$.
At each sample size, we generate two coresets, where the first is using uniform sampling and the latter is using Algorithm~\ref{alg:mainAlg}.
For each coreset $(S,v)$, we find $x^* \in \arg\min_{x \in \REAL^d} \sum_{p \in S} v(p) f(p,x)$, and the approximation error $\varepsilon$ is set to be $(\sum_{p \in P} f\term{p,x^*}) / (\min_{x \in \Q} \sum_{p \in P} f(p,x))-1$. The results were averaged across $40$ trials, while the shaded regions correspond to the standard deviation.

%The shaded areas in our figures translate to $\delta$ (the probability of failure), and as portrayed, our coreset has lower probability of failure and better approximation per each sample size than those of uniform sampling.

%The effectiveness of our coreset is established by showing that the optimal solution which is computed upon the coreset, approximates the optimal solution which is with respect to the whole data. The approximation is shown upon the evaluation of the objective cost function; see Definition~\ref{def:epsCore}.
%Post to constructing the coreset, we use state-of-the-art solvers to compute the optimal solution, which in our case, we simply used the solvers provided by Scikit-learn~\cite{scikit-learn} and we then show that the optimal solution, $x^\prime$ which was computed on compressed data (coreset), is an $(1+\eps)$ approximation to the optimal solution, $x^*$ which is computed on the whole data.
%Such results are obtained when the coreset is under the off-line setting, as portrayed at Figure~\ref{fig:logisticOffLineResults} and Figure~\ref{fig:svmOffLineResults}.
%In Figure~\ref{fig:logisticOffLineResults} (and as well as for Figure~\ref{fig:svmOffLineResults}), the dots on the strict line portray the average relative error over $96$ repetitions and the shaded area represents the variance with respect to the relative error. As seen in the following figures, we outperform uniform sampling both with respect to the average relative error and the variance (i.e., lower variance).

\newcommand\s{0.49}
% \begin{figure*}[t!]
% \centering

% 	\caption{$\ell_z$-regression for $z \in (0,1)$ tests.}\label{fig:vecsum}
% \end{figure*}

\begin{figure*}[t!]
\centering

%%%%%%%%%%%%%%% HTRU_2
\begin{subfigure}[t]{0.49\textwidth}
\includegraphics[width=.49\textwidth]{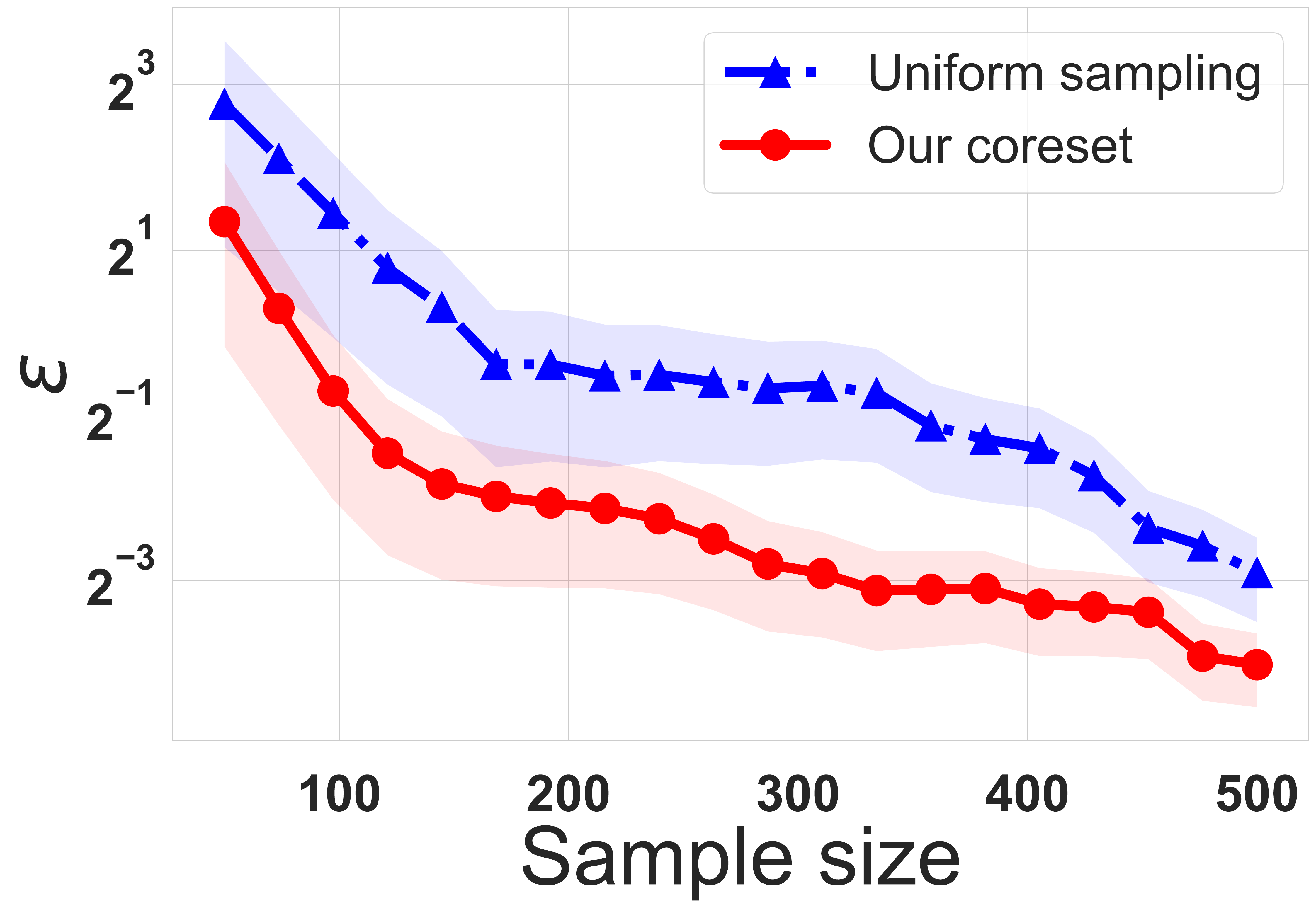}
\includegraphics[width=.49\textwidth]{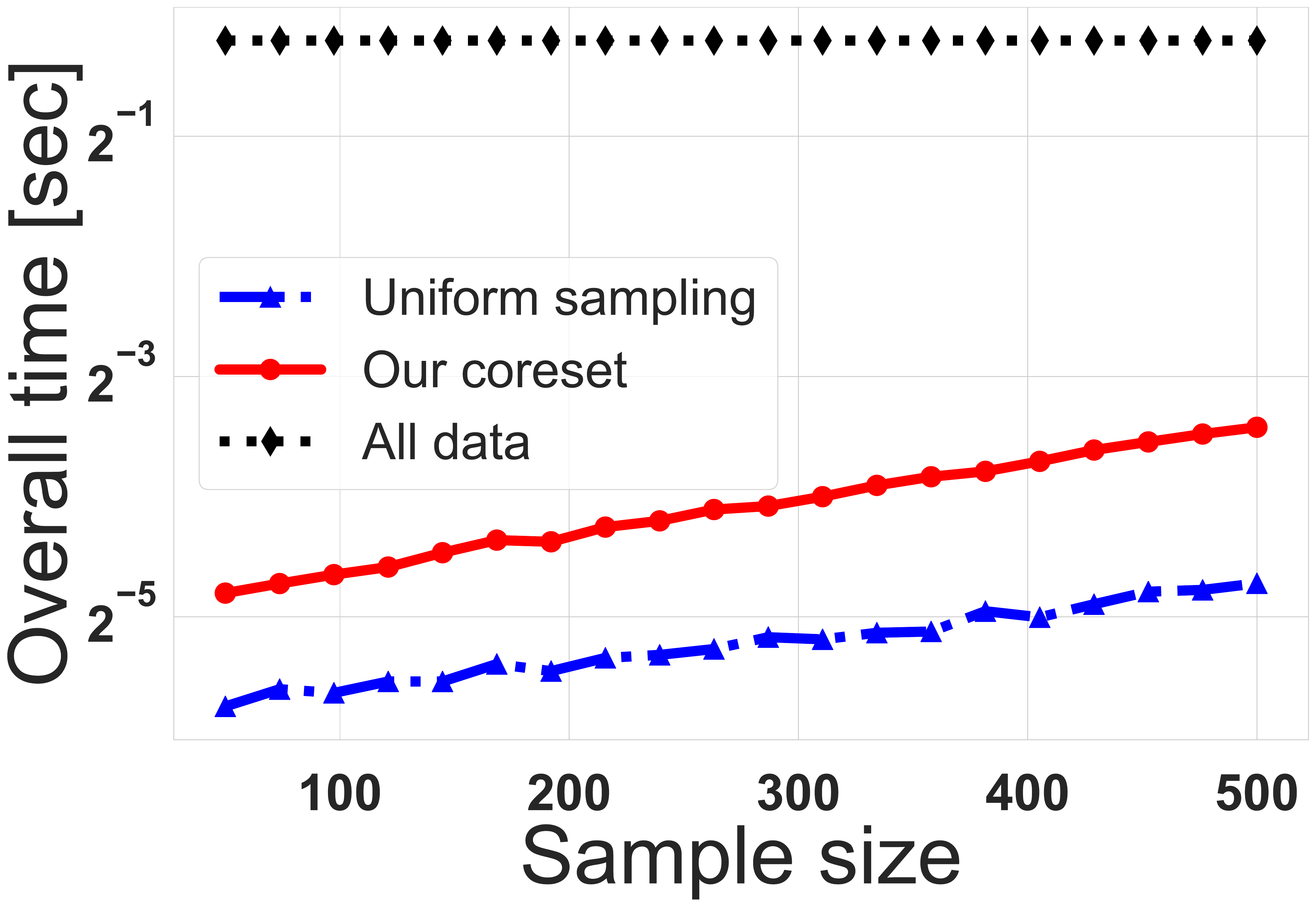}
\caption{Dataset (i): Support vector machines}
\label{fig:HTRUsvm}
\end{subfigure}
\begin{subfigure}[t]{0.49\textwidth}
\includegraphics[width=.49\textwidth]{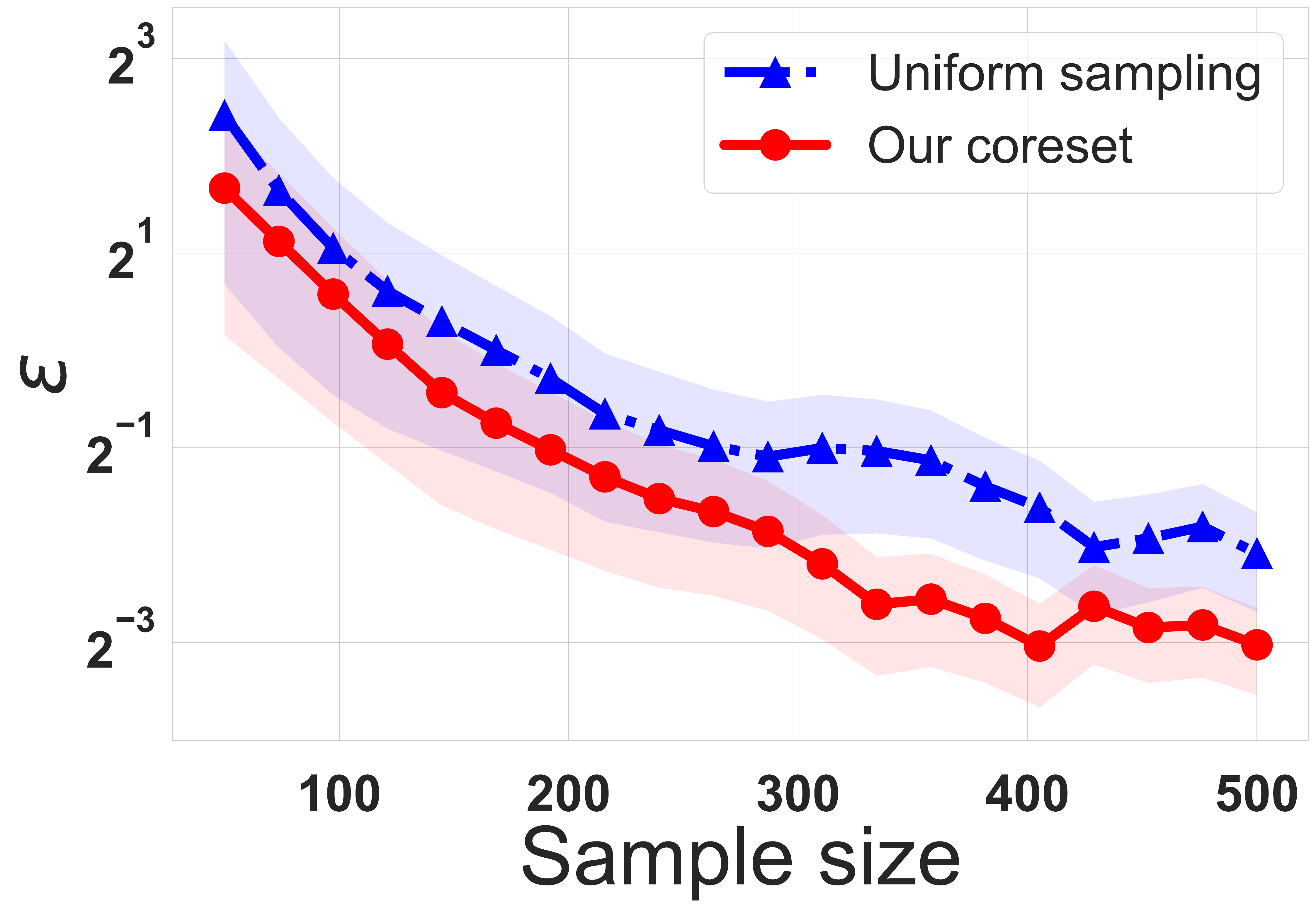}
\includegraphics[width=.49\textwidth]{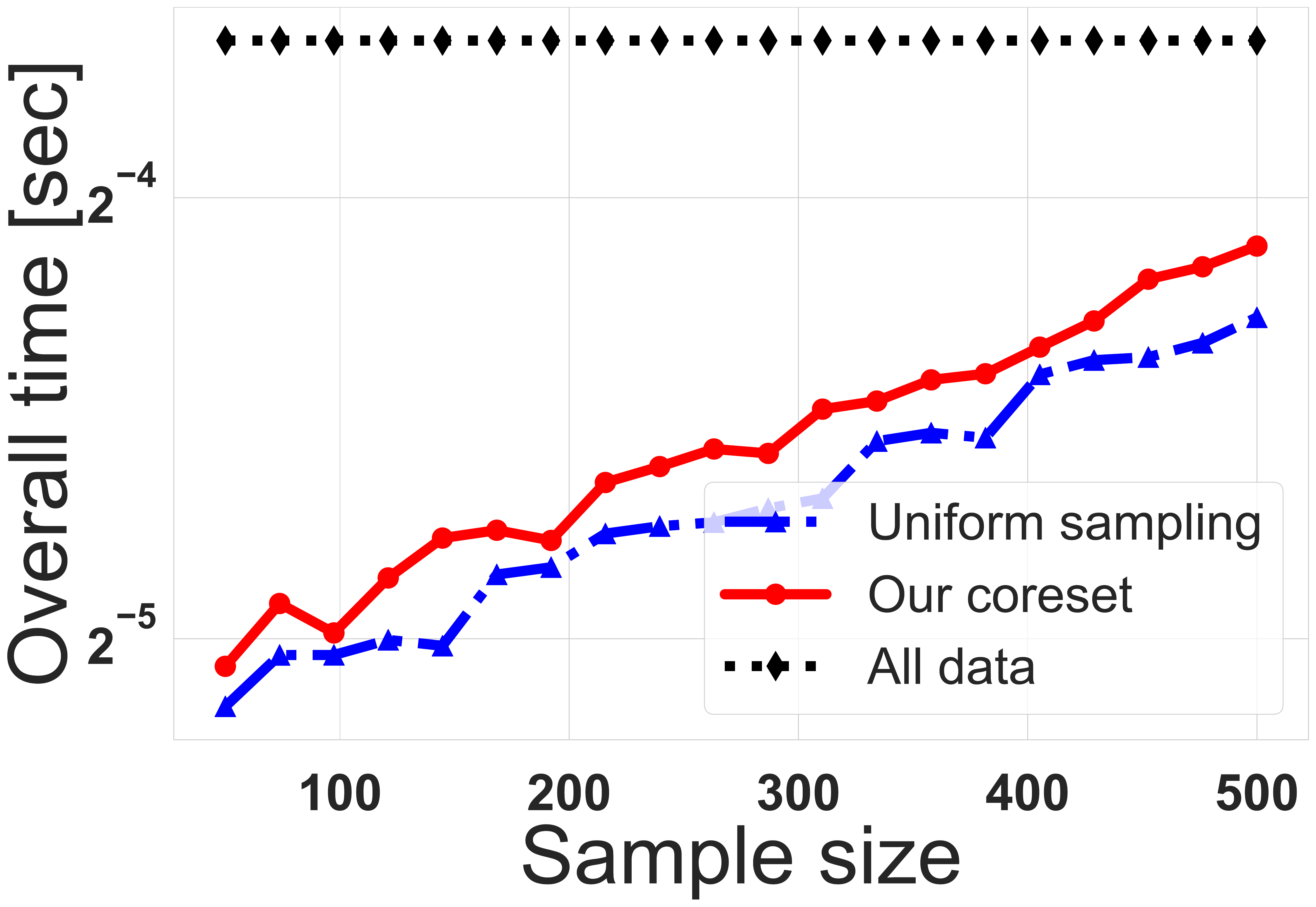}
\caption{Dataset (i): Logistic regression}
\label{fig:HTRUlogistic}
\end{subfigure}

%%%%%%%%%%%%%%%%%%%%%% Skin %%%%%%%%%%%%%%%%%%%%%
\begin{subfigure}[t]{0.49\textwidth}
\includegraphics[width=.49\textwidth]{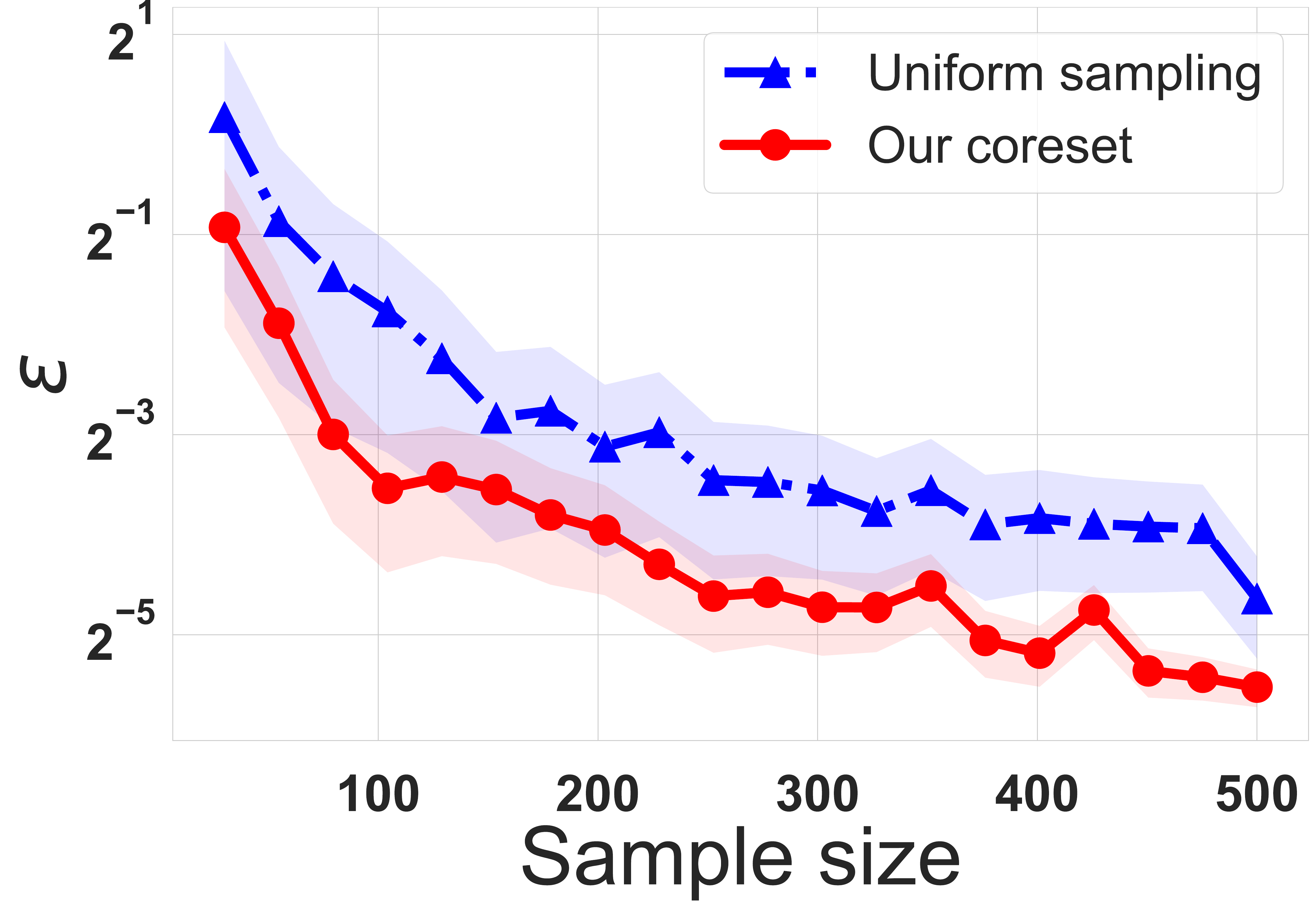}
\includegraphics[width=.49\textwidth]{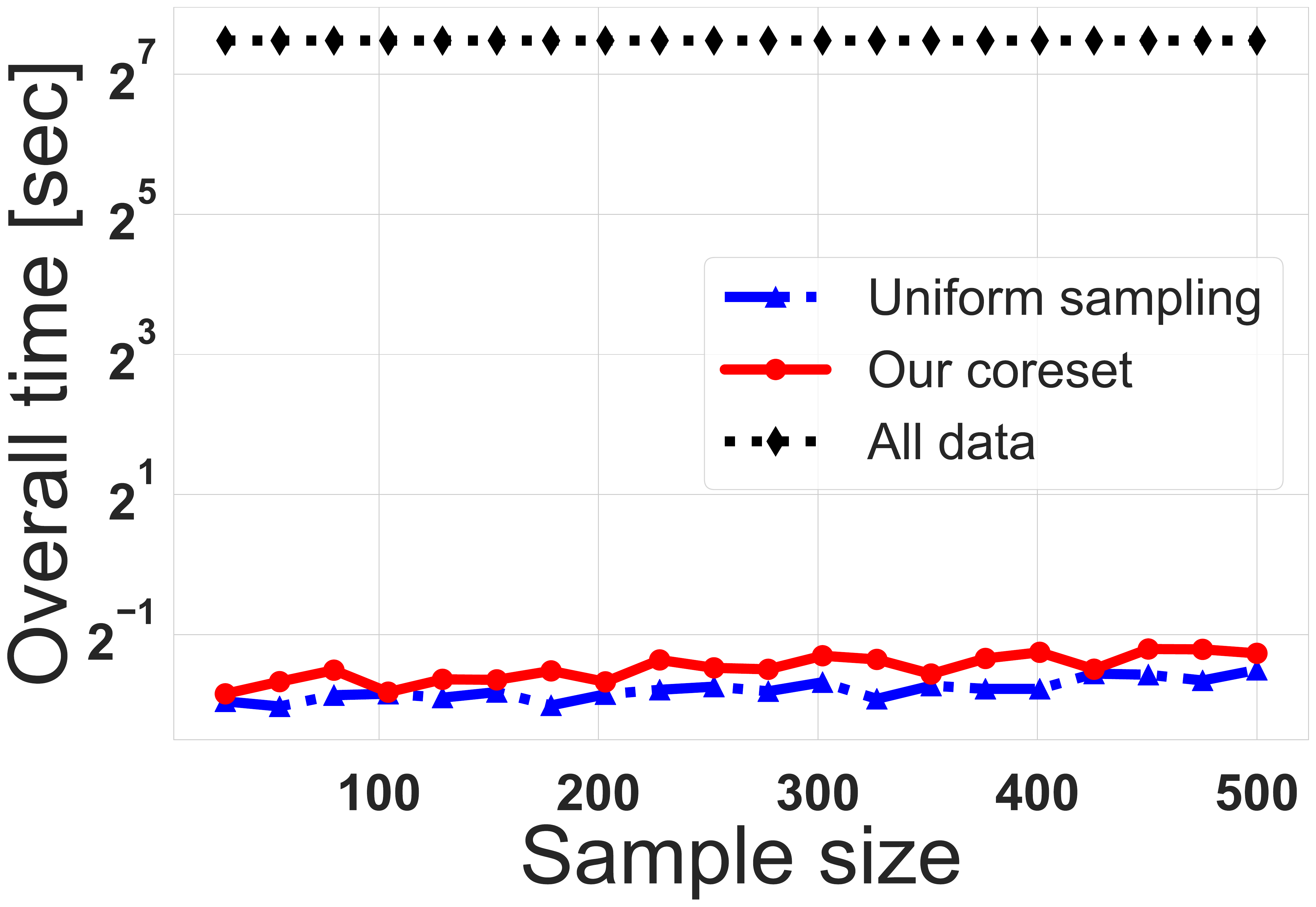}
\caption{Dataset (ii): Support vector machines}
\label{fig:Skinsvm}
\end{subfigure}
\begin{subfigure}[t]{0.49\textwidth}
\includegraphics[width=.49\textwidth]{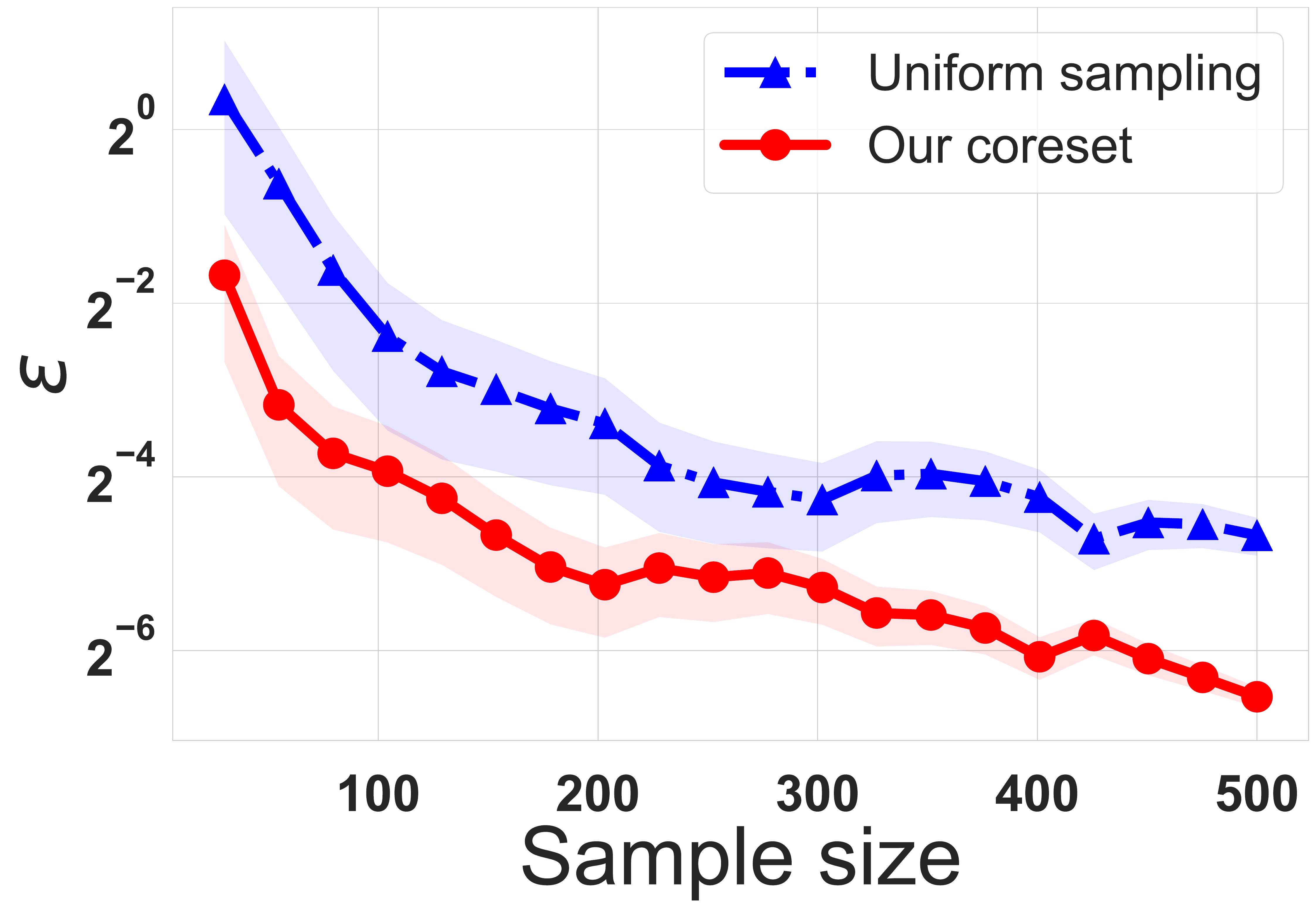}
\includegraphics[width=.49\textwidth]{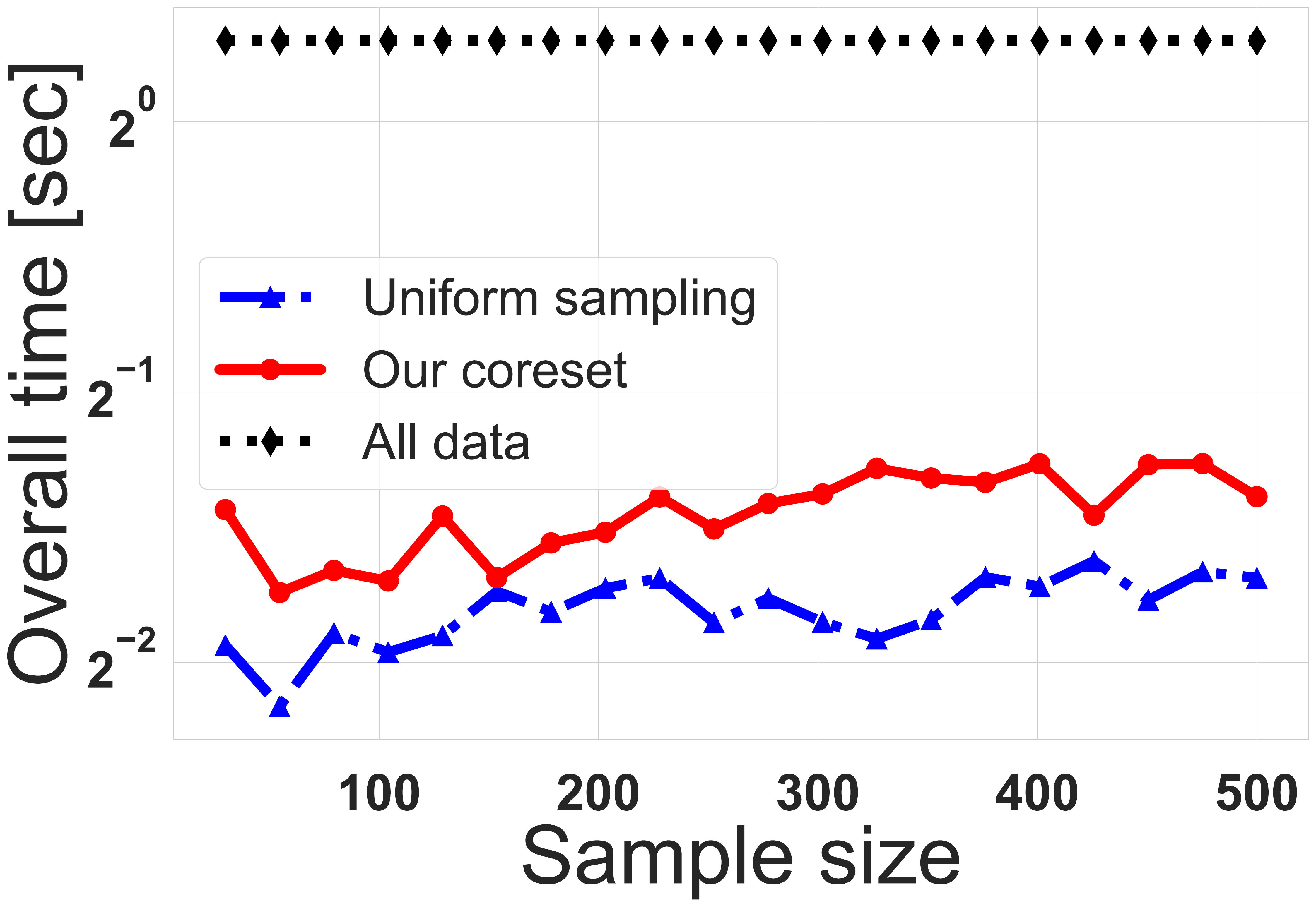}
\caption{Dataset (ii): Logistic regression}
\label{fig:Skinlogistic}
\end{subfigure}

%%%%%%%%%%%%%%%%%%%%%%% Cod %%%%%%%%%%%%%%%%%%%%%%
\begin{subfigure}[t]{0.49\textwidth}
\includegraphics[width=.49\textwidth]{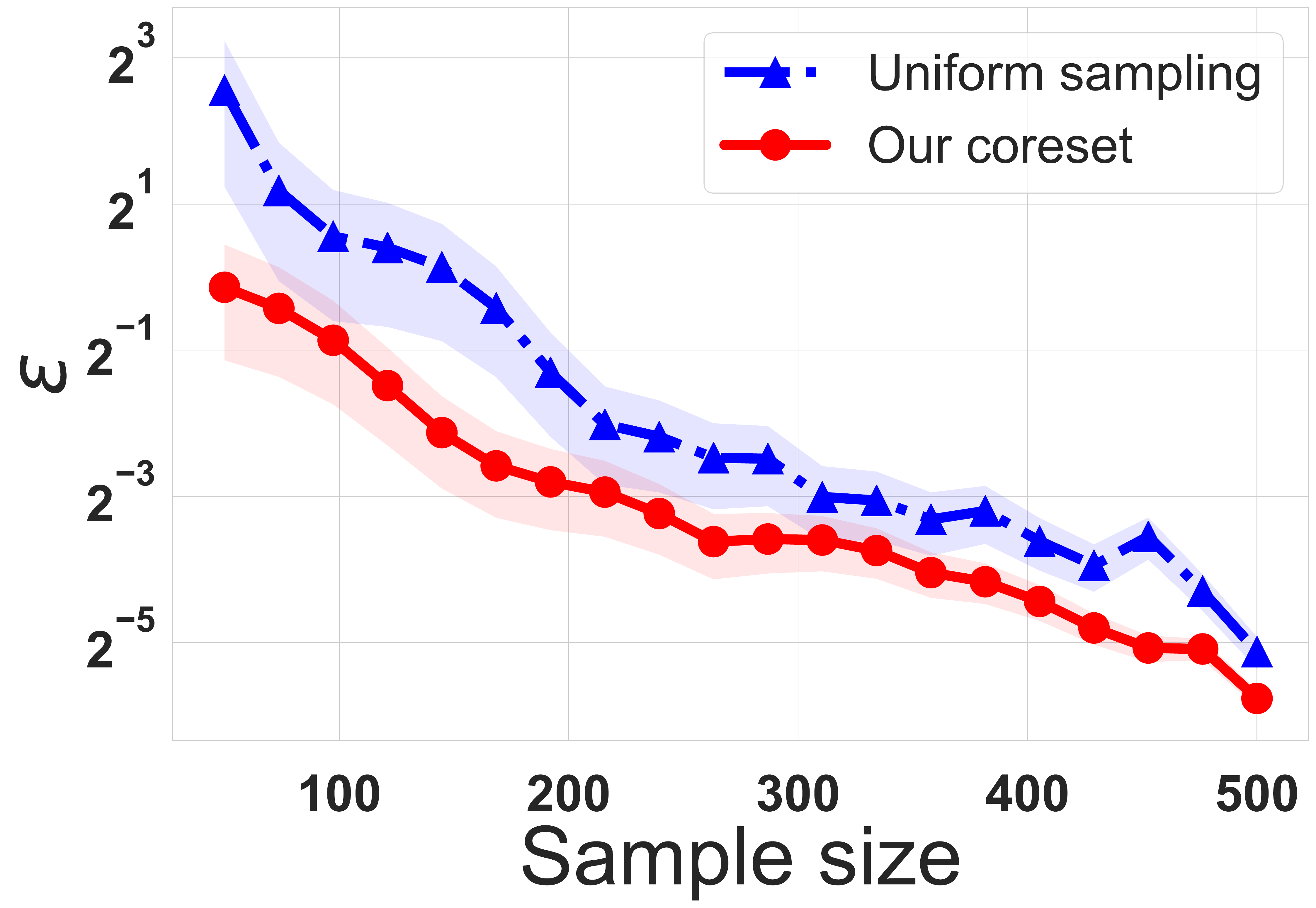}
\includegraphics[width=.49\textwidth]{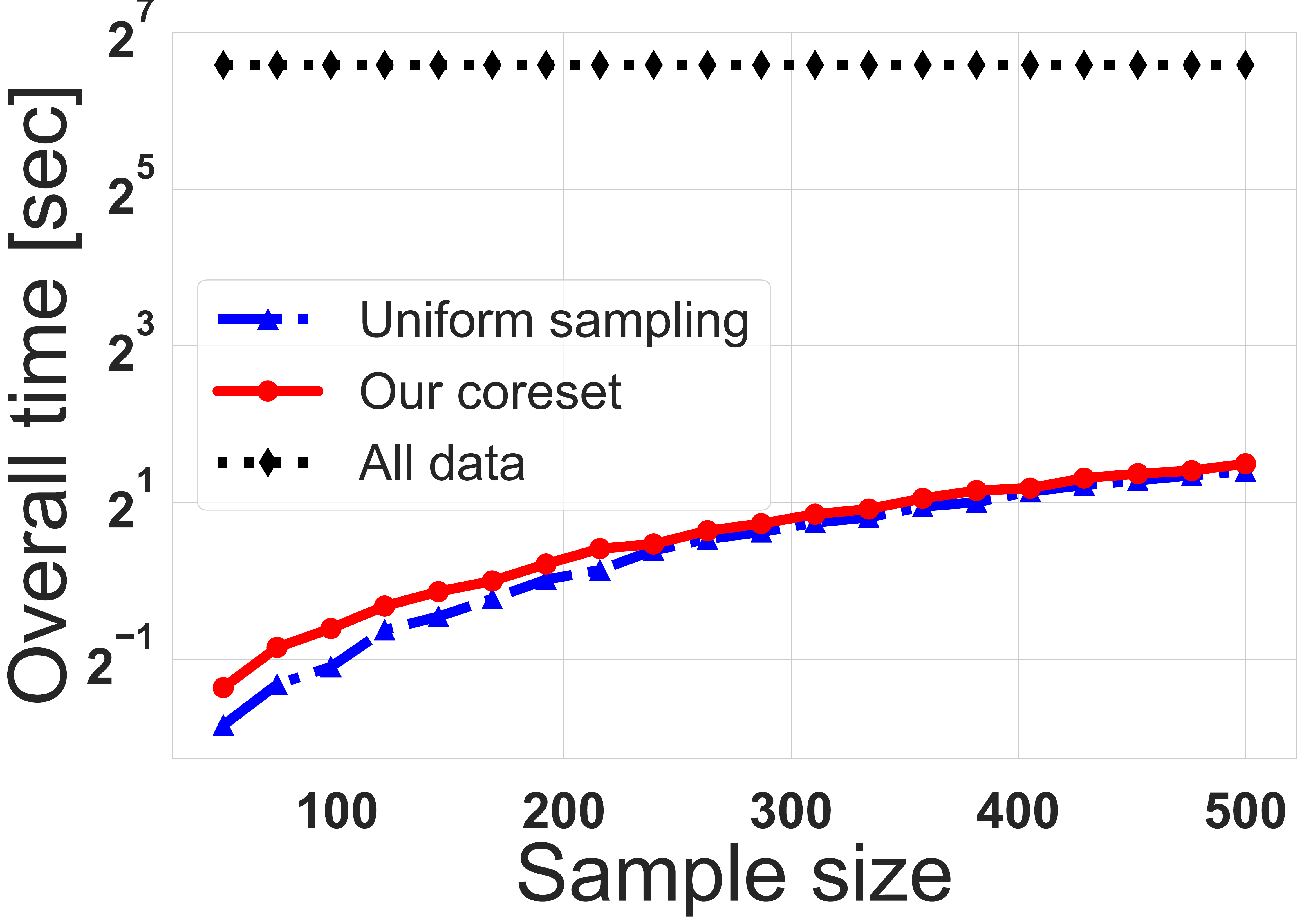}
\caption{Dataset (iii): Support vector machines}\label{fig:Codsvm}
\end{subfigure}
\begin{subfigure}[t]{0.49\textwidth}
\includegraphics[width=.49\textwidth]{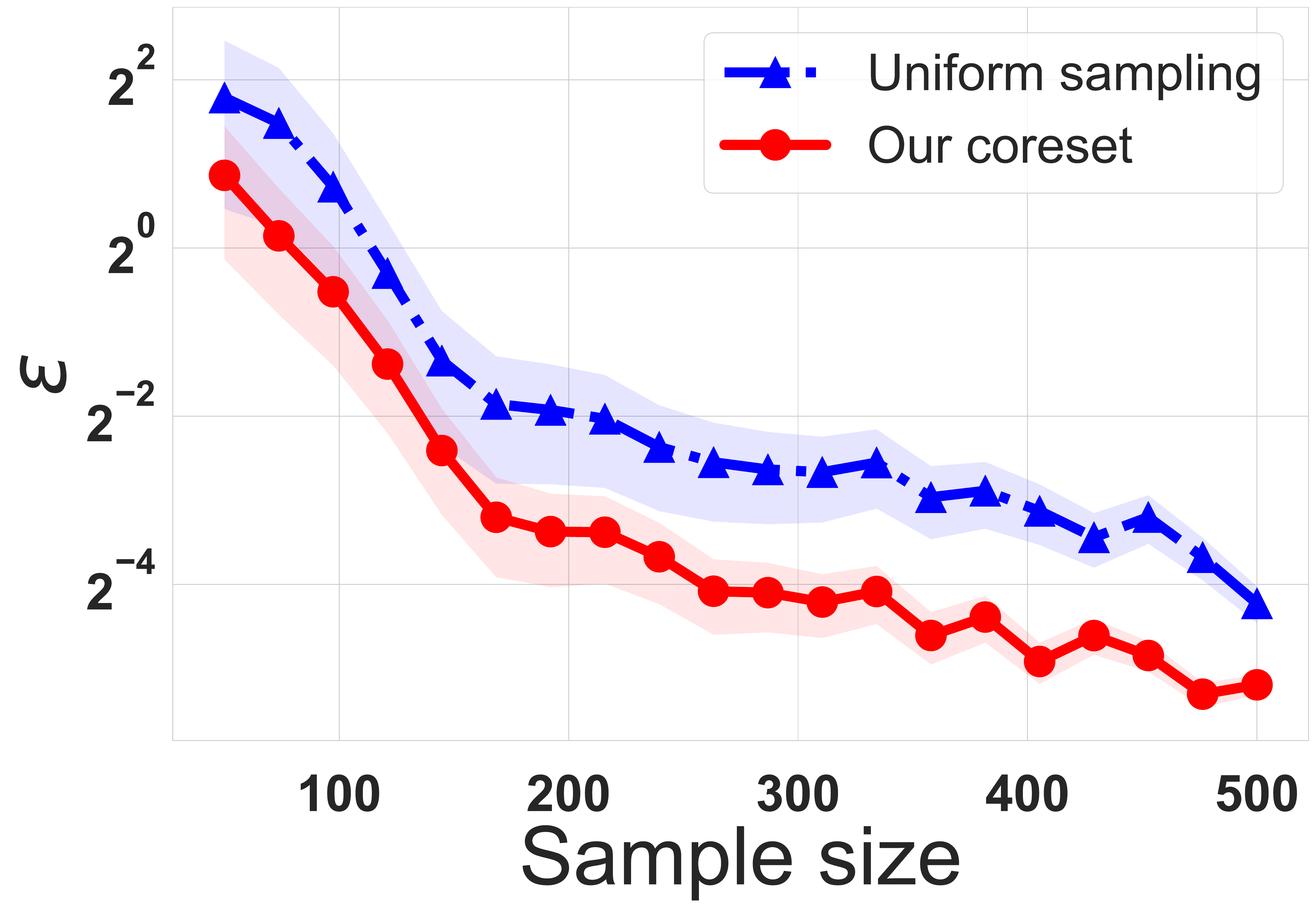}
\includegraphics[width=.49\textwidth]{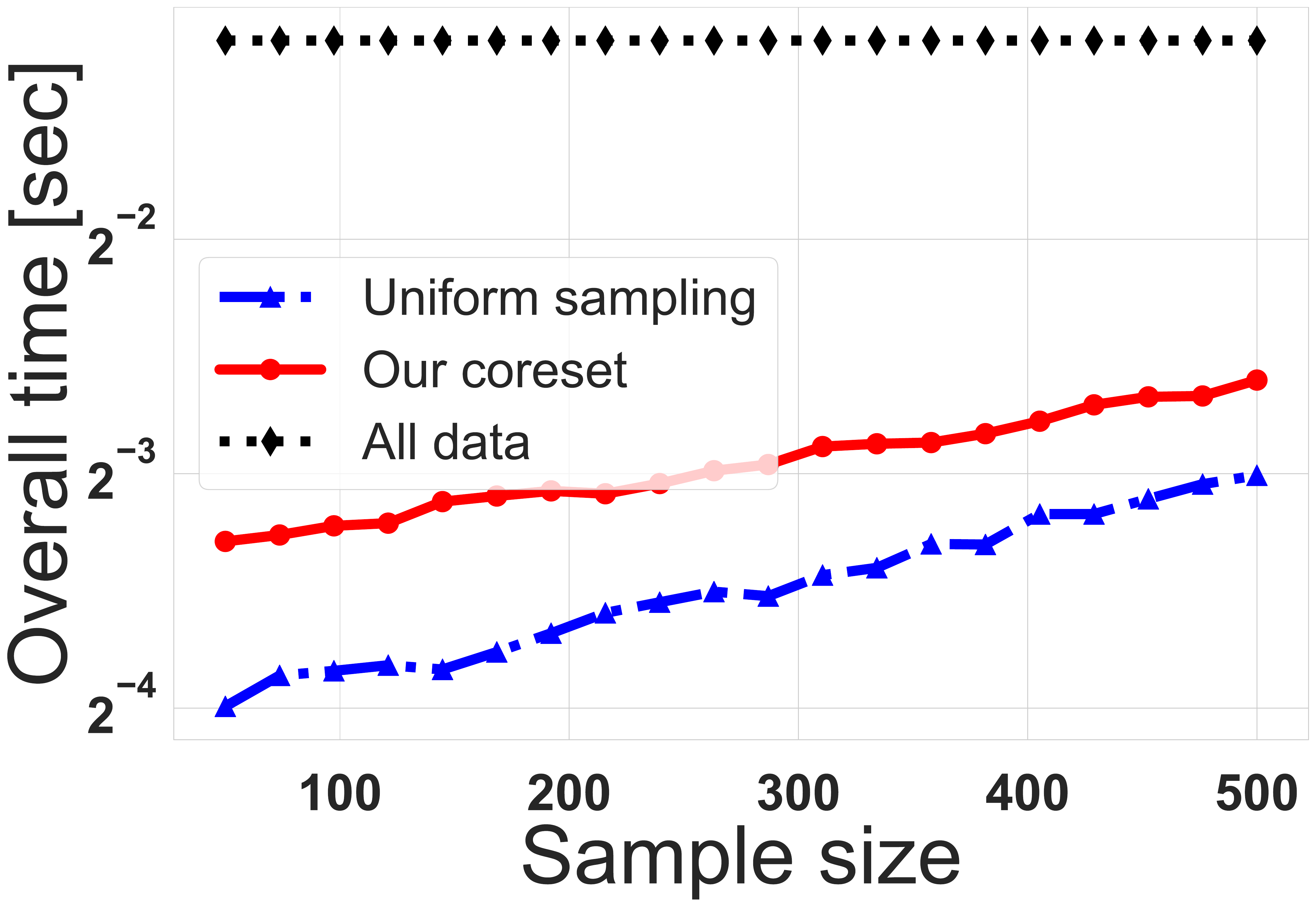}
\caption{Dataset (iii): Logistic regression}
\label{fig:Codlogistic}
\end{subfigure}

%%%%%%%%%%%%%%%%%%%%%%%%%%% W8a %%%%%%%%%%%%%%%%%%%
\begin{subfigure}[t]{0.49\textwidth}
\includegraphics[width=.49\textwidth]{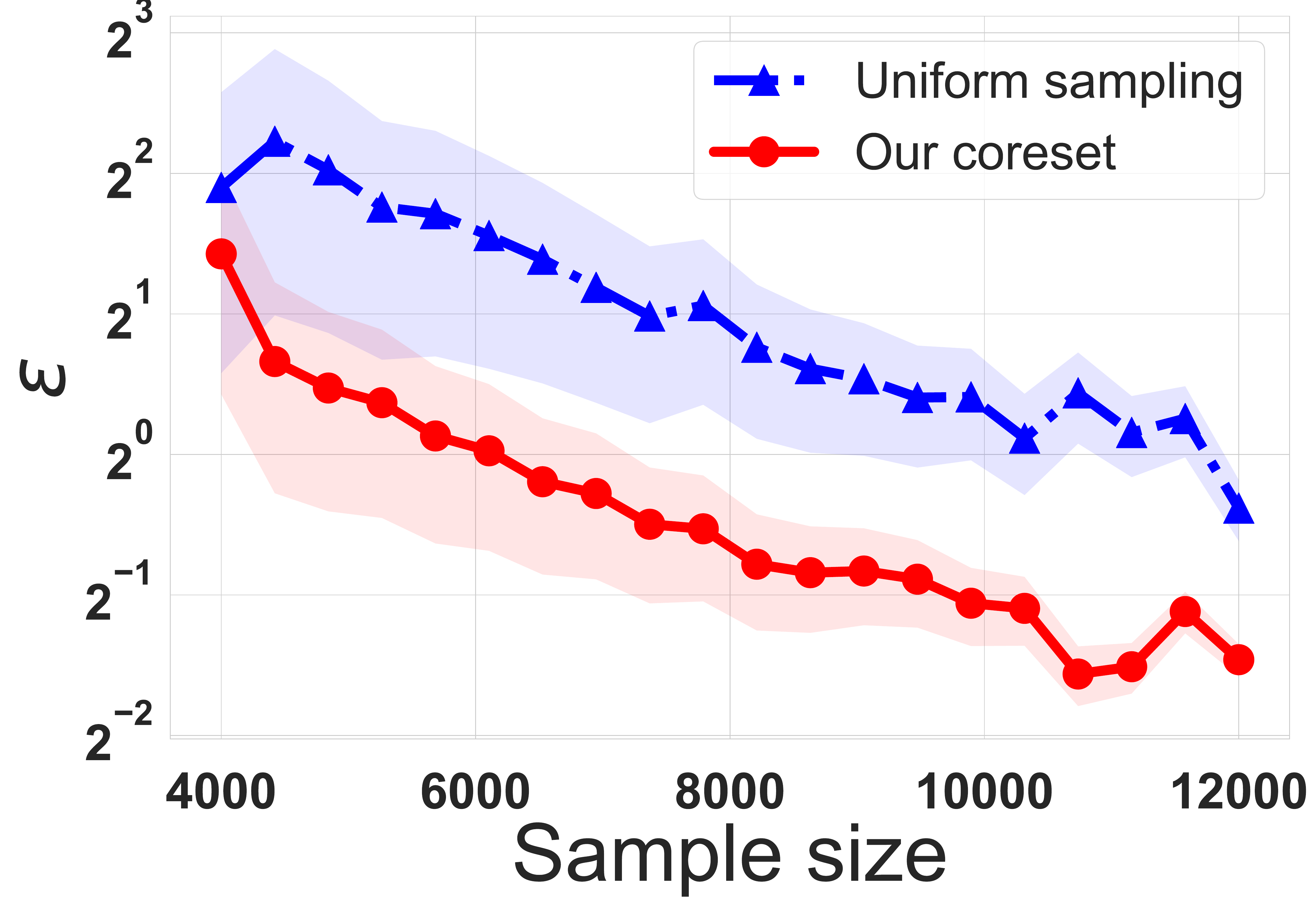}
\includegraphics[width=.49\textwidth]{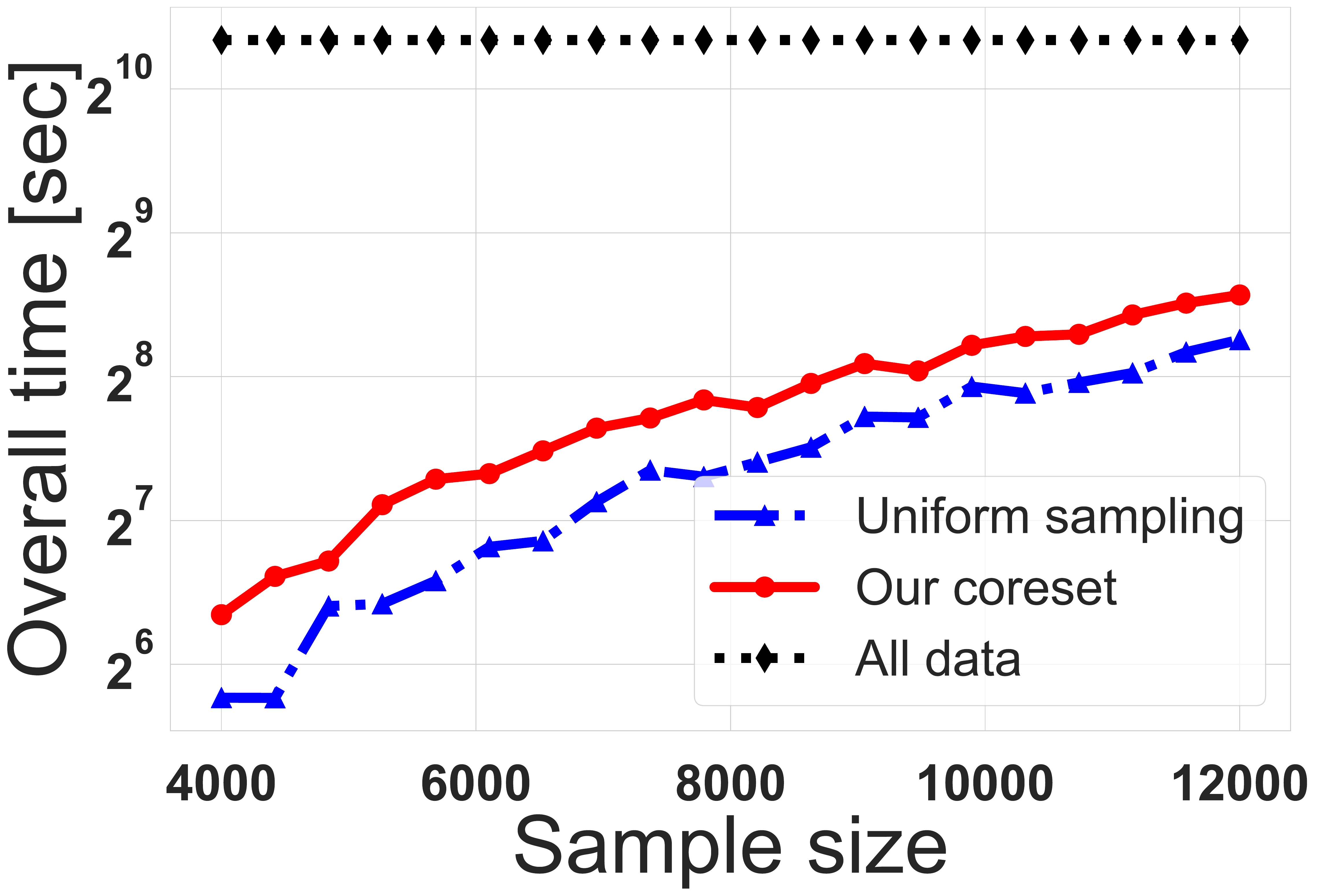}
\caption{Dataset (iv): Support vector machines}
\label{fig:w8asvm}
\end{subfigure}
\begin{subfigure}[t]{0.49\textwidth}
\includegraphics[width=.49\textwidth]{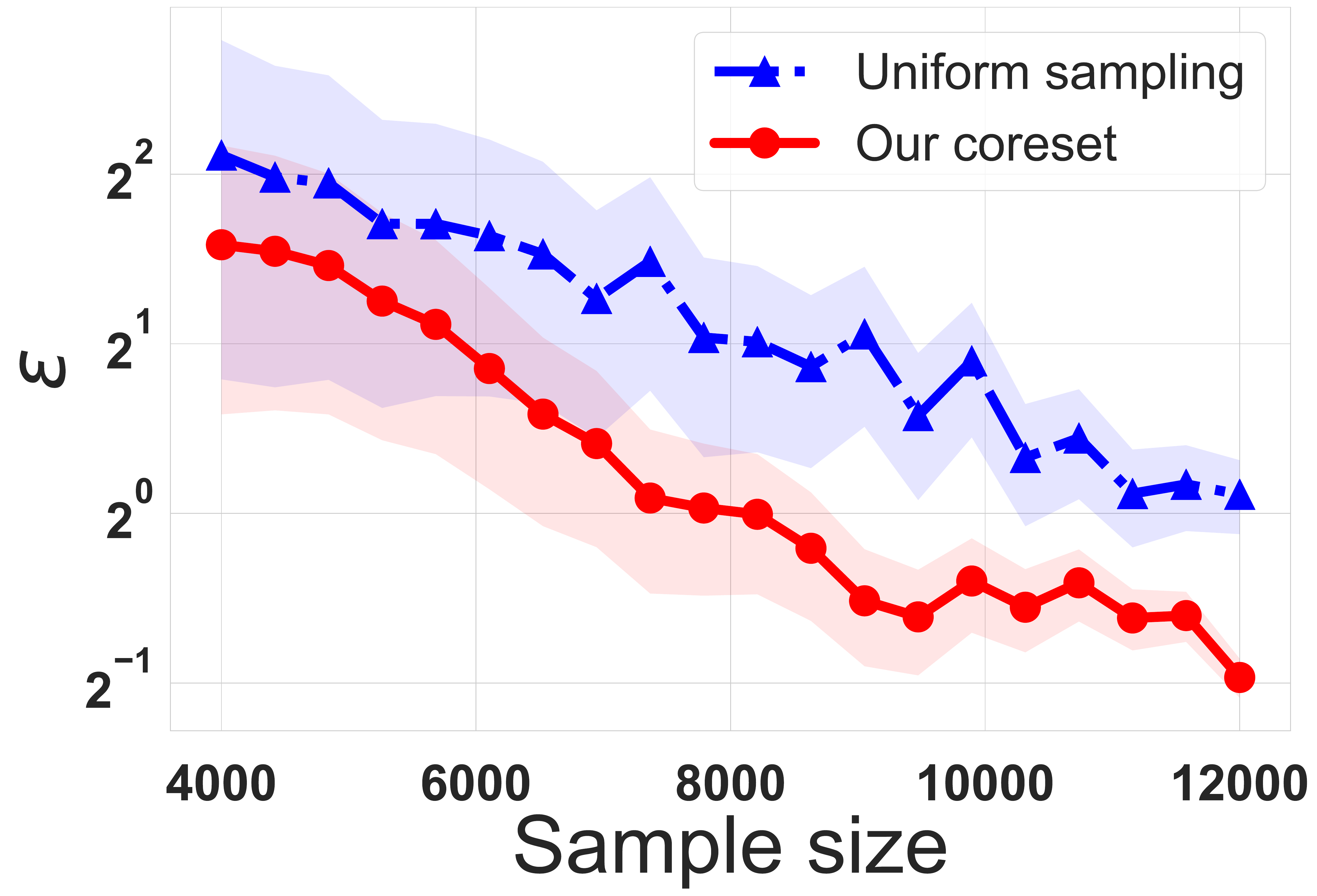}
\includegraphics[width=.49\textwidth]{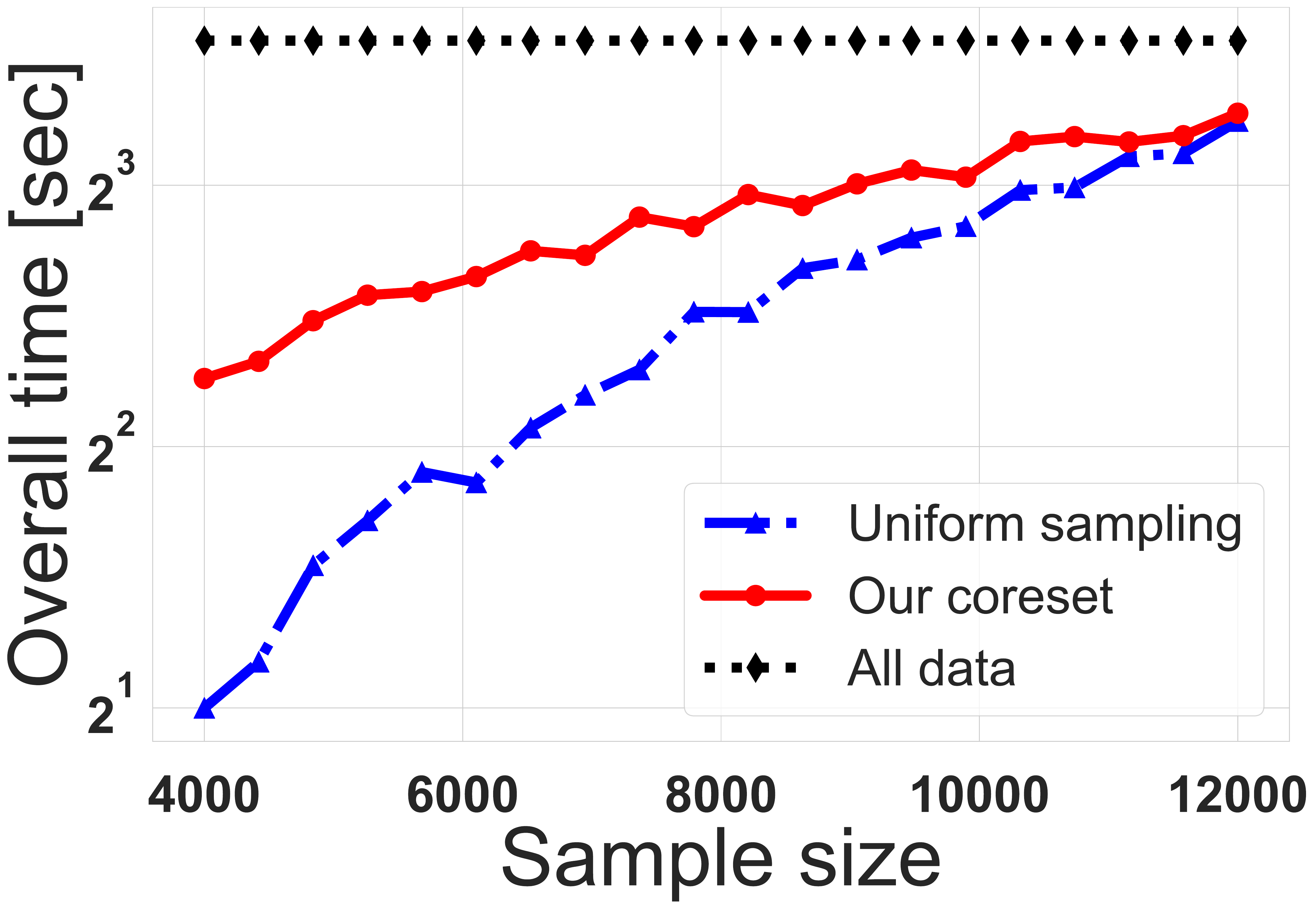}
\caption{Dataset (iv): Logistic regression}
\label{fig:w8alogistic}
\end{subfigure}

%%%%%%%%%%%%%%%%%%%%%%%%%%%%%%% 3D network
\begin{subfigure}[t]{0.49\textwidth}
\includegraphics[width=.49\textwidth]{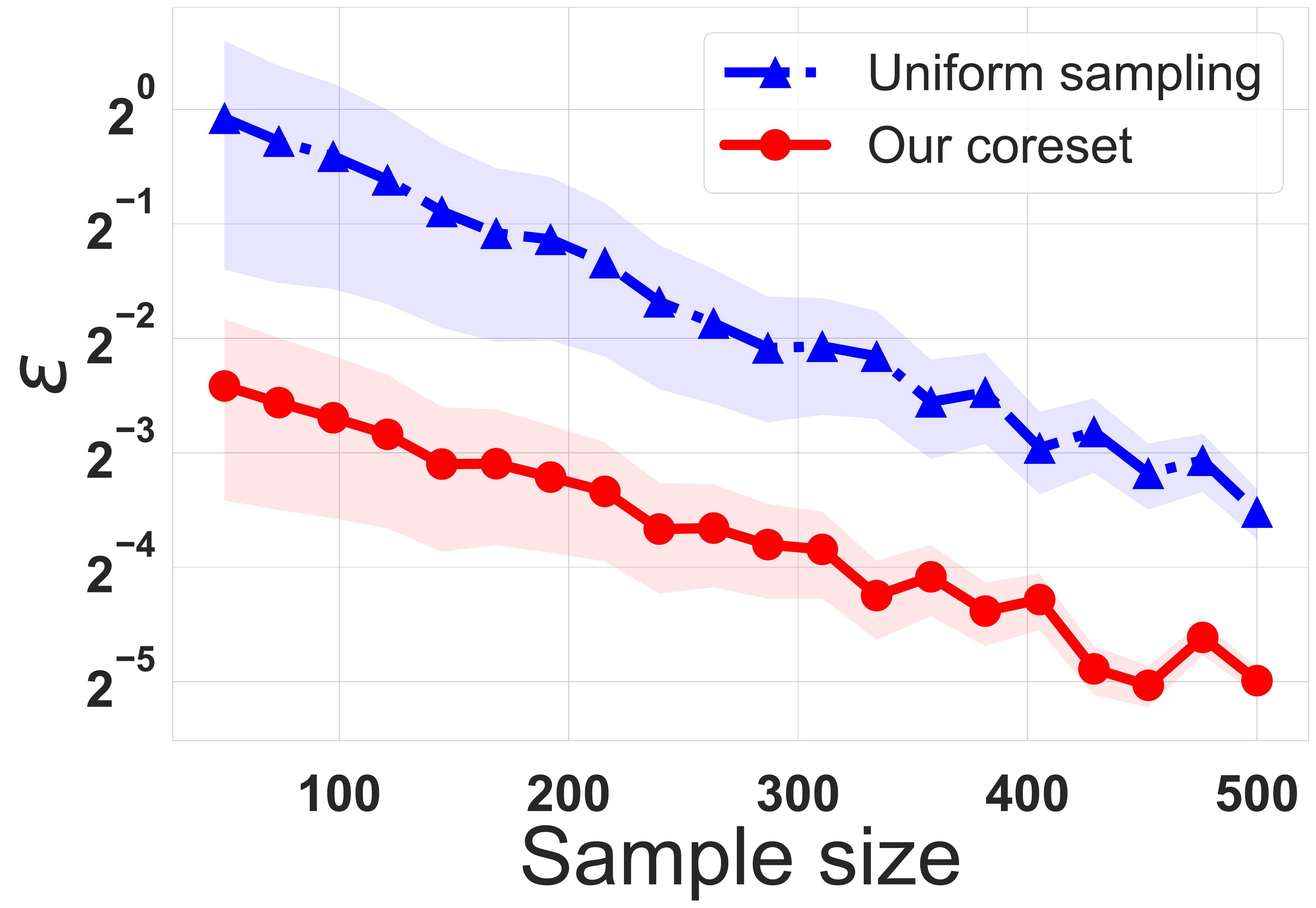}
\includegraphics[width=.49\textwidth]{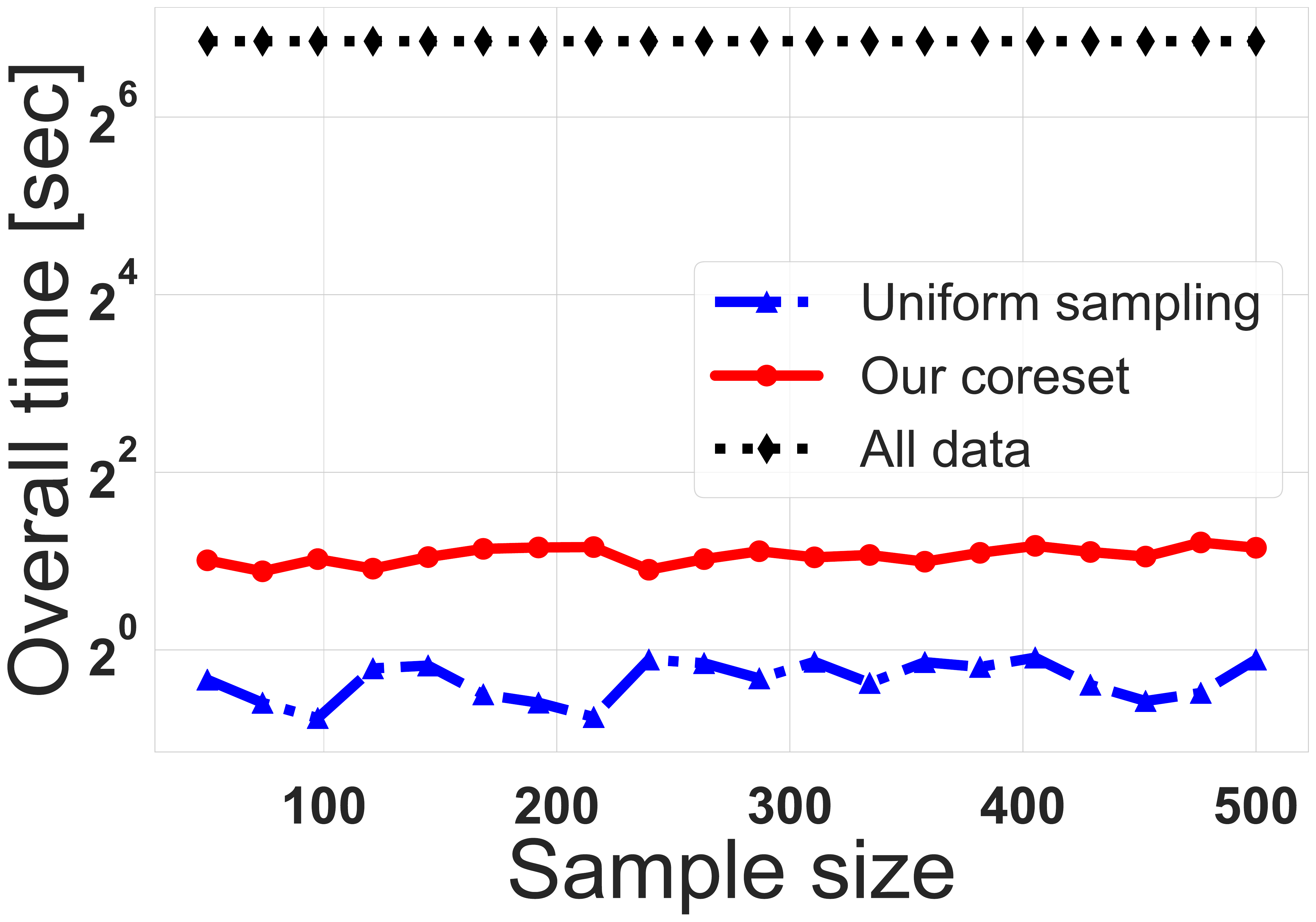}
\caption{Dataset (v): $\ell_{0.5}$-regression}
\label{fig:l05reg}
\end{subfigure}
\begin{subfigure}[t]{0.49\textwidth}
\includegraphics[width=.49\textwidth]{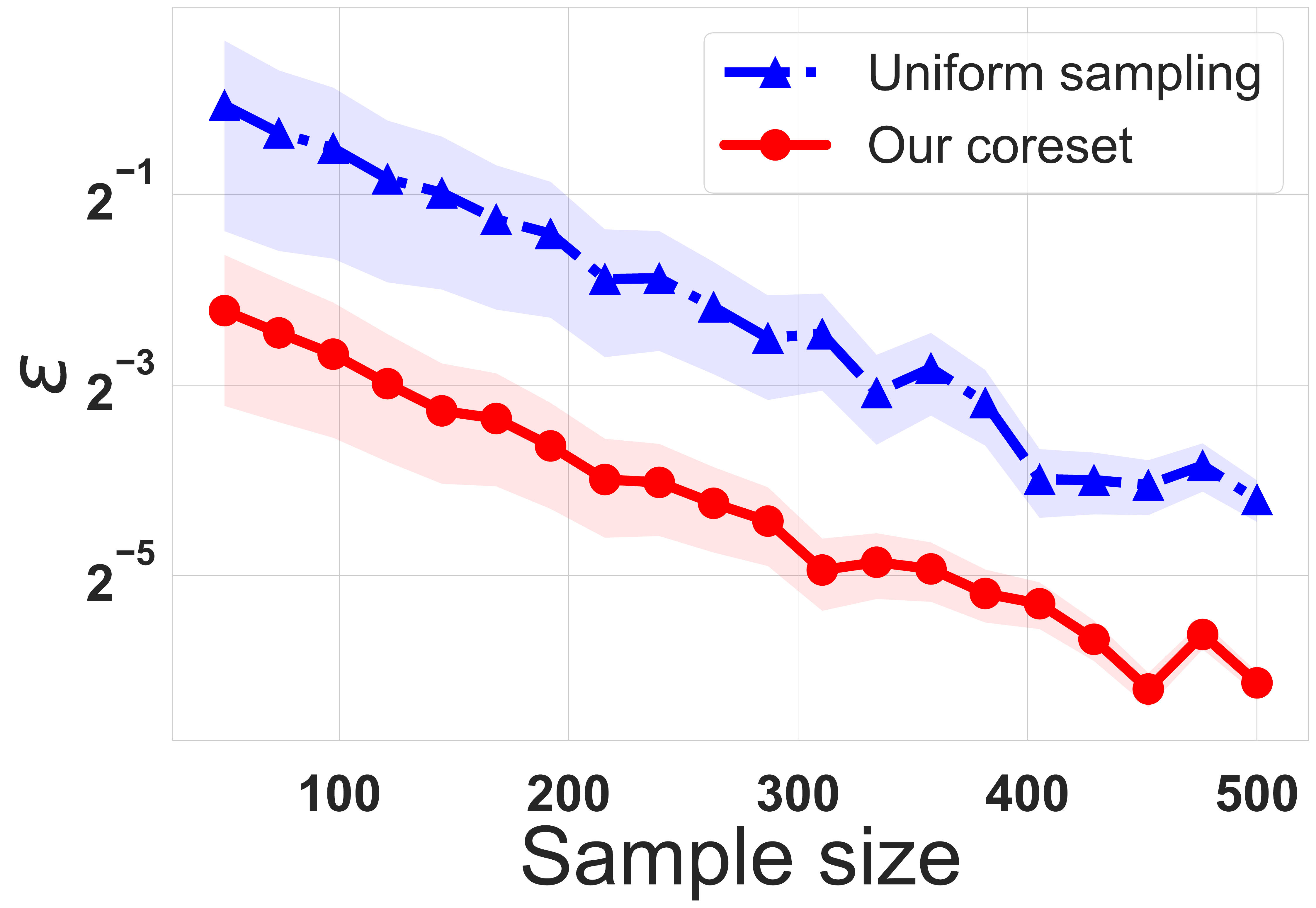}
\includegraphics[width=.49\textwidth]{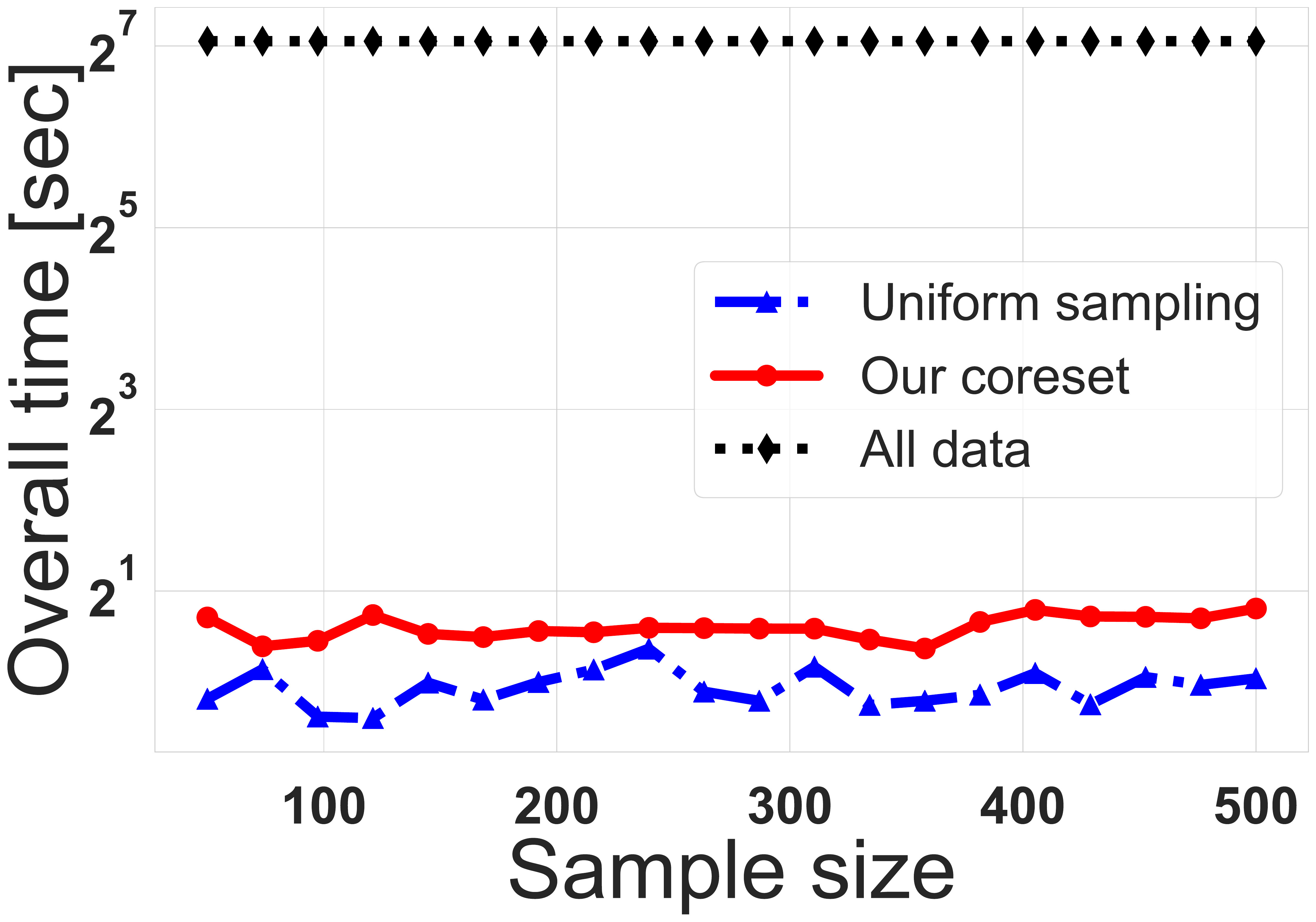}
\caption{Dataset (v): $\ell_{0.8}$-regression}
\label{fig:l08reg}
\end{subfigure}

\caption{Experimental results}\label{fig:vecsum}

\end{figure*}

\section{Conclusions and open problems}
\label{sec:conclutions}
In this paper, we have provided what we call the $f$-SVD of $P$ with respect a given near-convex loss function $f \in \set{F}$, as well as sensitivity bounding framework using the $f$-SVD. What interests us is to draw back forcing $f$ to have a centrally symmetric level set as well as embedding the center of the L\"{o}wner ellipsoid into the sensitivity bound. This is crucial step for generalizing the framework towards a much broader family of functions, e.g., loglog-Lipschitz functions~\cite{feldman2012data}. We are aware that for $\ell_z$-regression problems where $z \geq 1$, Lewis weights have been used by~\cite{cohen2015p} and are considered to be the state of the art coreset for these problems. We aim to generalize the applicability of Lewis weights and other sketching techniques towards different functions, and as far as we know, we consider the above issues to be open problems.
% \small
%% This defines the bibliography file (main.bib) and the bibliography style.
%% If you want to create a bibliography file by hand, change the contents of
%% this file to a `thebibliography' environment.  For more information 
%% see section 4.3 of the LaTeX manual.
\bibliography{main}
\bibliographystyle{abbrv}

\appendix
\section{Generalization of our tools}
We first define the term \emph{query space} which will aid us in simplifying the proofs as well as the corresponding theorems.

\begin{Definition}[Query space]
\label{def:querySpace}
Let $P$ be a set of $n\geq1$ points in $\REAL^d$, $w : P \to [0,\infty)$ be a non-negative weight function, and let $f: P \times \Q \to [0,\infty)$ denote a loss function. The tuple $(P, w, \Q, f)$ is called a query space.
\end{Definition}

Our paper relies on using known theorems associated with convex loss functions to prove our technical results. Thus, for completeness we give a formal definition of a convex loss functions as follows.

\begin{Definition}[Convex loss function] \label{def:conv}
Let $P \subseteq \REAL^d$ be a set of $n$ points, and let $f : P \times \Q \to [0, \infty)$ be a loss function. We say that $f$ is a convex loss function if for every $p \in P$, $f(p,\cdot) : \Q \to [0, \infty)$ is a convex function i.e., for every $\theta \in [0,1]$ and every $x,y \in \Q$ $$f\left(p,\theta x + (1-\theta) y \right) \leq \theta f(p,x) + (1-\theta) f(p,y).$$
\end{Definition}

Below, we present a straightforward generalization of the properties in Definition~\ref{def:familyOfConvexFuncs}, is applied to grasp much more variety of functions, by taking the weights into account and not setting them to $1$ for every point in the input set of points as well as other properties.

\begin{Definition}[Generalization of Definition~\ref{def:familyOfConvexFuncs}]
\label{def:familyOfConvexFuncs_Gen}
Let $(P,w,\Q,f)$ be a query space, where $f : P \times \Q \to [0,\infty)$ is a loss function. We call $f$ a near-convex loss function if there exists a convex loss function $g : P \times \Q \to [0, \infty)$, a function $h : P \times \Q \to [0,\infty)$ and a scalar $z > 0$ that satisfies:
\begin{enumerate}[label=(\roman*)]

\item \label{prop:reductionToFam} There exist $c_1,c_2 \in (0,\infty)$ such that for every $p \in P$, and $x \in \Q$,
\[
c_1 \term{g(p,x)^z + h(p,x)^z} \leq f(p,x) \leq c_2 \term{ g(p,x)^z+ h(p,x)^z}.
\]

\item \label{prop:def} There exist $c_3,c_4 \in (0,\infty)$ such that for every $p \in P$, $x \in \Q$ and $b \in (0,\infty)$,
\[
c_3 b g(p,x) \leq g(p,bx) \leq c_4 b g(p,x).
\]

\item \label{prop:h_trait} There exists $c_5\in(0,\infty)$ such that for every $p \in P$ and $x \in \Q$, $$
\frac{w(p)h(p,x)^z}{\sum_{q \in P} w(q) h(p,x)^z} \leq \frac{c_5 w(p)}{\sum_{q \in P} w(q)}.
$$

% \item If $z \in (0,1]$ then it holds that
% \[
% \term{\sum\limits_{p \in P} w(p)^{\frac{1}{z}} \term{g(p,x) + h(p,x)}}^z \leq \sum\limits_{p \in P} w(p) f(p,x) \leq n^{\frac{1}{z}} \term{\sum\limits_{p \in P} w(p)^{\frac{1}{z}} \term{g(p,x) + h(p,x)}}^z.
% \]

\item \label{prop:defLvl} The set $\mathcal{X}_g = \br{x \in \Q \middle| \sum_{p \in P }  w(p)^{\maxArgs{1, \frac{1}{z}}} g(p,x)^{\maxArgs{1,z}} \leq 1}$ is centrally symmetric, i.e., for every $x \in \mathcal{X}_g$ we have $-x \in \mathcal{X}_g$, and there exist $R,r \in (0,\infty)$ such that
$
B(0_d,r) \subset \mathcal{X}_g \subset B(0_d,R),
$
where $B(x,y)$ denotes a ball of radius $y > 0$, centered at $x \in \Q$.
\end{enumerate}
We denote by $\set{F}$ the union of all functions $f$ with the above properties.
\end{Definition}

Due to such changes, we also give a generalization towards the definition of $f$-SVD, as in what follows.

\begin{Definition}[Generalization of Definition~\ref{def:svdWRTF}]
\label{def:svdWRTF_Gen}
Let $(P,w,\Q,f)$ be a query space, such that $f \in \set{F}$, and let $g,h,c_1,c_2,c_3,c_4,z$ be defined as in the context of Definition~\ref{def:familyOfConvexFuncs} with respect to $f$. Let $D,V \in \REAL^{d \times d}$ be a diagonal matrix and an orthogonal matrix respectively, and let $\alpha \in \Theta\left( \sqrt{d} \right)$ such that for every $x \in \Q$,
\[
c_1 \term{\term{c_3 \norm{DV^Tx}_2}^z + \sum_{p \in P} w(p) h(p,x)^z} \leq \sum_{p \in P} w(p) f(p,x),
\]
and
\[
\sum_{p \in P} w(p)^{\maxArgs{1, \frac{1}{z}}} g(p,x)^{\max\br{1, z}} \leq \term{c_4 \alpha \norm{DV^Tx}_2}^{\max\br{1, z}}.
\]
Let $U:P\to \Q$ such that $U(p) = \left( V D \right)^{-1} p$ for every $p\in P$. The tuple $\left( U,D,V\right)$ is the $f$-SVD of $(P,w)$.
\end{Definition}

\section{VC dimension}
\label{sec:appendix_VC}

\begin{definition}[VC-dimension~\cite{braverman2016new}]
\label{def:dimension}
For a query space $(P,w,\Q,f)$ and $r \in [0,\infty)$, we define
\[
\RANGES(x,r) = \br{p \in P \mid w(p) f(p,x) \leq r},
\]
for every $x \in \Q$ and $r \geq 0$. The dimension of $(P, w, \Q, f)$ is the size $\abs{S}$ of the largest subset $S \subset P$ such that
\[
\abs{\br{S \cap \RANGES(x,r) \mid x \in \Q, r \geq 0 }} = 2^{\abs{S}},
\]
where $\abs{ A }$ denotes the number of points in $A$ for every $A \subseteq \REAL^d$.
\end{definition}

\section{Existence of $f$-SVD factorization}
\begin{restatable}{lemma}{LownerToF}
\label{lem:LownerToF}
Let $(P,w,\Q,f)$ be a query space, such that $f \in \set{F}$. Let $g,h,c_1,c_2,c_3,c_4,z$ be defined as in the context of Definition~\ref{def:familyOfConvexFuncs_Gen} with respect to $f$, $\alpha \in \Theta \left( \sqrt{d} \right)$ and let $\beta =\maxArgs{1,z}$. Then there exists a diagonal matrix $D \in \REAL^{d \times d}$ and an orthogonal matrix $V \in \REAL^{d \times d}$ such  that for every $x \in \REAL^d$,
\begin{equation}\label{upper_bound}
\sum\limits_{p \in P} w(p)^{\maxArgs{1,\frac{1}{z}}} g(p,x)^{\beta} \leq \term{c_4 \alpha \norm{DV^Tx}_2}^{\beta},
\end{equation}
and
\begin{equation}\label{lower_bound}
c_1 \term{\term{c_3 \norm{DV^Tx}_2}^z + \sum_{p \in P} w(p) h(p,x)^z} \leq \sum_{p \in P} w(p) f(p,x).
\end{equation}
\end{restatable}

\begin{proof} We prove Lemma~\ref{lem:LownerToF} using L\"{o}wner ellipsoid; See~\cite{john2014extremum}.
\paragraph{Using L\"{o}wner ellipsoid.} Let $\mathcal{X}_g = \br{x \in \Q \middle| \term{\sum_{p \in P}  w(p)^{\maxArgs{1, \frac{1}{z}}} g(p,x)^{\beta}}^{\frac{1}{\beta}} \leq 1}$. Since $f\in F$ (see Definition~\ref{def:familyOfConvexFuncs_Gen}), and $g$ is convex, we have that \begin{enumerate*}[label=(\roman*)]
\item $\set{X}_{g}$ is a convex set, and
\item $\set{X}_{g}$ is centrally symmetric.
\end{enumerate*}
Then by Theorem \rom{3} of~\cite{john2014extremum}, there exists an ellipsoid $E$, known as the L\"owner ellipsoid that is centered at the origin $0_d$, such that
\begin{equation}\label{lowner}
\frac{1}{\sqrt{d}}E\subseteq X_{g} \subseteq E,
\end{equation}
where $\frac{1}{\sqrt{d}}E$ denotes the set $\br{\frac{1}{\sqrt{d}}x \mid x\in E}$.
%which is the dilation of $E$ by a factor of $\frac{1}{\sqrt{d}}$.

By combining Property~\ref{prop:defLvl} of Definition~\ref{def:familyOfConvexFuncs_Gen} with~\eqref{lowner}, there exists $r \in (0,\infty)$ such that $B(0_d,r) \subseteq \mathcal{X}_g \subseteq E$. Since $B(0_d,r) \subseteq E$, then
$E$ is an ellipsoid where each of its axes has positive length. By that, there exists a diagonal matrix $D \in \REAL^{d \times d}$ of positive entries and an orthogonal matrix $V \in \REAL^{d \times d}$ such that,
\begin{enumerate*}[label=(\roman*)]
\item $E = \br{y\in \Q \mid \norm{DV^Ty}_2 \leq 1}$, and \label{Edef}
\item $V^T D D V$ is a positive semi-definite matrix.
\end{enumerate*}

%What is left to be done, is to derive~\eqref{Gx}.
Put $x \in \Q$ and now we proceed to derive the bounds.

\paragraph{Proving~\eqref{upper_bound}.} Let $y = \frac{1}{\norm{DV^Tx}_2}x$. By the definition of $E$ in~\ref{Edef} and the definition of $y$, we have that $y \in E$, and by combining~\eqref{lowner} with the assumption that $\alpha \in \Theta\left(\sqrt{d}\right)$ we obtain that $y \in E \subseteq \alpha X_{g}$.
Then
\begin{equation}\label{bounding}
\frac{1}{\alpha} y \in \frac{1}{\alpha}E \subseteq X_{g},
\end{equation}
which consequently leads to
\begin{equation}
\label{eq:bounding_y}
\sum_{p \in P}  w(p)^{\maxArgs{1, \frac{1}{z}}} g\term{p,\frac{1}{\alpha} y}^{\beta} \leq 1,
\end{equation}
where the inequality holds by Property~\ref{prop:defLvl} of Definition~\ref{def:familyOfConvexFuncs_Gen} with respect to $g$.
Hence,
\begin{equation}
\label{eq:upperbound_Gx}
\begin{split}
\sum\limits_{p \in P} w(p)^{\maxArgs{1,\frac{1}{z}}} g(p,x)^{\beta} &\leq \term{c_4\alpha \norm{DV^Tx}_2}^\beta \sum\limits_{p \in P} w(p)^{\maxArgs{1,\frac{1}{z}}} g\term{p,\frac{x}{\alpha\norm{DV^Tx}_2}}^{\beta} \\
&\leq \term{c_4 \alpha \norm{DV^Tx}_2}^{\beta},
\end{split}
\end{equation}
where the first inequality is by substituting $b := \alpha \norm{DV^Tx}_2$ and $x := \frac{x}{b}$ in Property~\ref{prop:def} of $g$ (see Definition~\ref{def:familyOfConvexFuncs_Gen}), and the second inequality is by combining the fact that $y=\frac{1}{\norm{DV^Tx}_2}x$ with~\eqref{eq:bounding_y}.

\paragraph{Proving~\eqref{lower_bound}.}
Let $b^\prime = \alpha\norm{DV^Tx}_2$. By~\eqref{eq:bounding_y}, we get that $\sum_{p \in P} w(p)^{\maxArgs{1,\frac{1}{z}}} g\term{p,\frac{1}{b^\prime}x}^{\beta} \geq 1$. In addition, Property~\ref{prop:defLvl} of Definition~\ref{def:familyOfConvexFuncs_Gen} states that every vector in $\Q$ of norm $r$ is inside $\set{X}_g$. Thus, there exists $b \geq b^\prime$ such that $\sum_{p \in P} w(p)^{\maxArgs{1,\frac{1}{z}}} g\term{p,\frac{1}{b}x}^{\beta} = 1$.

By~\eqref{lowner}, we have that $\norm{DV^Tx}_2 = b \norm{DV^Tz}_2 \leq b$ where $z = \frac{1}{b}x$. Hence, by plugging $x := z$ and $b := b$ in Property~\ref{prop:def} of Definition~\ref{def:familyOfConvexFuncs_Gen}, we obtain that
\begin{equation}
\label{eq:lowerbound_Gx}
\sum_{p \in P} w(p)^{\maxArgs{1,\frac{1}{z}}} g\term{p,x}^{\beta} \geq c_3 b \sum_{p \in P} w(p)^{\maxArgs{1,\frac{1}{z}}} g\term{p, z}^{\beta} = \term{c_3b}^\beta \geq \term{c_3\norm{DV^Tx}_2}^\beta.
\end{equation}

By combining Property~\ref{prop:reductionToFam} of Definition~\ref{def:familyOfConvexFuncs_Gen}, with~\eqref{eq:upperbound_Gx} and~\eqref{eq:lowerbound_Gx}, Lemma~\ref{lem:LownerToF} holds.
\end{proof}

\section{Extension towards Streaming and distributed settings}
\label{sec:appendix_Stream}
\begin{algorithm}[!htb]
\caption{$\streamCore(P,w,f,l,\eps,\delta)$\label{alg:streamAlg}}
\label{algorithmStreams}
{\begin{tabbing}
\textbf{Input:} \quad\quad\= A set $P \subseteq \REAL^{d}$ of $n$ points, a weight function $w: \REAL^d \to [0,\infty)$, a leaf size $\ell > 0$\\
\>a convex loss function function $f : \REAL^d \to [0,\infty)$, an error parameter $\eps \in (0,1)$,\\\>and probability $\delta\in(0,1)$.\\
\textbf{Output:} \>A pair $(S,v)$ which is an $\left(h\right)$-coreset for $(P,\Q,X,f)$,\\
\>with probability of at least $1-\delta h$.
\end{tabbing}}
\vspace{-0.3cm}
$B_i \gets \emptyset$ for every $1 \leq i \leq \infty$ \label{alg2:L1} \\
$h \gets 1$ \label{alg2:L2}\\
\For{\label{alg2:L3} each  set $Q$ of consecutive $2\ell$ points from $P$}{
    $(T,v) := \coreset(Q, w, f, \frac{\eps}{2\log{n}},\frac{\delta}{2\log{n}})$ \label{alg2:LCore1}\\
    $j \gets 1$\\
    $B_j := B_j \cup (T,v)$ \label{alg2:L4} \\
    \For{\label{alg2:L5} each $j \leq h$}{
    	\While{$\abs{B_j} \geq 2$}{
        	$(T_1,v_1) , (T_2,v_2) := $ pop first pair of consecutive items in $B_j$ \label{alg2:LCore2}\\
        	%\STATE $(T_1,u_1) , (T_2,u_2) \gets $ be the pop of the first two items in $B_j$;
        	For every $p \in T_1 \cup T_2$, set $v'(p) := \begin{cases}
        	v_1(p) & p \in T_1,\\
        	v_2(p) & \text{Otherwise}
        	\end{cases}$ \\
        	$(T,v) := \coreset(T_1 \cup T_2, v', f, \frac{\eps}{2\log{n}},\frac{\delta}{2\log{n}})$\\
            $B_{j+1} := B_{j+1} \cup (T,v)$\label{alg2:L6}\\
            $h := \max\br{h, j+1}$ \label{alg2:L7}\\
       	}
    }
}
$(S,v) := B_h$ \label{alg2:L8}\\
\Return $(S,v)$ \label{alg2:L9}
\end{algorithm}

Algorithm~\ref{alg:mainAlg} can be easily extended towards streaming and distributed settings as presented at Algorithm~\ref{alg:streamAlg}.
At the beginning, the data arrives in a streaming fashion, e.g. in batches, where our coreset scheme (see Algorithm~\ref{alg:mainAlg}) is applied on each of these batches. When we have two $\eps$-coresets in memory, we merge them and an $\eps$-coreset is constructed upon their merge. This procedure is done until \begin{enumerate*}[label=(\roman*)] \item there is no points left in the stream and \item there is exactly one coreset left in memory. \end{enumerate*}

Algorithm~\ref{alg:streamAlg} begins with initializing the batches to an empty sets as well as setting the height of the tree to $1$; see lines~\ref{alg2:L1}--\ref{alg2:L2}. In what follows, for each $2l$ of streamed points, we generate an $\eps$-coreset on this set as presented at lines~\ref{alg2:L3}--\ref{alg2:L4}. Lines~\ref{alg2:L5}--\ref{alg2:L7} depict the core of the merge-and-reduce tree, which is the binary tree building fashion from the leaves (the incoming batches) towards the root of the tree. Finally, we return the root of the tree as shown at lines~\ref{alg2:L8}--\ref{alg2:L9}.
For much broader and detailed explanation regarding the merge-and-reduce tree, we refer the reader towards~\cite{braverman2016new}.

\subsection{From sublinear to poly-logarithmic coreset size}
\label{sec:thm-polylog-coreset-merge-reduce-proof}

\begin{restatable}[Variant of Lemma 4, ~\cite{tukan2020coresets}]{lemma}{mergeandreduce}
\label{thm:polylog-coreset-merge-reduce}
Let $P \subseteq \REAL^d$ be a set of $n$ points, and let $f \in \set{F}$ be a near-convex loss function. Let $\eps \in \left[ \frac{1}{\log{n}},\frac{1}{2} \right]$, $\delta \in \left[\frac{1}{\log{n}},1 \right)$ and let $t$ denote the total sensitivity from Lemma~\ref{thm:sensitivityBound}. Suppose that there exists some $\beta \in (0.1, 0.8)$ such that $t \in \Theta(n^\beta)$ and let $\ell \geq 2^{\frac{\beta}{1-\beta}}$. Let $(S,v)$ be the output of a call to $\streamCore(P,w,f,\ell,\eps,\delta)$. Then $(S,v)$ is an $\eps$-coreset of size
\[
\abs{S} \in  \left( \log{n} \right)^{\bigO{1}}.
\]
\end{restatable}

% \setcounter{myCounter}{\getrefnumber{lem:LownerToF}}

% \setcounter{theorem}{\getrefnumber{thm:polylog-coreset-merge-reduce}-1}
% \mergeandreduce*
\begin{proof}
First we note that using Theorem~\ref{thm:runTime} on each node in the merge-and-reduce tree, would attain that the root of the tree, i.e., $(S,v)$ attains that for every $w$
\[
(1-\eps)^{\log{n}} \sum\limits_{p \in P} w(p) f(p,x) \leq \sum\limits_{p \in S} v(p) f(p,x)\leq (1+\eps)^{\log{n}} \sum\limits_{p \in P} w(p) f(p,x),
\]
with probability at least $(1-\delta)^{\log{n}}$.

We observe by the properties of the natural number $e$,
\begin{equation*}
(1+\eps)^{\log{n}} = \left(1 + \frac{\eps \log{n}}{\log{n}} \right)^{\log{n}} \leq e^{\eps\log{n}},
\end{equation*}
which when replacing $\eps$ with $\eps^\prime = \frac{\eps}{2 \log{n}}$ in the above inequality as done at Lines~\ref{alg2:LCore1} and~\ref{alg2:LCore2} of Algorithm~\ref{algorithmStreams}, we obtain that
\begin{align}
    (1+\eps^\prime)^{\log{n}} \leq e^{\frac{\eps}{2}} \leq 1 + \eps,
\end{align}
where the second inequality holds since $\eps \in [\frac{1}{\log{n}}, \frac{1}{2}]$.

As for the lower bound, observe that
\[
(1-\eps)^{\log{n}} \geq 1 - \eps\log{n},
\]
where the inequality holds since $\eps \in [\frac{1}{\log{n}}, \frac{1}{2}]$.

Hence,
\[
(1-\eps^\prime)^{\log{n}} \geq 1 - \eps^\prime\log{n} = 1 - \frac{\eps}{2} \geq 1 - \eps.
\]

Similar arguments holds also for the failure probability $\delta$. What is left for us to do is setting the leaf size which will attain us an $\eps$-coreset of size poly-logarithmic in $n$ (the number of points in $P$).

Let $\ell \in (0,\infty)$ be the size of a leaf in the merge-and-reduce tree. We observe that a coreset of size poly-logarithmic in $n$, can be achieved by solving the inequality
\begin{align*}
\frac{2\ell}{2} \geq (2\ell)^\beta,
\end{align*}
which is invoked when ascending from any two leafs and their parent node at the merge-and-reduce tree.

Rearranging the inequality, we yield that
\begin{align*}
\ell^{1-\beta} \geq 2^\beta.
\end{align*}

Since $\ell \in (0,\infty)$, any $\ell \geq \sqrt[1-\beta]{2^\beta}$ would be sufficient for the inequality to hold. What is left for us to do, is to show that when ascending through the merge-and-reduce tree from the leaves towards the root, each parent node can't be more than half of the merge of it's children (recall that the merge-and-reduce tree is built in a binary tree fashion, as depicted at Algorithm~\ref{algorithmStreams}).

Thus, we need to show that,
\begin{align*}
2^{\sum\limits_{j=1}^{i} \beta^j} \cdot \ell^{\beta^i} \leq \frac{2^{\sum\limits_{k=0}^{i-1} \beta^k} \cdot \ell^{\beta^{i-1}}}{2} = 2^{\sum\limits_{k=1}^{i-1} \beta^k} \cdot \ell^{\beta^{i-1}},
\end{align*}
holds, for any $i \in \left[\ceil{\log{n}} \right]$ where $\log{n}$ is the height of the tree. Note that the left most term is the parent node's size and the right most term represents half the size of both parent's children nodes.

In addition, for $i = 1$, the inequality above represents each node which is a parent of leaves. Thus, we observe that for every $i \geq 1$, the inequality represents ascending from node which is a root of a sub-tree of height $i-1$ to it's parent in the merge-and-reduce tree.

By simplifying the inequality, we obtain the same inequality which only addressed the leaves. Hence, by using any $\ell \geq 2^{\frac{\beta}{1-\beta}}$ as a leaf size in the merge and reduce tree, we obtain an $\eps$-coreset of size poly-logarithmic in $n$.
\end{proof}

\section{Proofs for the Main Theorems}
\label{sec:appendix_proofs}

Throughout this section, we will present generalized versions of the lemmata and theorems that are presented at Section~\ref{sec:method} and Section~\ref{sec:analysis}.

\setcounter{theorem}{17}
\subsection{Generalization of Lemma~\ref{thm:sensitivityBound}}
\begin{lemma}[Equivalence of norms,~\cite{rudin1991functional}]
\label{thm:equivalenceNorms}
Let $a,b > 0$ such that $a \leq b$. Then for every $x \in \REAL^d$,
\[
\norm{x}_b \leq \norm{x}_a \leq d^{\frac{1}{a} - \frac{1}{b}} \norm{x}_b.
\]
\end{lemma}

\begin{restatable}{claim}{pRootSum}[Result of H\"{o}lder's Inequality]
\label{lem:pRootSum}
Let $\br{a_i}_{i=1}^n$ be a set of $n$ non-negative numbers, $z \in (0,1)$ be a real number. Then
\[
\sum\limits_{i=1}^n |a_i|^z \leq n^{1-z} \left(
\sum\limits_{i=1}^n |a_i| \right)^z.
\]
\end{restatable}

\begin{proof}
Let $z^\prime = \frac{1}{z}$ and for every $i \in [n]$, let $\hat{a_i} = |a_i|^z$. Let $e \in \left[ 1\right]^n$. We have
\[
\sum\limits_{i=1}^n |a_i|^z = \sum\limits_{i=1}^n \hat{a_i} \leq \norm{e}_{\frac{1}{1-z}} \left(\sum\limits_{i=1}^n \hat{a_i}^{z^\prime} \right)^{1/{z}^\prime} = n^{1-z}\left(\sum\limits_{i=1}^n \abs{a_i}\right)^z,
\]
where the first and last equalities are by definition of $\hat{a_i}$, and the inequality is by H\"{o}lder's inequality.
\end{proof}

\begin{lemma}
\label{thm:sensitivityBound:appendix}
Let $(P,w, \Q, f)$ be a query space (see Definition~\ref{def:querySpace}) such that $f \in \set{F}$ as in Definition~\ref{def:familyOfConvexFuncs}. Let $g,h,c_1,c_2,c_3,c_4,c_5,z$ be defined as in the context of Definition~\ref{def:familyOfConvexFuncs} with respect to $f$, $(U,D,V)$ be the $f$-SVD of $(P,w)$, and let $\alpha \in \Theta\term{\sqrt{d}}$ which satisfies the conditions in Definition~\ref{def:svdWRTF}. Suppose that there exists a set of $d$ unit vectors $\br{v_j}_{j=1}^d$ and $c \in (0,\infty)$, such that for every unit vector $y \in \REAL^d$ and $p \in P$,
\begin{equation*}
g(p,(D V^T)^{-1}y)^z \leq c \sum_{j=1}^{d} g(p,(D V^T)^{-1}v_j)^z.
\end{equation*}
Then, for every $p\in P$, the sensitivity of $p$ with respect to the query space $(P,w,\REAL^d,f)$ is bounded by
\begin{equation*}
s(p) \leq \term{\frac{2c_2c_5 w(p)}{c_1 \sum_{q \in P} w(q)}}^z + \frac{cc_2}{c_1c_3^{2z}}\sum_{j=1}^d w(p) \term{g\left( p, \left( DV^T\right)^{-1} v_j\right)}^z,
\end{equation*}
and the total sensitivity is bounded by
\[
\sum_{p \in P} s(p) \leq  \frac{c_2c_5}{c_1}+ \frac{c c_2 c_4^z}{c_1c_3^{2z}}\maxArgs{n^{1-z}, 1}\alpha^z d.
\]
\end{lemma}

\begin{proof}
%, and let $x \in \argsup{\substack{x \in \REAL^d }} \frac{w(p)f(p,x)}{\sum\limits_{q \in P} w(q) f(q,x)}$ such that the denominator in not zero
Let $n$ denote the number of points in $P$. Put $p \in P$, $x \in \Q$ such that $\sum\limits_{q \in P} w(q) f(q,x)>0$, and let $y=\frac{1}{\norm{DV^Tx}_2} DV^Tx$. We observe that,
\begin{align}
\frac{f(p,x)}{\sum\limits_{q \in P} w(q) f(q,x)} &\leq \frac{f(p,x)}{c_1\term{\term{c_3\norm{DV^Tx}_2}^z + \sum\limits_{q \in P} w(q) h(q,x)^z}} \label{eq:bounding_sens_1}\\
&\leq \frac{c_2 g(p,x)^z + c_2 h(p,x)^z}{c_1\term{c_3\norm{DV^Tx}_2}^z + c_1\sum\limits_{q \in P} w(q) h(q,x)^z}\label{eq:bounding_sens_2}\\
&\leq \frac{c_2g(p,x)^z}{c_1\term{c_3\norm{DV^Tx}_2}^z} + \frac{c_2h(p,x)^z}{c_1 \sum\limits_{q \in P}h(q,x)^z},\label{eq:bounding_sens_3}
\end{align}
where~\eqref{eq:bounding_sens_1} holds by Lemma~\ref{lem:LownerToF},~\eqref{eq:bounding_sens_2} holds by Property~\ref{prop:reductionToFam} of Definition~\ref{def:familyOfConvexFuncs} with respect to $f$, and the last inequality follows from plugging $a_1:=c_2 g(p,x)$, $r_1 := c_2 h(p,x)$, $a_2 := c_1 c_3 \norm{DV^Tx}_2$ and $r_2 := c_3 \sum_{q \in P} h(q,x)$ into Claim~\ref{clm:divOfPosSums}.

Note that when $h(q,z) = 0$ for every $q \in P$ and $z \in \Q$, then we obtain from~\eqref{eq:bounding_sens_2}, that the rightmost term of~\eqref{eq:bounding_sens_3} is zero.

We also have,
\begin{align}
\frac{1}{\norm{DV^Tx}_2^z} g\term{p, x}^z&\leq \frac{1}{c_3^z}g\term{p, \frac{x}{\norm{DV^Tx}_2}}^z \label{eq:boundingSens1} \\
&= \frac{1}{c_3^z} g\term{p, \left( DV^T\right)^{-1}y}^z \label{eq:boundingSens2}\\
&\leq \frac{c}{c_3^z}\sum\limits_{j=1}^d g\term{p, \term{ DV^T}^{-1} v_j}^z \label{eq:boundingSens3},
\end{align}
where~\eqref{eq:boundingSens1} follows from substituting $b := \frac{1}{\norm{DV^Tx}_2}$ and $x := \frac{x}{b}$ in Property~\ref{prop:def} of $f$ (see Definition~\ref{def:familyOfConvexFuncs}),~\eqref{eq:boundingSens2} holds since $x = (DV^T)^{-1}(DV^T)x$, and~\eqref{eq:boundingSens3} is by the assumption of Lemma~\ref{thm:sensitivityBound}.

By combining~\eqref{eq:bounding_sens_1}--\eqref{eq:boundingSens3} with Property~\ref{prop:h_trait} of $f$, the sensitivity of $p$ is bounded by
\begin{equation}
\label{eq:sensBound}
s(p) \leq \frac{c_2c_5 w(p)}{c_1 \sum_{q \in P} w(q)} + \frac{c c_2}{c_1c_3^{2z}}\sum\limits_{j=1}^d w(p)g\term{ p, \term{ DV^T}^{-1} v_j}^z.
\end{equation}

As for the total sensitivity, we first observe that if $z \in (0,1)$
\begin{align}
\sum\limits_{q \in P} \sum\limits_{j=1}^d w(q) g\term{q, \term{DV^T}^{-1} v_j}^z &= \sum\limits_{j=1}^d \sum\limits_{q \in P} w(q) g\term{q, \term{DV^T}^{-1} v_j}^z \label{eq:sensSumBoundSmallZ_1}\\
&= \sum\limits_{j=1}^d \sum\limits_{q \in P}  \term{w(q)^{\frac{1}{z}} g\term{q, \term{DV^T}^{-1} v_j}}^z \label{eq:sensSumBoundSmallZ_2}\\
&\leq n^{1-z} \sum\limits_{j=1}^d \term{\sum\limits_{q \in P} w(q)^{\frac{1}{z}} g\term{q, \term{DV^T}^{-1} v_j}}^z \label{eq:sensSumBoundSmallZ_3}\\
&\leq n^{1-z}\term{c_4 \alpha}^z \sum\limits_{j=1}^d \norm{DV^T \left(DV^T\right)^{-1} v_j}_2^z \label{eq:sensSumBoundSmallZ_4}\\
&= n^{1-z}\term{c_4\alpha}^z \sum\limits_{j=1}^d \norm{v_j}_2\label{eq:sensSumBoundSmallZ_5}\\
&= n^{1-z}\term{c_4\alpha}^z d \label{eq:sensSumBoundSmallZ_6},
\end{align}
where~\eqref{eq:sensSumBoundSmallZ_1} holds by the independency between the summation over $q \in P$ and summation over $j \in [d]$,~\eqref{eq:sensSumBoundSmallZ_2} holds since the weights are non-negative by definition,~\eqref{eq:sensSumBoundSmallZ_3} holds by plugging $z:= z$, $n := n$, $a_i := w(q)^{\frac{1}{z}} g\term{q, \term{DV^T}^{-1} v_j}$ for every $i \in [n]$ into Claim~\ref{lem:pRootSum} where $q$ denotes the $i$th point in $P$,~\eqref{eq:sensSumBoundSmallZ_4} holds by Lemma~\ref{lem:LownerToF},~\eqref{eq:sensSumBoundSmallZ_5} follows since $DV^T \left( DV^T\right)^{-1} = I_{d}$, and finally~\eqref{eq:sensSumBoundSmallZ_6} follows from the assumption of Lemma~\ref{thm:sensitivityBound}.

Similarly for the case of $z \geq 1$,
\begin{equation}
\label{eq:sensSumBound}
\sum\limits_{q \in P} \sum\limits_{j=1}^d w(q) g\term{q, \term{DV^T}^{-1} v_j}^z \leq \term{c_4\alpha}^z d.
\end{equation}

% \begin{align}
% \sum\limits_{q \in P} \sum\limits_{j=1}^d w(q) g\term{q, \term{DV^T}^{-1} v_j}^z &= \sum\limits_{j=1}^d \sum\limits_{q \in P} w(q) g\left(q, \left( DV^T\right)^{-1} v_j \right)^z \label{eq:sensSumBound_1}\\
% &\leq c_4 \alpha \sum\limits_{j=1}^d \norm{DV^T \left(DV^T\right)^{-1} v_j}_2 \label{eq:sensSumBound_2}\\
% &= c_4\alpha \sum\limits_{j=1}^d \norm{v_j}_2\label{eq:sensSumBound_3}\\
% &= c_4\alpha d \label{eq:sensSumBound_4},
% \end{align}
% where~\eqref{eq:sensSumBound_1} holds from the independency between the summation over $q \in P$ and summation over $j \in [d]$,~\eqref{eq:sensSumBound_2} is by Definition~\eqref{def:svdWRTF},~\eqref{eq:sensSumBound_3} holds since $DV^T \left( DV^T\right)^{-1} = I_{d}$, and~\eqref{eq:sensSumBound_4} follows from the assumption of Lemma~\ref{thm:sensitivityBound}.

Hence, Lemma~\ref{thm:sensitivityBound} holds as
\begin{equation*}
\begin{split}
\sum_{p\in P}s(p) &\leq \sum_{p\in P} \frac{c_2c_5 w(p)}{c_1 \sum_{q \in P} w(q)} +\sum_{p\in P} \frac{c c_2}{c_1c_3^{2z}}\sum\limits_{j=1}^d w(p)g\term{ p, \term{ DV^T}^{-1} v_j}^z \\
&\leq \frac{c_2c_5}{c_1}+ \frac{c c_2 c_4^z}{c_1c_3^{2z}}\maxArgs{n^{1-z}, 1}\alpha^z d,
\end{split}
\end{equation*}
where the first inequality holds by~\eqref{eq:sensBound}, and the second inequality holds by combining~\eqref{eq:sensSumBoundSmallZ_1}--\eqref{eq:sensSumBound}.
\end{proof}

% \subsection{Proof of Corollary~\ref{coro:sensitivityBoundGeneral}}
% \setcounter{theorem}{\getrefnumber{coro:sensitivityBoundGeneral}-1}

% \sensitivityBoundGeneral*

% \begin{proof}
% Put $p \in P$, let $x \in \argsup{\substack{x \in \REAL^d\\ f(p,x) > 0}} \frac{w(p)f(p,x)}{\sum\limits_{q \in P} w(q) f(q,x)}$ and let $y=\frac{Gx}{\norm{Gx}_2}$. We observe that,

% \begin{align}
% \label{eq:coroGen}
% \frac{f(p,x)}{\sum\limits_{q \in P} w(q) f(q,x)} &\leq \frac{c_4 k(p,x) + c_4 r(p,x)}{\sum\limits_{q \in P} c_3 w(q) \left( k(q,x) + r(q,x)\right)}  \leq \frac{c_4k(p,x)}{\sum\limits_{q \in P} c_3 w(q) k(q,x)} +  \frac{c_4r(p,x)}{\sum\limits_{q \in P} c_3 w(q) r(q,x)}
% \end{align}
% where the first inequality of~\eqref{eq:boundingSens1} is by the assumption of Corollary~\ref{coro:sensitivityBoundGeneral}, and the second inequality holds by Claim~\ref{clm:divOfPosSums}.

% Corollary~\ref{coro:sensitivityBoundGeneral} now follows by combining~\eqref{eq:coroGen} with Theorem~\ref{thm:sensitivityBound} and the assumption in the body of Corollary~\ref{coro:sensitivityBoundGeneral} with respect to $k$ and $r$.
% \end{proof}

\setcounter{myCounter}{\getrefnumber{thm:sensitivityBound:appendix}}

\subsection{Proof of Theorem~\ref{thm:runTime}}
\runTime*

\begin{proof}
In algorithm~\ref{alg:mainAlg}, we first compute the sensitivity bounds $s(p)$ for every $p \in P$ with respect to the query space $(P,w,\Q,f)$. This is done based on Lemma~\ref{thm:sensitivityBound}; See Line~\ref{algLine:5}.
We then sample a sufficiently large number of points based on those sensitivity bound as Theorem~\ref{thm:coreset} states; See Line~\ref{algLine:donesamle}. Hence, By plugging $P,w,f,\REAL^d,\eps,\delta$ and $s(p)$ for every $p \in P$ into Theorem~\ref{thm:coreset}, we obtain that with probability at least $1 - \delta$, $(S,v)$ is an $\eps$-coreset (see Definition~\ref{def:epsCore}) of size $\abs{S}\in\bigO{ \frac{t}{\eps^2}\left(d'\log t+\log\left(\frac{1}{\delta}\right)\right)}$.

The overall time is dominated by computing the $f$-SVD of $(P,w)$, i.e., $(U,D,V)$ at Line~\ref{algLine:1} of Algorithm~\ref{alg:mainAlg}. This is done by computing the L\"{o}wner ellipsoid, as explained in the proof of Lemma~\ref{lem:LownerToF}.

The computation of the L\"{o}wner ellipsoid,
requires a separation oracle, where we use the gradient of $g$ as a candidate, similarly to~\cite{clarkson2005subgradient}. We refer to ~\cite{lovasz1986algorithmic} for more details on the computation of L\"{o}wner ellipsoid.
\end{proof}

\subsection{Proof of Corollary~\ref{thm:logisticReg}}

\setcounter{theorem}{20}

\begin{restatable}{claim}{divOfPosSums}
\label{clm:divOfPosSums}
Let $a_1, r_1, a_2, r_2 \in [0,\infty)$ such that $a_2,r_2 > 0$. Then,
\[
\frac{a_1+r_1}{a_2+r_2} \leq \frac{a_1}{a_2} + \frac{r_1}{r_2}.
\]
\end{restatable}
\begin{proof}
Observe that,
\[
\frac{a_1 + r_1}{a_2 + r_2} = \frac{a_1}{a_2 + r_2} + \frac{r_1}{a_2 + r_2} \leq \frac{a_1}{a_2} + \frac{r_1}{r_2},
\]
where the inequality holds since $a_2,r_2 > 0$ and $a_1,r_1 \geq 0$.
\end{proof}

\begin{restatable}{claim}{maxProp}
\label{clm:maxProp}
Let $N \geq 2$. For every $i \in [N]$, let $a_i \geq 0$ and $b_i > 0$. Then,
\[
\frac{\max\br{a_1, a_2, \cdots, a_N}}{\max\br{b_1, b_2, \cdots, b_N}} \leq \max\br{\frac{a_1}{b_1}, \frac{a_2}{b_2}, \cdots, \frac{a_N}{b_N}}.
\]
\end{restatable}

\begin{proof}
Let $\hat{i} \in \arg\max_{i \in [N]} a_i$ and let $\hat{j} \in \arg\max_{i \in [N]} b_i$. Then,
\[
\frac{\max\br{a_1, a_2, \cdots, a_N}}{\max\br{b_1, b_2, \cdots, b_N}} = \frac{a_{\hat{i}}}{b_{\hat{j}}} \leq \frac{a_{\hat{i}}}{b_{\hat{i}}} \leq \max\br{\frac{a_1}{b_1}, \frac{a_2}{b_2}, \cdots, \frac{a_N}{b_N}},
\]
where the first inequality holds by the definition of $b_{\hat{j}}$.

\end{proof}

\begin{claim}
\label{clm:Logistic_upper_1}
For every $z,b \in \REAL$,
\begin{equation*}
\ln{\left( 1 + e^{z + b}\right)} \leq 2\ln{\left( 1 + e^{z^2}e^b\right)}.
\end{equation*}

\end{claim}

\begin{proof}
Put $b \in \REAL$, and note that for every $z \in \REAL$, we have
$$\ln{(2)} + z^2 - z \geq 0,$$
by rearranging the above, we get that
$$\ln{(2)} + z^2  \geq z.$$

Applying the exponentiation operation on both sides with respect to the natural number $e$ as the base, yields
\[
2e^{z^2} \geq e^z,
\]
and since $e^{z^2 + b} > 0$,
\[
e^z \leq e^{z^2} \left( 2 + e^{z^2 + b}\right).
\]

By multiplying each side by $e^b$ and adding $1$, we obtain that
\[
1 + e^{z+b} \leq 1 + 2e^{z^2 + b} + e^{z^2 + 2b}.
\]
Applying the logarithm function on both sides of the inequality above proves Claim~\ref{clm:Logistic_upper_1} as
$$\ln{\left( 1 + e^{z + b}\right)} \leq 2\ln{\left( 1 + e^{z^2}e^b\right)}.$$
\end{proof}

\begin{lemma}[Bernoulli's inequality,~\cite{kuang2004applied}]
\label{lem:bernoulli}
Let $x \geq -1$ be a real number and let $r \in [0,1]$ be a positive real number. Then,
\[
\left( 1 + x \right)^r \leq 1 + rx
\]
\end{lemma}

\begin{lemma}
\label{lem:boundingLog2}
Let $N > 1$, $c \in [1,N]$ and let $p \in \REAL^d$ such that $\norm{p}_2 \leq 1$. Then for every $(x,b) \in \REAL^d \times \REAL$,
\begin{enumerate}[label=(\roman*)]
    \item $\frac{1}{c} \ln{\left( 1 + e^{p^Tx + b}\right)} + \frac{1}{2N}\norm{x}_2^2 \leq \frac{4}{c} (p^Tx)^2 + 4 \maxArgs{\frac{1}{c} \ln{\left( 1 + e^b\right)}, \frac{1}{2N} \norm{x}_2^2}$,
    \item and $\frac{1}{c} \ln{\left( 1 + e^{p^Tx + b}\right)} + \frac{1}{2N}\norm{x}_2^2 \geq \frac{c}{8N}\left(\frac{1}{c} (p^Tx)^2 + \maxArgs{\frac{1}{c} \ln{\left( 1 + e^b\right)}, \frac{1}{2N} \norm{x}_2^2} \right)$.
\end{enumerate}
\end{lemma}

\begin{proof}
Put $(x,b) \in \REAL^{d} \times \REAL$. We now proceed to prove Lemma~\ref{lem:boundingLog2}.

\paragraph{Proof of Claim (i).}
By plugging $z := p^Tx$ and $b := b$ into Claim~\ref{clm:Logistic_upper_1}, we obtain that
\begin{equation}
\label{eq:Logistic_upper_2}
\ln{\left( 1 + e^{p^Tx + b}\right)} \leq 2\ln{\left( 1 + e^{\left( p^Tx\right)^2 + b}\right)} \leq 2\ln{\term{e^{\left( p^Tx\right)^2}\left( 1 + e^{b}\right)}} = 2\left( p^Tx\right)^2 + 2\ln{\left( 1 + e^b\right)},
\end{equation}
where the second inequality holds since $e^{\left( p^Tx\right)^2} \geq 1$, and the equality follows from properties of the logarithm function.

Thus, Claim (i) holds since
\begin{align}
\frac{1}{c} \ln{\left( 1 + e^{p^Tx + b}\right)} + \frac{1}{2N}\norm{x}_2^2 &\leq \frac{2}{c} \left(p^Tx\right)^2 + \frac{2}{c} \ln{\left(1 + e^b\right)} + \frac{1}{2N}\norm{x}_2^2 \nonumber\\
&\leq \frac{4}{c} \left(p^Tx\right)^2 + 4 \maxArgs{\frac{1}{c} \ln{\left(1 + e^b\right)}, \frac{1}{2N}\norm{x}_2^2} \nonumber,
\end{align}
where the first inequality is by~\eqref{eq:Logistic_upper_2}, the second inequality holds by properties of the \emph{max} operator.

\paragraph{Proof of Claim (ii). } We start by noting that since $\norm{p} \leq 1$, we have that
\begin{equation}
\label{eq:xTopx}
\norm{x}_2 \geq \abs{p^Tx},
\end{equation}
which consequently leads to
\begin{align}\label{eq1prof}
\frac{1}{c}\ln{\left(1+e^{p^T x + b}\right)} \geq \frac{1}{c} \ln{\left(1+e^{-\norm{x}_2 + b}\right)} \geq \frac{1}{2N} \ln{\left(1+e^{-\norm{x}_2 + b}\right)},
\end{align}
where the second inequality holds since $c \leq N$.

We show that
\begin{equation}
\label{eq:Logistic_lower_1}
\ln{\left(1+e^{-\norm{x}_2 + b}\right)} + \norm{x}^2 \geq \frac{1}{2} \ln{\left(1+e^b\right)},
\end{equation}
holds for every $x \in \Q$ and $b \in \REAL$.
In order to to that, we first define the function $q:\REAL\to (0, \infty)$ such that for every $r\in \REAL$, $q(r)=\ln{\left(1+e^{-\abs{r} + b}\right)} + r^2$.

Let $W$ denotes the \emph{Lambert W function}
( see~\cite{corless1996lambertw}).
Minimizing $q(r)$ over $r \in \REAL$, requires computing the derivative of $q(r)$ with respect to $r$, and setting it to zero. We observe that when setting the derivative to zero we obtain that $r^* \in [-W(1), W(1)]$, i.e., the left term of~\eqref{eq:Logistic_lower_1} attains its minimal value at some $x^* \in \Q$ such that $\norm{x^*}_2 \in [0, W(1)]$.

Observe that for every $x\in \Q$
\begin{align*}
\ln{\left(1+e^{-\norm{x}_2 + b}\right)} + \norm{x}_2^2 \geq \ln{\left(1+e^{-\norm{x^*}_2 + b}\right)} + \norm{x^*}_2^2 \geq   \ln{\left(1+e^{-\norm{x^*}_2 + b}\right)} \geq \ln{\left(1+e^{-W(1) + b}\right)},
\end{align*}
where the first inequality holds by the definition of $x^*$, the second inequality holds since $\norm{x^*}_2^2 \geq 0$, and the last inequality follows from the observation that $\norm{x^*} \in [0, W(1)]$.

Since $e^{-W(1)} \in (0,1)$, we have that
\[
\ln{\left(1+e^{-W(1) + b}\right)} \geq \ln{\left(\left(1+e^{b}\right)^{e^{-W(1)}}\right)} = e^{-W(1)} \ln{\left(1+e^{b}\right)} \geq \frac{1}{2} \ln{\left(1+e^{b}\right)},
\]
where the first inequality holds by plugging $r := e^{-W(1)}$ and $x := e^b$ into Lemma~\ref{lem:bernoulli}, the equality holds by properties of the logarithm function, and the last inequality holds since $e^{-W(1)} \geq \frac{1}{2}$.

We also observe that
\begin{equation}
\label{eq:xTopx2}
\frac{1}{2N} \norm{x}_2^2  \geq \frac{1}{2N} \abs{p^Tx}^2 = \frac{c}{2N} \left( \frac{1}{c} \abs{p^Tx}^2 \right),
\end{equation}
where the first inequality holds by~\eqref{eq:xTopx}, and the equality holds since $\frac{c}{2N} \cdot \frac{1}{c} = \frac{1}{2N}$.

Thus by combining~\eqref{eq1prof},~\eqref{eq:Logistic_lower_1}, and~\eqref{eq:xTopx2}, Claim~(ii) holds as
\[
\frac{1}{c}\ln{\left(1+e^{p^T x + b}\right)} + \frac{1}{2N}\norm{x}_2^2 \geq \frac{c}{4N}\left( \frac{1}{2c} \abs{p^Tx}^2 + \maxArgs{\frac{1}{2c} \ln{\left( 1 + e^b\right)}, \frac{1}{2N}\norm{x}_2^2} \right).
\]
\end{proof}

\begin{lemma}
\label{lem:sensBoundLogistic}
Let $(P,w,\REAL^{d+1},\flog)$ be a query space, $y : P \to \br{1,-1}$ be a labelling function, $\lambda \geq 1$ be a regularization parameter, such that for every $p \in P$, $b \in \REAL$ and $x\in \REAL^d$, $$\flog(p,(x\mid b)) = \frac{1}{2\sum\limits_{q \in P} w(q)}\norm{x}_2^2 + \frac{1}{\lambda}\ln{\left( 1+e^{p^Tx + y(p)b}\right)}.$$
For every $p \in P$, let $P_{y(p)} = \br{q \mid q \in P, y(q) = y(p)}$ denote the set of points with the same label as the label assigned to $p$.
%Let $\flog[g]:P \times \REAL^{d+1}\to (0,\infty)$, such that for every $p \in P$, $b \in \REAL$ and $x\in \REAL^d$, $\flog[g](p,(x\mid b)) = \abs{p^Tx}^2$.
Let $(U,D,V)$ be the $f$-SVD of $(P,w)$ with respect to $\flog[f]$. Then, claims (i) -- (ii) hold as follows:
\begin{enumerate}[label=(\roman*)]
\item for every $p \in P$, the sensitivity of $p$ with respect to the query space $(P,w, \REAL^{d+1},\flog)$ is bounded by
\[
s(p) = \frac{32}{\lambda} \term{\frac{2w(p)}{\sum\limits_{q \in P_{y(p)}} w(q)} +  w(p)\norm{U(p)}_2^2} \sum\limits_{q \in P_{y(p)}} w(q),
\]
\item and the total sensitivity is bounded by
\[
\sum\limits_{p \in P} s(p) \leq \frac{32}{\lambda} \left(2 + d \right) \sum\limits_{p \in P}w(p).
\]
\end{enumerate}
\end{lemma}

\begin{proof}
Put $p \in P$ and let $P_{y(p)}$ denote the subset of points from $P$ with same label as $p$, $P_{y(p)} = \br{q \mid q \in P, y(q) = y(p)}$. Observe that for every $q \in P_{p}$
\[
\sup\limits_{(x , b) \in \REAL^{d} \times \REAL} \frac{w(p) \flog(p,(x \mid b))}{\sum\limits_{q \in P} w(q) \flog(q,(x\mid b))} \leq \sup\limits_{(x , b) \in \REAL^{d} \times \REAL} \frac{w(p) \flog(p,(x \mid b))}{\sum\limits_{q \in P_{y(p)}} w(q) \flog(q,(x\mid b))},
\]
where the inequality holds since $P_{y(p)} \subseteq P$, and $\flog(q,(x \mid b)) \geq 0$ for every $q \in P$, and $(x \mid b) \in \REAL^d \times \REAL$.

Note the following:
\begin{enumerate}[label=(\alph*)]
    \item For every $q \in P$, $ x \in \Q$ and $\gamma \geq  0$ we have $\abs{q^T \gamma x} = \gamma \abs{q^T x}$.
    \item Since $\abs{q^T x}$ is convex function, it also holds that $\sum\limits_{q \in P} \abs{q^T x}^2$ is convex due to the fact that sum of convex functions is also convex,
    \item The level set $\br{x \middle| x \in \Q, \sum\limits_{q \in P} w(q) \abs{q^Tx}^2 \leq 1}$ is convex and is centrally symmetric.
    \item For every $x\in \REAL^{d}$, and $b\in \REAL$ we have that
    \[
    \frac{w(q) \maxArgs{\frac{1}{\lambda} \ln{\left( 1 + e^b\right)}, \frac{1}{2\sum\limits_{q \in P_{y(p)}} w(q)} \norm{x}_2^2}}{\sum\limits_{\Tilde{q} \in P_{y(p)}} w(\Tilde{q}) \maxArgs{\frac{1}{\lambda} \ln{\left( 1 + e^b\right)}, \frac{1}{2\sum\limits_{q \in P_{y(p)}} w(q)} \norm{x}_2^2}} \leq 2 \frac{w(q)}{\sum\limits_{\Tilde{q} \in P_{y(p)}} w\term{\Tilde{q}}},
    \]
    where the inequality holds by plugging $a_1 := w(q) \frac{1}{\lambda} \ln{\left( 1 + e^b\right)}$, $a_2 := \frac{1}{2\sum\limits_{q \in P_{y(p)}} w(q)} \norm{x}_2^2$, $b_1 := \frac{\sum\limits_{\Tilde{q} \in P_{y(p)}} w\term{\Tilde{q}}}{\lambda} \ln{\left( 1 + e^b\right)}$, $b_2 := \frac{1}{2} \norm{x}_2^2$ into Claim~\ref{clm:maxProp}.
\end{enumerate}

Thus, combining (a), (b), (c), (d) and Lemma~\ref{lem:boundingLog2}, allows us to plug
\begin{itemize}
    \item $f(p,(x \mid b)) := \flog[f](p, (x \mid b))$, $g(p,(x \mid b)) := \frac{1}{\sqrt{\lambda}}\abs{p^Tx}$ and $h(p,(x \mid b)) := \maxArgs{\sqrt{\frac{1}{\lambda} \ln{\left( 1 + e^b\right)}}, \sqrt{\frac{1}{2\sum\limits_{q \in P_{y(p)}} w(q)} \norm{x}_2^2}}$, for every $p \in P$, $x \in \REAL^{d}$, $b\in \REAL$,
    \item $\alpha = 1$,
    \item $c_1 := \frac{\lambda}{8N}$ and $c_2 = 4$,
    \item $c_i := 1$ for every $i \in [3,4]$
    \item $c_5 := 2$,
    \item $z := 2$,
    \item $v_j := e_j$ for every $j \in [d]$ where $e_j$ denotes the vector with a $1$ in the $j$th coordinate and $0$'s elsewhere,
    \item and $c := 1$,
\end{itemize}
into Lemma~\ref{thm:sensitivityBound}, which yields that $\flog \in \set{F}$ and the sensitivity of each point $q \in P_{y(p)}$ is bounded by
$$s(q) =  \frac{32}{\lambda} \term{2 \frac{w(q)}{\sum\limits_{\Tilde{q} \in P_{y(p)}} w\term{\Tilde{q}}} + w(q)\sum\limits_{j = 1}^d\abs{U(q)^Te_j}^2}\sum\limits_{\Tilde{q} \in P_{y(p)}} w(\Tilde{q}).$$
%as well as the total sensitivity with respect to $(P_+, w)$.

Claim (i) now holds since for every $q \in P$,
\[
\sum\limits_{j = 1}^d\abs{U(q)^Te_j}^2 = \norm{U(q)}_2^2,
\]
where the equality follows from definition of $e_j$ for every $j \in [d]$.

%The exact same analysis can be done to bound the sensitivity of every  $q\in P \setminus P_{y(p)}$, i.e., similar arguments are applied for the set $P \setminus P_{y(p)}$ and Theorem~\ref{thm:logisticReg} follows from taking into account both sets, i.e., $P_+$ and $P_-$.

As for the total sensitivity, we have by Lemma~\ref{thm:sensitivityBound},
\[
\sum\limits_{q \in P_{y(p)}} s(q) \leq \frac{32}{\lambda} \left(2 \frac{w(q)}{\sum\limits_{\Tilde{q} \in P_{y(p)}} w\term{\Tilde{q}}} + d \right) \sum\limits_{q \in P_{y(p)}} w(p),
\]
and
\[
\sum\limits_{q \in P \setminus P_{y(p)}} s(q) \leq \frac{32}{\lambda} \left(2 \frac{w(q)}{\sum\limits_{\Tilde{q} \in P \setminus P_{y(p)}} w\term{\Tilde{q}}} + d \right) \sum\limits_{q \in P \setminus P_{y(p)}} w(p).
\]

Hence, Claim (ii) holds as
\[
\sum\limits_{q \in P} s(q) \leq \frac{32}{\lambda} \left(2 + d \right) \sum\limits_{q \in P} w(q).
\]
\end{proof}

\setcounter{myCounter}{\getrefnumber{lem:sensBoundLogistic}}

\setcounter{theorem}{\getrefnumber{thm:logisticReg}-1}
\logisticReg*
\begin{proof}
First, observe that by Lemma~\ref{lem:sensBoundLogistic} the total sensitivity is bounded by $t := \frac{32}{\lambda} \left(2 + d \right) \sum\limits_{q \in P} w(p)$. Hence, plugging $s(p)$ for every $p \in P$ from Lemma~\ref{lem:sensBoundLogistic}, $t:=t$,  $\eps:=\eps$ and $\delta:=\delta$ into Theorem~\ref{thm:runTime}, yields that $(S,v)$ is an $\eps$-coreset of size $\bigO{\frac{d\W}{\C \eps^2} \term{d \log{\term{\frac{dC}{\C}}} + \log{\term{\frac{1}{\delta}}}}}$. %Such coreset can be constructed by Algorithm~\ref{alg:mainAlg} using the variables defined in the proof of Lemma~\ref{lem:sensBoundLogistic}.
\end{proof}

\setcounter{theorem}{\themyCounter}

\subsection{Proof of Corollary~\ref{thm:lpRegressionNC}}

\setcounter{theorem}{26}

\begin{lemma}
\label{lem:sensBoundLZNCReg}
Let $(P,w,\REAL^{d+1},\fnclz)$ be a query space, such that for every $p \in P$ and $x\in \REAL^d$, $$\fnclz(p,x) = \abs{p^Tx}^z.$$
Let $(U,D,V)$ be the $f$-SVD of $(P,w)$ with respect to $\fnclz[f]$. Then, claims (i) -- (ii) hold as follows:
\begin{enumerate}[label=(\roman*)]
\item for every $p \in P$, the sensitivity of $p$ with respect to the query space $(P,w, \Q,\fnclz)$ is bounded by
\[
s(p) = w(p)\norm{U(p)}_z^z,
\]
\item and the total sensitivity is bounded by
\[
\sum\limits_{p \in P} s(p) \leq n^{1-z} d^{\frac{z}{2} + 1}.
\]
\end{enumerate}

\end{lemma}

\begin{proof}
Let $g : P \to [0, \infty)$ such that for every $p \in P$ and $x \in \REAL^d$, $g(p,x) = \abs{p^Tx}$, and for every $ i \in [d]$ let $e_i$ denote the vector with $1$ in the $i$th coordinate and $0$'s elsewhere. Observe that:
\begin{enumerate}[label=(\alph*)]
    \item For every $q \in P$, $ x \in \Q$ and $b \geq  0$ we have $g(p, b\cdot x) = b \cdot g(p,x)$.
    \item Since $g(q,x)$ is a convex function for every $q \in P$, it also holds that $\sum\limits_{q \in P} w(q)^{\frac{1}{z}} g(q,x)$ is convex due to the fact that sum of convex functions is also convex.
    \item The level set $\br{x \middle| x \in \Q, \sum\limits_{q \in P} w(q)^{\frac{1}{z}} g(q, x) \leq 1}$ is convex and is centrally symmetric.
    \item In addition, for every unit vector $y \in \REAL^d$
\begin{equation*}
\abs{p^Ty}^z \leq \norm{p}_2^z \leq \norm{p}_z^z = \sum\limits_{i=1}^d \abs{p^Te_i}^z,
\end{equation*}
where the first inequality holds by Cauchy's inequality, the second inequality is by Lemma~\ref{thm:equivalenceNorms}, and the equality is by properties of norm.
\end{enumerate}

Hence combining (a), (b), (c) and (d), allows us to plug
\begin{itemize}
\item $f(q,x) := \fnclz[f](q,x)$, $g(q,x) := \abs{{q}^{T}x}$ and $h(q,x) := 0$ for every $q \in P$ and $x \in \Q$,
\item $c_i := 1$ for every $i \in [4]$,
\item $c_5 := 0$,
\item $\alpha := \sqrt{d}$,
\item $v_j := e_j$ for every $j \in [d]$ where $e_j$ denotes the vector with a $1$ in the $j$th coordinate and $0$'s elsewhere,
\item $c := 1$, and
\item $z := z$
\end{itemize}
into Lemma~\ref{thm:sensitivityBound}, which yields that $\fnclz \in \set{F}$ and the sensitivity of each point $p \in P$ is bounded by by
\begin{equation*}
\label{eq:L1SVD}
s(p) = \sum\limits_{i=1}^d \abs{U(p)^Te_i}_z^z,
\end{equation*}
and the total sensitivity is bounded by
\begin{equation*}
\sum\limits_{q \in P} s(q) \leq n^{1-z} d^{\frac{z}{2}+1}.
\end{equation*}
\end{proof}

\setcounter{myCounter}{\getrefnumber{lem:sensBoundLZNCReg}}

\setcounter{theorem}{\getrefnumber{thm:lpRegressionNC}-1}

\lpRegressionNC*

\begin{proof}
First, observe that by Lemma~\ref{lem:sensBoundLZNCReg}, the total sensitivity is bounded by $t := n^{1-z} d^{\frac{z}{2} + 1}$. Plugging $s(p)$ for every $p \in P$ from Lemma~\ref{lem:sensBoundLZNCReg}, $t:=t$ , $\eps:=\eps$ and $\delta:=\delta$ into Theorem~\ref{thm:runTime}, yields an $\eps$-coreset of size $\bigO{\frac{n^{1-z}d^{\frac{z}{2} + 1}}{\eps^2} \term{d \log{\term{n^{1-z}d^{\frac{z}{2} + 1}}} + \log{\term{\frac{1}{\delta}}}}}$.
\end{proof}

\setcounter{theorem}{27}

\subsection{Proof of Corollary~\ref{lem:SVM}}

\begin{restatable}{lemma}{boundingSVMs}
\label{lem:boundingSVMs}
Let $\gamma \in (0,1]$, $N \geq 1$,  and let $c \in [1,N]$. Let $p \in \REAL^d$ such that $\norm{p}_2 \leq 1$ and let $X = \left\lbrace (x,b) \in \REAL^d \times \REAL \mid \norm{x}_2 \geq \gamma, \abs{b} \leq 9\norm{x}_2 \right\rbrace$. Then, for every $(x,b) \in X$, claims (i) -- (ii) hold as follows:

\begin{enumerate}[label=(\roman*)]
\item $\frac{\norm{x}_2^2}{N} + \frac{1}{c} \maxArgs{0, 1 + p^Tx + b} \leq \frac{2}{c}\abs{p^Tx}^2 + 2 \maxArgs{\frac{1}{c},\frac{b}{c}, \frac{\norm{x}_2^2}{N}}$,
\item $\frac{\norm{x}_2^2}{N} + \frac{1}{c} \maxArgs{0, 1 + p^Tx + b} \geq \frac{c\gamma^2}{\term{1 + 10 \gamma}N} \left( \frac{1}{c}\abs{p^Tx}^2 + \maxArgs{\frac{1}{c}, \frac{b}{c}, \frac{\norm{x}_2^2}{N}} \right)$.
\end{enumerate}
\end{restatable}

\begin{proof}
Put $(x,b) \in X$.
%, and note that in order for us to prove both Claims (i) and (ii), we use case analysis.

\paragraph{Proof of Claim (i).} The proof is by the following case analysis:
\begin{enumerate}
\item If $p^Tx + b \geq 0$, we have
\[
\maxArgs{0,1+p^Tx+b} = 1 + p^Tx + b \leq 2 + 2 \abs{p^Tx}^2 + b \leq 2\abs{p^Tx}^2 + 2 \maxArgs{1, b},
\]
where the equality holds by the assumption of the case, and the first inequality holds since for every $z\in \REAL$, we have $1 + z \leq 2z^2 + 2.$

\item Otherwise,
\[
\maxArgs{0,1+p^Tx+b} \leq 1 \leq  2\maxArgs{1, b} \leq 2 \abs{p^Tx}^2 + 2\maxArgs{1, b}
\]
where the first inequality holds by the assumption of the case.
\end{enumerate}
By taking both the above cases in mind, Claim (i) holds.

\paragraph{Proof of Claim (ii).} Similar to the proof of Claim~(i), we use the same case analysis:

\begin{enumerate}
\item If $p^Tx + b \geq 0$, we observe that
\begin{equation}
\begin{split}
\frac{1}{c}\maxArgs{0,1+p^Tx+b} + \frac{\norm{x}_2^2}{N} &= \frac{1}{c} \left( 1 + p^Tx + b \right) + \frac{\norm{x}_2^2}{N} \\
&\geq \frac{c}{2N} \left( \frac{1}{c} \abs{p^Tx}^2 + \maxArgs{1, b, \frac{1}{2N}\norm{x}_2^2}\right),
\end{split}
\end{equation}
where the equality holds by the assumption of this case, and the inequality holds since $\abs{p^Tx}^2 \leq \norm{x}_2^2$ due to the assumption that $\norm{p}_2 \leq 1$.

\item Otherwise,
\begin{equation}
\label{eq:svm_lower_1}
\frac{\abs{p^Tx}^2}{c} + \max\br{\frac{1}{c}, \frac{b}{c}, \frac{\norm{x}_2^2}{N}} \leq \frac{1}{c} + \frac{10\norm{x}_2^2}{\gamma c},
\end{equation}
where the inequality follows since by the definition of the set $X$, we have that $b \leq 9\norm{x}_2 \leq \frac{9\norm{x}_2^2}{\gamma}$.

We also note that,
\begin{equation}
\label{eq:svm_lower_2}
\frac{1}{c}\maxArgs{0,1+p^Tx+b} + \frac{\norm{x}_2^2}{N} \geq \frac{\norm{x}_2^2}{N},
\end{equation}
holds since the \emph{$\max$} term is non-negative.

Let $l = \frac{c\gamma^2}{N \left(1 + 10\gamma \right)}$. Observe that
\begin{equation}
\label{eq:svm_lower_3}
\frac{l}{c} \left(1 + \frac{10\norm{x}^2}{\gamma} \right) \leq \frac{\norm{x}_2^2}{N},
\end{equation}
since $\norm{x}_2 \geq \gamma$.

Hence, we obtain that
\[
\frac{1}{c}\maxArgs{0,1+p^Tx+b} + \frac{\norm{x}_2^2}{N} \geq \frac{c\gamma^2}{N \left(1 + 10\gamma \right)} \left( \frac{\abs{p^Tx}^2}{c} + \maxArgs{\frac{1}{c}, \frac{b}{c}, \frac{1}{2N} \norm{x}_2^2}\right),
\]
where the inequality holds by combining~\eqref{eq:svm_lower_1},~\eqref{eq:svm_lower_2} and~\eqref{eq:svm_lower_3}.

Combining both cases proves Claim (ii).

% By combining the inequalities above, it suffices to find $l \geq 0$ such that
% \[
% \frac{l}{c} \left(1 + \frac{10\norm{x}^2}{\gamma} \right) \leq \frac{\norm{x}_2^2}{N},
% \]
% in order for Claim (ii) to hold. By a simple simplification of the inequality above, yields that
% \[
% l \leq \frac{c \norm{x}_2^2}{N \left( 1 + \frac{10\norm{x}_2^2}{\gamma}\right)}.
% \]

% Since upper bound on $l$, increases as $\norm{x}_2$ increases, it suffices to take its lower bound by simply plugging in $\norm{x}_2=\gamma$, resulting in
% \[
% l = \frac{c\gamma^2}{N \left(1 + 10\gamma \right)}.
% \]
\end{enumerate}

\end{proof}

\begin{lemma}
\label{lem:sensBoundSVM}
Let $(P,w, \REAL^{d+1}, \fsvm)$ be a query space, $y : P \to \br{1,-1}$ be a labelling function, $\C \geq 1$ be a regularization parameter such that for every $p \in P$, $x \in \REAL^d$, and $b \in \REAL$,
$$\fsvm(p,(x\mid b)) = \frac{1}{2\sum_{q \in P} w(q)}\norm{x}_2^2 +  \lambda \max\br{0, 1-\left( p^Tx + y(p)b \right)}.$$
For every $p \in P$, let $P_{y(p)} = \br{q \mid q \in P, y(q) = y(p)}$ denote the set of points with the same label as the label assigned to $p$.

Let $(U,D,V)$ be the $f$-SVD of $(P,w)$ with respect to $\fsvm[f]$. Then, claims (i) -- (ii) hold as follows:

\begin{enumerate}[label=(\roman*)]
\item \label{claim:svm_1}for every $p \in P$, the sensitivity of $p$ with respect to the query space $(P,w, \REAL^{d+1},\fsvm)$ is bounded by
\begin{align*}
s(p)  = &\maxArgs{\frac{9w(p)}{\sum\limits_{q \in P_{y(p)}} w(q)}, \frac{2w(p)}{\sum\limits_{q \in P \setminus P_{y(p)}} w(q)}} + \frac{13w(p)}{4 \sum\limits_{q \in P_{y(p)}} w(q)} \\
&+ \frac{125 \sum\limits_{q \in P} w(q)}{4\lambda} \cdot \term{w(q) \norm{U(p)}_2^2 + \frac{w(p)}{\sum\limits_{q \in P} w(q)}},
\end{align*}

\item \label{claim:svm_2} and the total sensitivity is bounded by
\[
\sum\limits_{p \in P} s(p) \leq 25 +  \frac{\sum\limits_{p \in P_{y(p)}}w(p)}{\sum\limits_{q \in P \setminus P_{y(p)}} w(q)} + \frac{\sum\limits_{q \in P \setminus P_{y(p)}}w(p)}{\sum\limits_{p \in P} w(q)} + \frac{125 \sum\limits_{q \in P} w(q)}{4C} \cdot \term{d + 2}.
\]
\end{enumerate}
\end{lemma}

\begin{proof}
Put $p \in P$, let $P_{y(p)}$ denote the subset of points from $P$ with same label as $p$, i.e., $P_{y(p)} = \br{q \mid q \in P, y(q) = y(p)}$, let $\gamma=0.4$, and let $X = \br{(x,b) \mid x \in \REAL, b \in \REAL, \norm{x}_2  \leq \gamma}.$ We have that
%In what follows, we will only handle points from $P_+$, as the case where a point is from $P_-$ is handled similarly. Put $p \in P_+$ and consider the following subspace:
%\begin{equation*}
%\begin{split}
%\end{split}
%\end{equation*}
%By Theorem~\ref{thm:sensitivityBound}, it follows that
\begin{align}
&\sup\limits_{(x,b) \in \REAL^d \times \REAL} \frac{w(p) \fsvm(p,(x\mid b))}{\sum\limits_{q \in P} w(q) \fsvm(q,(x,b))}\nonumber \\
&\leq \sup\limits_{(x,b) \in X} \frac{w(p) \fsvm(p,(x \mid b))}{\sum\limits_{q \in P} w(q) \fsvm(q,(x \mid b))} + \sup\limits_{(x,b) \in \REAL^d \times \REAL \setminus X} \frac{w(p) \fsvm(p,(x \mid b))}{\sum\limits_{q \in P} w(q) \fsvm(q,(x \mid b))}, \label{lefthandside}
\end{align}
\paragraph{Proof of Claim~\ref{claim:svm_1}.}
By the above inequality, in order to bound the sensitivity of a point $p\in P$, we can bound the term in~\ref{lefthandside}. For that, we first proceed to bound the left hand side of~\eqref{lefthandside}.
%We bound the sensitivity of a point by looking at the following two cases: First, we will bound the sensitivity of a point based on the assumption that the queries are from $X$ (left hand side of~\eqref{lefthandside}), we then bound the sensitivity of a pount under the assumption that  the queries are from $X$.

 %of First, we will bound the sensitivity of a point based on the assumption that the queries are from $X$.
\subparagraph{Handling queries from $X$.} We observe that for every $q \in P$

\begin{equation}
\label{eq:pxTox}
-\gamma \leq -\gamma \norm{q}_2 \leq -\norm{x}_2 \norm{q}_2 \leq q^Tx \leq \norm{x}_2 \norm{q}_2 \leq \gamma \norm{q}_2 \leq \gamma,
\end{equation}
where the first and last inequalities hold since $\norm{q} \leq 1$ for every $q \in P$, the second and fifth inequalities hold since $\norm{x}_2 \leq \gamma$, and the third and forth inequalities hold by Cauchy-Schwartz's inequality.

In addition,
\begin{align}
\label{eq:boundingX1_1}
\sup\limits_{(x,b) \in X} \frac{w(p) \fsvm(p,(x \mid b))}{\sum\limits_{q \in P} w(q) \fsvm(q,(x \mid b))}
\leq \sup\limits_{(x,b) \in X}\frac{w(p) \norm{x}_2^2}{\sum\limits_{q \in P} w(q) \norm{x}_2^2} + \sup\limits_{(x,b) \in X} \frac{w(p)\maxArgs{0, 1 + p^Tx + y(p)b}}{\sum\limits_{q \in P} w(q)\maxArgs{0, 1 + q^Tx + y(q)b}},
\end{align}
where the inequality holds by plugging $a_1 := \frac{w(p)}{2\sum\limits_{q \in P} w(q)}\norm{x}_2^2$, $r_1 := w(p)\maxArgs{0, 1 + p^Tx +y(p)b}$, $a_2 := \frac{1}{2}\norm{x}_2^2$ and $r_2 := \sum\limits_{q \in P} w(q) \fsvm(q, (x,b)) - \frac{1}{2}\norm{x}_2$ into Claim~\ref{clm:divOfPosSums}.

Bounding the rightmost term of~\eqref{eq:boundingX1_1} requires carefully checking three cases:
\begin{enumerate}[label=(\alph*)]
\item If $y(p) b > 0$, then we have
\begin{equation}
\begin{split}
\frac{w(p)\maxArgs{0, 1 + p^Tx + y(p)b}}{\sum\limits_{q \in P} w(q)\maxArgs{0, 1 + q^Tx + y(q)b}} &\leq \frac{w(p)\maxArgs{0,1 + p^Tx + y(p)b} }{\sum\limits_{q \in P_{y(p)}} w(q)\maxArgs{1 + q^Tx + y(q)b}}
\\
&= \frac{w(p)\term{1 + p^Tx + y(p)b} }{\sum\limits_{q \in P_{y(p)}} w(q)\term{1 + q^Tx + y(q)b}} \\
&\leq \frac{w(p) \term{1 + p^Tx} }{\sum\limits_{q \in P_{y(p)}} w(q)\term{1 + q^Tx}} + \frac{y(p) w(p) b}{\sum\limits_{q \in P_{y(p)}} y(q) w(q) b},
\end{split}
\end{equation}
where the first inequality holds since $P_{y(p)} \subseteq P$, the equality follows from combining the assumption that $\gamma \in (0,1)$ and~\eqref{eq:pxTox}, and the last inequality holds by combining the fact that $1 + q^Tx \geq 0$ for every $q \in P_{y(q)}$, the assumption of the case, and the result of plugging $a_1 := w(p)\left(1 + p^Tx \right)$, $r_1 := w(p) y(p) b$, $a_2 := \sum\limits_{q \in P_{y(p)}} w(q)\left( 1 + q^Tx \right)$ and $r_2 := \sum\limits_{q \in P_{y(p)}} y(q) w(q)b$ into Claim~\ref{clm:divOfPosSums}.

We also have
\begin{equation*}
\begin{split}
\frac{w(p)\maxArgs{0, 1 + p^Tx}}{\sum\limits_{q \in P} w(q)\maxArgs{0, 1 + q^Tx}} &\leq \frac{w(p)\maxArgs{0, 1 + \gamma \norm{p}_2}}{\sum\limits_{q \in P} w(q)\maxArgs{0, 1 - \gamma \norm{q}_2}} \\
&\leq \frac{w(p) \left( 1 + \gamma \right)}{\sum\limits_{q \in P} w(q) \left(1 - \gamma \right)},
\end{split}
\end{equation*}
where the first inequality holds by~\eqref{eq:pxTox}, and the second inequality follows from the assumption that for every $q \in P$, $\norm{q} \leq 1$.

\item If $y(p) b \in \left[-\gamma, 0 \right]$, then
\begin{align*}
\label{eq:boundingX1_2}
\begin{split}
&\frac{w(p)\maxArgs{0, 1 + p^Tx + y(p)b}}{\sum\limits_{q \in P} w(q)\maxArgs{0, 1 + q^Tx + y(q)b}}\leq \frac{w(p)\maxArgs{0, 1 + \gamma \norm{p}_2 + \gamma}}{\sum\limits_{q \in P_{y(p)}} w(q)\maxArgs{0, 1 - \gamma \norm{q}_2 - \gamma}}  \leq \frac{w(p)\maxArgs{0, 1 + 2\gamma}}{\sum\limits_{q \in P_{y(p)}} w(q) \left( 1-2 \gamma \right)},
\end{split}
\end{align*}
where the first inequality holds since $\abs{y(p) b} \leq \gamma$ and $P_{y(p)} \subseteq P$, and the second inequality holds since $\gamma \in \left(0, \frac{1}{2} \right)$.

\item Otherwise, we have $-\gamma > y(p)b$, which means that for every $q \in P$ such that $y(q) \neq y(p)$, we have $\gamma < y(q)b$.

Thus,
\begin{equation*}
\begin{split}
\frac{w(p)\maxArgs{0, 1 + p^Tx + y(p)b}}{\sum\limits_{q \in P} w(q)\maxArgs{0, 1 + q^Tx - \gamma}} &\leq \frac{w(p)\maxArgs{0, 1 + p^Tx + y(p)b}}{\sum\limits_{q \in P \setminus P_{y(p)}} w(q)\maxArgs{0, 1 + q^Tx - \gamma}} \\
&\leq \frac{w(p)\maxArgs{0, 1 + \gamma \norm{p}_2 + \gamma}}{\sum\limits_{q \in P \setminus P_{y(p)}} w(q)\maxArgs{0, 1 - \norm{q}_2 \gamma + \gamma}}\\
&\leq \frac{w(p)}{\sum\limits_{q \in P \setminus P_{y(p)}} w(q)},
\end{split}
\end{equation*}
where the first inequality holds since $P \setminus P_{y(p)} \subseteq P$, the second inequality follows from~\eqref{eq:pxTox}, and the last inequality holds by the assumption that $\norm{q}_2 \leq $ for every $q \in P$.

% would be smaller than $\abs{b}$ which then, it holds that to maximize the contribution of a point $p$ compared to others, one need to make sure that $y(p)b >0$. This means that the contribution of each point $q \in P$ is bounded by $\frac{2w(q)}{\sum\limits_{\substack{q^\prime \in P\\ y(q^\prime) = y(p)}} w(q^\prime)}$.
% %we have a convex combination of the right most term of~\eqref{eq:boundingX1_2} and $\frac{w(p)}{\sum\limits_{\substack{q \in P  \\y(q) = y(p)}} w(q)}$.
\end{enumerate}

Since $\gamma \in \left( 0, \frac{1}{2}\right)$, we have $1\leq \frac{1 + \gamma}{1 - \gamma} \leq
\frac{1 + 2 \gamma }{1 - 2 \gamma}$, and by that we get
\begin{equation}
\label{eq:less_gamma}
\frac{w(p)\maxArgs{0, 1 + \gamma}}{\sum\limits_{q \in P} w(q) \left( 1- \gamma \right)} \leq \frac{w(p)\maxArgs{0, 1 + 2\gamma}}{\sum\limits_{q \in P} w(q) \left( 1-2 \gamma \right)}.
\end{equation}

Combining the cases above with~\eqref{eq:less_gamma}, yields that
\begin{equation}
\begin{split}
\label{eq:boundingX1_3}
\sup_{(x,b) \in X} \frac{w(p) \fsvm(p,(x \mid b)}{\sum\limits_{q \in P} w(q) \fsvm(q,(x \mid b))} \leq 2 \maxArgs{\frac{w(p)\term{1 + 2\gamma}}{\sum\limits_{q \in P_{y(p)}} w(q) \left( 1-2 \gamma \right)}, \frac{w(p)}{\sum\limits_{q \in P \setminus P_{y(p)}} w(q)}}.
\end{split}
\end{equation}

\subparagraph{Handling queries from $\REAL^d \times \REAL \setminus X$.}
Put $(x,b) \in \REAL^d \times \REAL \setminus X$, and consider the following case analysis:

\begin{enumerate}[label=(\alph*)]
\item If $\abs{b} \leq 9\norm{x}_2$, then we note the following:
\begin{enumerate}[label=(\Alph*)]
    \item For every $q \in P$, $ x \in \Q$ and $\beta \geq  0$ we have $\abs{q^T \beta x} = \beta \cdot \abs{q^T x}$.
    \item Since $\abs{q^T x}$ is a convex function for every $q \in P$, it also holds that $\sum\limits_{q \in P} w(q) \abs{q^T x}^2$ is convex due to the fact that sum of convex functions is also convex.
    \item The level set $\br{x \middle| x \in \Q, \sum\limits_{q \in P} w(q) \abs{q^Tx}^2 \leq 1}$ is convex and is centrally symmetric.
    \item In addition, for every unit vector $y \in \REAL^d$
\begin{equation*}
\abs{q^Ty}^2 \leq \norm{q}_2^2 = \sum\limits_{i=1}^d \abs{q^Te_i}^2,
\end{equation*}
where the inequality holds by Cauchy's inequality and the equality holds by properties of norm.
\end{enumerate}

By combining (A), (B), (C), (D) and the result of substituting $c := \C$, $N:= 2\sum\limits_{q \in P} w(q)$ and $\gamma := 0.4$ into Lemma~\ref{lem:boundingSVMs}, we get that we can plug
\begin{itemize}
    \item $f(p, (x \mid b)) := \fsvm(p, (x \mid b))$, $g(p,(x \mid b)):= \frac{1}{\C}\abs{p^Tx}$, and $h(p,(x \mid b)) := \maxArgs{\frac{1}{\C}, \frac{b}{\C}, \frac{\norm{x}_2^2}{N}}$
    \item $\alpha := 1$,
    \item $c_1 := \frac{\C\gamma^2}{\term{1 + 10\gamma} \sum\limits_{q \in P} w(q)}$ and $c_2 := 2$,
    \item $c_i := 1$ for every $i \in [3,4]$
    \item $c_5 := 2$,
    \item $v_j := e_j$ for every $j \in [d]$ where $e_j$ denotes the vector with a $1$ in the $j$th coordinate and $0$'s elsewhere,
    \item and $c := 1$,
\end{itemize}
into Lemma~\ref{thm:sensitivityBound}, to obtain that $\fsvm \in \set{F}$ with respect to any $x \in \REAL^d \times \REAL \setminus X$ and the sensitivity $p$ is bounded by
\begin{equation}
\label{eq:svm_cost}
\begin{split}
\frac{w(p)\fsvm(p,(x \mid b))}{\sum_{q \in P} w(q)\fsvm(q,(x \mid b))} \leq \frac{\term{1 + 10\gamma} \sum\limits_{q \in P} w(q)}{\gamma^2\C} \term{\frac{2w(p)}{\sum_{q \in P} w(q)} + \sum\limits_{j=1}^d \abs{U(p)^Te_j}^2},
\end{split}
\end{equation}
with respect to any query in $\REAL^d \times \REAL \setminus X$.

\item If $y(p) b \geq 9 \norm{x}_2$ then we have that
\begin{equation}
\label{eq:svm_sens_1}
\begin{split}
\frac{w(p)\fsvm(p,(x \mid b))}{\sum_{q \in P} w(q)\fsvm(q,(x \mid b))} &\leq \frac{w(p) \norm{x}_2^2}{\norm{x}_2^2 \sum\limits_{q \in P} w(q)} + \frac{w(p)\maxArgs{0, 1 + p^Tx + y(p)b}}{\sum\limits_{q \in P}w(q) \maxArgs{0, 1 + q^Tx + y(q)b}} \\
&= \frac{w(p)}{\sum\limits_{q \in P} w(q)} + \frac{w(p)\maxArgs{0, 1 + p^Tx + y(p)b}}{\sum\limits_{q \in P}w(q) \maxArgs{0, 1 + q^Tx + y(q)b}},
\end{split}
\end{equation}
where the inequality holds by plugging $a_1 := \frac{w(p) \norm{x}_2^2}{2 \sum\limits_{q \in P} w(q)}$, $r_1 := w(p)\maxArgs{0, 1 + p^Tx + y(p)b}$, $a_2 := \frac{1}{2} \norm{x}_2^2$ and $r_2 = \sum\limits_{q \in P} w(q)\maxArgs{0, 1 + q^Tx + y(q)b}$ into Claim~\ref{clm:divOfPosSums}, and the equality holds since $\norm{x}_2 \geq \gamma$.

In addition, we observe that
\begin{equation}
\label{eq:svm_sens_2}
\begin{split}
\frac{w(p)\maxArgs{0, 1 + p^Tx + y(p)b}}{\sum\limits_{q \in P}w(q) \maxArgs{0, 1 + q^Tx + y(q)b}} &\leq \frac{w(p)\maxArgs{0, 1 + p^Tx + y(p)b}}{\sum\limits_{q \in P_{y(p)}}w(q) \maxArgs{0, 1 + q^Tx + y(q)b}} \\
&\leq \frac{w(p)\maxArgs{0, 1 + \norm{x}_2 + y(p)b}}{\sum\limits_{q \in P_{y(p)}}w(q) \maxArgs{0, 1 - \norm{x}_2 + y(q)b}} \\
&\leq \frac{w(p)\maxArgs{0, 1 + \frac{10y(p)}{9}b}}{\sum\limits_{q \in P_{y(p)}}w(q) \maxArgs{0, 1 + \frac{8y(q)}{9}b}}
\\
&= \frac{w(p)\term{1 + \frac{10y(p)}{9}b}}{\sum\limits_{q \in P_{y(p)}}w(q) \term{1 + \frac{8y(q)}{9}b}} \\
&\leq \frac{w(p)}{\sum\limits_{q \in P_{y(p)}} w(q)} + \frac{5 w(p)}{4 \sum\limits_{q \in P_{y(p)}} w(q)} \\
&= \frac{9w(p)}{4 \sum\limits_{q \in P_{y(p)}} w(q)}
\end{split}
\end{equation}
where the first inequality holds since $P \subseteq P_{y(p)}$, the second inequality holds since $\norm{q}_2 \leq 1$ for every $q \in P$, both the third inequality and the equality is by the assumption of the case, and the last inequality follows from plugging $a_1 := w(p)$, $r_1 := w(p) \frac{11 y(p)}{10} b$, $a_2 := \sum\limits_{q \in P} w(q)$, and $r_2 := \frac{8}{9} \sum\limits_{q \in P} w(q)$ into Claim~\ref{clm:divOfPosSums}.

Combining~\eqref{eq:svm_sens_1} and~\eqref{eq:svm_sens_2}, yields that
\[
\frac{w(p)\fsvm(p,(x \mid b))}{\sum_{q \in P} w(q)\fsvm(q,(x \mid b))} \leq \frac{13w(p)}{4 \sum\limits_{q \in P_{y(p)}} w(q)}
\]

\item Otherwise, i.e., $y(p) b \leq -9 \norm{x}_2$, we have that for every $q \in P_{y(p)}$
\[
\maxArgs{0, 1 + q^Tx + y(q)b} \leq \maxArgs{0, 1 + \norm{x}_2 + y(q)b} = 0,
\]
where the first inequality holds since $\norm{q}_2 \leq 1$ for every $q \in P$, and the fact that $1 - 8\gamma < 0$.

Thus,
\begin{equation}
\frac{w(p)\fsvm(p,(x \mid b))}{\sum\limits_{q \in P} \fsvm(p,(x\mid b))} \leq \frac{w(p)\fsvm(p,(x \mid b))}{\sum\limits_{q \in P_{y(p)}} w(q)\fsvm(q,(x\mid b))} = \frac{w(p)}{\sum\limits_{q \in P_{y(p)}} w(q)}.
\end{equation}
\end{enumerate}

By combining the three cases above, we obtain that
\begin{equation}
\begin{split}
\label{eq:svm_bound_sense_final}
s(p)  \leq &2\maxArgs{\frac{w(p)\term{1 + 2\gamma}}{\sum\limits_{q \in P_{y(p)}} w(q) \left( 1-2 \gamma \right)}, \frac{w(p)}{\sum\limits_{q \in P \setminus P_{y(p)}} w(q)}} + \frac{13w(p)}{4 \sum\limits_{q \in P_{y(p)}} w(q)} \\
&+ \frac{\term{1 + 10\gamma} \sum\limits_{q \in P} w(q)}{\gamma^2\C} \cdot \term{w(q) \norm{U(p)}_2^2 + \frac{w(p)}{\sum\limits_{q \in P} w(q)}},
\end{split}
\end{equation}

Claim~\ref{claim:svm_1} now holds as
\[
\sum\limits_{j = 1}^d \abs{U(p) e_j}^2 = \norm{U(p)}_2^2,
\]
where the equality follows from the definition of $e_j$ for every $j \in [d]$.

\paragraph{Proof of Claim~\ref{claim:svm_2}. }
As for the total sensitivity, we first note that that
\begin{equation}
\label{eq:sumWPerPy}
\sum\limits_{q \in P_{y(p)}} \frac{w(q)}{\sum\limits_{q^\prime \in P_{y(p)}} w(q^\prime)} = 1.
\end{equation}

In addition, by Lemma~\ref{thm:sensitivityBound},
\begin{equation}
\label{eq:svm_sum_U}
\sum\limits_{q \in P} w(q) \norm{U(q)}_2^2 \leq d.
\end{equation}

Hence,
\begin{align}
\label{eq:svm_sum_sens_1}\sum\limits_{q \in P} s(q) &\leq \sum\limits_{q \in P} 2\term{\frac{w(p)\term{1 + 2\gamma}}{\sum\limits_{q^\prime \in P_{y(q)}} w(q^\prime) \left( 1-2 \gamma \right)} +  \frac{w(q)}{\sum\limits_{q^\prime \in P \setminus P_{y(q)}} w(q^\prime)}} + \frac{13w(p)}{4 \sum\limits_{q \in P_{y(p)}} w(q)} \\
& \quad + \frac{\term{1 + 10\gamma} \sum\limits_{q \in P} w(q)}{\gamma^2C} \sum\limits_{q \in P} \term{w(q) \norm{U(p)}_2^2 + \frac{w(p)}{\sum\limits_{q \in P} w(q)}} \nonumber
\\
&\leq \frac{4\term{1+2\gamma}}{1-2\gamma} + \frac{\sum\limits_{q \in P_{y(p)}} w(q)}{\sum\limits_{q^\prime \in P \setminus P_{y(p)}} w(q^\prime)} + \frac{\sum\limits_{q^\prime \in P \setminus P_{y(p)}} w(q^\prime)}{\sum\limits_{q \in P_{y(p)}} w(q)} + \frac{13}{2} \label{eq:svm_sum_sens_2} \\
&+ \frac{\term{1 + 10\gamma} \sum\limits_{q \in P} w(q)}{\gamma^2\C} \cdot \term{d + 2} \nonumber
\end{align}
where~\eqref{eq:svm_sum_sens_1} holds since both arguments of the max operator at~\eqref{eq:svm_bound_sense_final} are non-negative and their sum exceeds the max among them, and~\eqref{eq:svm_sum_sens_2} holds by combining~\eqref{eq:sumWPerPy} with~\eqref{eq:svm_sum_U}

Claim~\ref{claim:svm_2} now holds by substituting $\gamma=0.4$.
\end{proof}

\setcounter{myCounter}{\getrefnumber{lem:boundingSVMs}}

\setcounter{theorem}{\getrefnumber{lem:SVM}-1}

\SVM*
\begin{proof}
First, observe that by Lemma~\ref{lem:sensBoundSVM} the total sensitivity of the query space $(P,w,\REAL^d,\fsvm)$ is bounded by $\bigO{\term{\frac{d\W}{\C} + \frac{\Tilde{\W}^2+1}{\Tilde{\W}}}}$. Let $s(p)$ be the upper bound on the sensitivity of each point $p \in P$ as in Lemma~\ref{lem:sensBoundSVM}, and let $t=\sum_{q\in P}s(q)$. Let $S$ be an i.i.d random sample of size $\bigO{\frac{1}{\eps^2}\term{\frac{d\W}{\C} + \frac{\Tilde{\W}^2+1}{\Tilde{\W}}} \term{d \log{\term{\frac{d\W}{\C} + \frac{\Tilde{\W}^2+1}{\Tilde{\W}}}} + \log{\frac{1}{\delta}}}}$, where each point $p \in P$ is sampled with probability $\frac{s(p)}{t}$, and let $v(p)=\frac{w(p)t}{s(p)|S|}$. Hence by Theorem~\ref{thm:coreset}, we get that with probability at least $1-\delta$, $(S,v)$ is an $\eps$-coreset for the query space $(P,w,\REAL^d,\fsvm)$ of size  $|S| \in \bigO{\frac{1}{\eps^2}\term{\frac{d\W}{\C} + \frac{\Tilde{\W}^2+1}{\Tilde{\W}}} \term{d \log{\term{\frac{d\W}{\C} + \frac{\Tilde{\W}^2+1}{\Tilde{\W}}}} + \log{\frac{1}{\delta}}}}$.

%for every $p \in P$ from Lemma~\ref{lem:sensBoundSVM}, $t:=t$ , $\eps:=\eps$ and $\delta:=\delta$ into Theorem~\ref{thm:coreset}, yields that with probability at least $1 - \delta$, $(S,v)$ is $\eps$-coreset  of size .
%Thus, a candidate algorithm for constructing an $\eps$-coreset for the \emph{SVM} problem is:
%\begin{itemize}
%    \item Use our sensitivity bounding scheme on a specific set of queries as discussed elaborately in the proof of Lemma~\ref{lem:sensBoundSVM}, while analytically bounding the sensitivity for the remaining sub query space.
 %   \item Plug in the bounds on the sensitivity and the total sensitivity, into Lines~\ref{algLine:5}--\ref{algLine:11} of Algorithm~\ref{alg:mainAlg}.
%\end{itemize}
\end{proof}

\subsection{Proof of Corollary~\ref{thm:restrictedLp}}

First, we provide the following definitions.
\setcounter{theorem}{29}
\begin{definition}[Induced matrix norm]
\label{def:lpInducedNormMat}
Let $z \in [1, \infty]$. Then the \emph{$\ell_z$ induced norm} for any matrix $A \in \REAL^{d \times d}$, is defined by,
\[
\norm{A}_z = \max_{\substack{x \in \REAL^d \\ \norm{x}_z = 1}} \norm{Ax}_z.
\]
\end{definition}

\begin{definition}[SVD factorization of a square matrix]
\label{def:svdFac}
Let $A \in \REAL^{d \times d}$ be matrix. The \emph{SVD factorization} of $A$ is defined to be
\[
A = \set{U} \Sigma \set{V}^T,
\]
where $\set{U} \in \REAL^{d \times d}$ is an orthogonal matrix, $\Sigma \in \REAL^{d \times d}$ is a diagonal matrix of non-negative entries in a descending order, i.e, for every $i,j \in [d]$ such that $i \leq j$, $\Sigma_{i,i} \geq  \Sigma_{j,j}$, and finally $\set{V} \in \REAL^{d \times d}$ is an orthogonal matrix.
\end{definition}

\begin{restatable}{lemma}{boundingMinOperator}
\label{lem:boundingMinOperator}
For every vector $x=(x_1,\cdots,x_d)\in\REAL^d$ there is $j\in[d]$ such that
\[
\norm{x}_1
\leq  \left(\frac{3d}{2}-1\right)\cdot \abs{x_1+x_j}.
\]
Equality holds for $x=(1,-3,\cdots,-3)$ and every $j\in[d]$, i.e., the bound is tight.
\end{restatable}

\begin{proof}
Without loss of generality, assume that $x_1\in\br{0,1}$.
Otherwise, we divide $x$ by $x_1$.%, and for (ii) the lemma is trivial.

Let $j\in\arg\max_{i\in [d]}\abs{x_1+x_i}$, $m\in\arg\max_{i\in[d]} \abs{x_i}$, $a=\max_{i\in[d]} x_i$, and $b=\max_{i\in[d]} -x_i$.
%We have
%\[
%|x_1+x_j|=|1+x_j|=
%\]
The proof is by case analysis of three cases: (i) $x_1=0$, %or $x_i\geq 0$ for every $i\in [d]$,
(ii) $x_1=1$ and $\abs{1+x_j}=1+a$, and (iii) $x_1=1$ and $\abs{1+x_j}=b-1$.

There are no other cases, since if $x_1=1$,
\[
\abs{x_1+x_j}=\abs{1+x_j}=\max\br{1+x_j,-x_j-1}=\max\br{1+a, b-1}.
\]

We observe that:
\begin{enumerate}[label=(\roman*)]
\item If $x_1 = 0$,
\[
\norm{x}_1=\sum_{i=1}^d \abs{x_i}\leq d\abs{x_{m}}= d\abs{x_1+x_{m}}\leq ((3d/2)-1)\cdot\abs{x_1+x_{m}},
\]
where the last inequality is by assumption $d\geq2$, otherwise the lemma is trivial.

\item If $\abs{1+x_j}=1+a$ and $x_1=1$, then for every $i \in [d]$, \begin{equation}\label{abbc}
\abs{x_i}=\abs{1+x_i-1}\leq \abs{1+x_i}+1\leq \abs{1+x_j}+1=2+a,
\end{equation}
where the first inequality is by the triangle inequality, and the last equality holds by the assumption of the case. Hence
\begin{equation}\label{ll}
\frac{\norm{x}_1}{\abs{x_1+x_j}}
\leq \frac{1+ (d-1)(a+2)}{1+ a}.
\end{equation}
The right hand side is decreasing with $a$ since the numerator of its derivative is
\begin{equation*}\label{lll}
(d-1)(1+a) -(1+ (d-1)(a + 2))=-d<0.
\end{equation*}
Its maximum is achieved at $a\geq x_1=1$ by the assumption $x_1=1$ of this case. By this and~\eqref{ll},
\[
\frac{\norm{x}_1}{\abs{x_1+x_j}}
\leq \frac{1+ (d-1)(a + 2)}{1+ a}
\leq \frac{1+ 3(d-1)}{2}= \frac{3d}{2}-1.
\]

\item $\abs{1+x_j}=b-1$ and $x_1=1$. For every $i\in[d]$ we thus have $\abs{x_i}\leq \abs{1+x_j}+1=b$, similarly to~\eqref{abbc}. Hence
\begin{equation}\label{ll2}
\frac{\norm{x}_1}{\abs{x_1+x_j}}
\leq \frac{1+ b(d-1)}{b-1}.
\end{equation}
The right hand side is decreasing with $b$ since the enumerator of its derivative is
\[
(d-1)(b-1)-(1+b(d-1))
=-d<0.
\]
Its maximum is achieved at $b= \abs{1+x_j}+1\geq \abs{1+x_1}+1=3$, where the first equality is by the assumptions of Case (iii).
By this and~\eqref{ll2},
\[
\frac{\norm{x}_1}{\abs{x_1+x_j}}
\leq \frac{1+b(d-1)}{b-1}
\leq \frac{1+3(d-1)}{2}
= \frac{3d}{2}-1.
\]
\end{enumerate}
\end{proof}

\begin{restatable}{claim}{matProp}
\label{clm:matProp}
Let $A \in \REAL^{d \times d}$ be an invertible matrix, and let $A = \set{U}\Sigma\set{V}$ be the \emph{SVD factorization} of $A$ (see Definition~\ref{def:svdFac}). Then for every $i \in [2,d]$,
\[
\norm{A}_2 \leq \norm{A(\set{V}_{1} + \set{V}_{i})}_2,
\]
where $\set{V}_j$ denotes the $j$th column of $\set{V}$ for every $j \in [d]$.
\end{restatable}

\begin{proof}
First, put $i \in [2,d]$, and note that by~\cite{meyer2000matrix}, we have that
\[
\norm{A}_2 = \norm{A\set{V}_{1}}_2.
\]

For every $j \in [d]$, let $e_j$ denotes the vector with a $1$ in the $j$th coordinate and $0$'s elsewhere. We observe that
\begin{align}
\label{eq:clm_1}
\norm{A\set{V}_{1} + A\set{V}_{i}}_2^2 = \norm{A\set{V}_{1}}_2^2 + 2\set{V}_{1}^T A^T A\set{V}_{i} + \norm{A \set{V}_{i}}_2^2.
\end{align}

By orthogonality of $\set{V}$ and $\set{U}$,
\begin{align}
\label{eq:clm_2}
\set{V}_{1}^T A^T A\set{V}_{i} = \set{V}_{1}^T \set{V} \Sigma^T \set{U}^T \set{U} \Sigma \set{V}^T \set{V}_{i} = \set{V}_{1}^T \set{V} \Sigma^T \Sigma \set{V}^T \set{V}_{i} = e_1^T \Sigma^T \Sigma e_i = \Sigma_{1,1} \Sigma_{i,i} e_1 e_i^T = 0,
\end{align}
where the first equality holds by Definition~\ref{def:svdFac}, the second equality is by orthogonality of $\set{U}$, the third equality is by orthogonality of $\set{V}$, the forth equality holds since $\Sigma$ is a diagonal matrix and the last equality holds by definition of $e_j$ for every $j \in [d]$.

Combining~\eqref{eq:clm_1} and~\eqref{eq:clm_2}, yields that
\begin{align*}
\norm{A\set{V}_{1} + A\set{V}_{i}}_2 = \sqrt{\norm{A\set{V}_{1}}_2^2 + \norm{A\set{V}_{i}}_2^2} \geq \norm{A\set{V}_{1}}_2 = \norm{A}_2.
\end{align*}
\end{proof}

\begin{lemma}
\label{lem:sensBoundRestrictedLp}
Let $(P,w, \REAL^d, \fmestimator)$ be a query space as in Definition~\ref{def:querySpace}, such that for every $x \in \REAL^d$, and $p \in P$, the loss function $\fmestimator$ is defined to be
\[
\fmestimator(p,x) = \min\br{\abs{p^T x}, \norm{x}_z}.
\]

Let $\fmestimator[g] \in \set{F}$ such that for every $x \in \REAL^d$ and $p \in P$, $\fmestimator[g](p,x) = \abs{p^Tx}$. Let $(U,D,V)$ be the $F$-SVD of $P$ with respect to $\fmestimator[g]$. Let $\gamma  = \maxArgs{1, \frac{2\pi d^{\abs{\frac{1}{2} - \frac{1}{z}}}}{\norm{DV^T}_2}}$. Then claims (i) -- (ii) hold as follows:
\begin{enumerate}[label=(\roman*)]
\item \label{case:lzres_1} For every $p\in P$, its sensitivity with respect to the query space $(P,w, \REAL^d, \fmestimator)$ is bounded by
\[
s(p) = w(p)\min\br{\norm{U(p)}_2, d^{\abs{\frac{1}{2} - \frac{1}{z}}}\norm{\term{DV^T}^{-1}}_2},
\]
\item \label{case:lzres_2} and the total sensitivity is bounded by
\[
\sum\limits_{p \in P} s(p) \leq 4\gamma d^{2+\abs{\frac{1}{2} - \frac{1}{z}}}.
\]
\end{enumerate}
\end{lemma}

\begin{proof}
First, we observe that the level set $\set{X}_{\fmestimator[g]}$ (see Definition~\ref{def:familyOfConvexFuncs}) is contained in the level set $L = \br{x \middle| x \in \REAL^d, \sum\limits_{p \in P} w(p)\fmestimator[f](p,x) \leq 1}$. By Theorem~\rom{3} of~\cite{john2014extremum}, the L\"{o}wner ellipsoid which contains the level set $\set{X}_{\fmestimator[g]}$ will also contain the level set $L$, when setting the dilation factor, i.e., $\alpha$ to
$\gamma \sqrt{d}$. In other words,
\[
\frac{1}{\sqrt{d}} E \subseteq \set{X}_{\fmestimator[g]} \subseteq L \subseteq \sqrt{d} \gamma E,
\]
where $E$ denotes the L\"{o}wner ellipsoid of the level set $\set{X}_{\fmestimator[g]}$. Since $L$ is contained in the ellipsoid $\sqrt{d} \gamma E$, and contains the ellipsoid $\frac{1}{\sqrt{d}} E$, using similar arguments to those established at the proof of Lemma~\ref{lem:LownerToF}, we obtain that there exists a diagonal matrix $D \in \REAL^{d \times d}$ and an orthogonal matrix $V \in \REAL^{d \times d}$ such that for every $x \in \REAL^d$,
\begin{equation}
\label{eq:semiLowner}
\norm{D^\prime V^T x}_2 \leq \sum\limits_{q \in P} w(q) \fmestimator[f](q,x) \leq \gamma \sqrt{d} \norm{D^\prime {V}^T x}_2,
\end{equation}
where $D^\prime := \frac{1}{2\gamma} D$.

With this, we proceed to bound the sensitivity of each point $p \in P$.

\paragraph{Proof of Claim~\ref{case:lzres_1}.}
Put $p \in P$, and let $U(q) := \term{V D^\prime}^{-1}q $ for every $q \in P$. Observe that,
\begin{equation}
\label{eq:restrictedLp_1}
\begin{split}
\sup_{\substack{x \in \REAL^d\\ \fmestimator(p,x) > 0}} \frac{w(p)\fmestimator(p,x)}{\sum\limits_{q \in P} w(q)\fmestimator(q,x)}
&\leq \sup_{\substack{x \in \REAL^d, \\ \fmestimator(p,x) > 0}} \frac{w(p)\fmestimator(p,x)}{\norm{D^\prime{V}^Tx}_2} \\
&= \sup_{\substack{x \in \REAL^d, \\ \fmestimator(p,x) > 0}} w(p)\min\br{\frac{\abs{U(p)^T D^\prime V^T x}}{\norm{D^\prime {V}^Tx}_2}, \frac{\norm{x}_z}{\norm{D^\prime {V}^T x}_2}} \\
&\leq w(p)\min\br{\norm{U(p)}_2, d^{\abs{\frac{1}{2} - \frac{1}{z}}}\norm{\term{D^\prime {V}^T}^{-1}}_2}
\end{split}
\end{equation}
where the first inequality is by~\eqref{eq:semiLowner}, the equality is by definition of $\fmestimator$, and the last inequality follows from combining Lemma~\ref{thm:equivalenceNorms} with the fact that $\frac{D^\prime {V^\prime}^Tx}{\norm{D^\prime {V^\prime}^Tx}_2}$ is a unit vector and $\term{D^\prime V^T}^{-1} D^\prime V^T = I_d$.

\paragraph{Proof of Claim~\ref{case:lzres_2}.}
In order to bound the total sensitivity, we first let $\beta_z = d^{\abs{\frac{1}{2} - \frac{1}{z}}}$, $M \in \REAL^{d \times d}$ be an orthogonal matrix that corresponds to the matrix $\set{V}$ of the \emph{SVD} factorization of $\term{D^\prime V^T}^{-1}$ (See Definition~\ref{def:svdFac}), and let $M_i$ denote the $i$th column of $M$ for every $i \in [d]$. Thus,
\begin{equation}\label{eq:restrictedLp_2}
\begin{split}
\min\br{\norm{U(p)}_2,\beta_z \norm{\term{D^\prime V^T}^{-1}}_2} &=\min\br{\norm{MU(p)}_2, \beta_z \norm{\term{D^\prime V^T}^{-1 }M_{1*}}_2} \\
&\leq \min\br{\norm{MU(p)}_1, \beta_z \norm{\term{D^\prime V^T}^{-1}M_{1*}}_2}\\
&\leq \min\br{2d\abs{U(p)^T(M_{1}+M_{j})},\beta_z\norm{\term{D^\prime V^T}^{-1}(M_{1} + M_{j})}_2} \\ &\leq 2d \min\br{\abs{U(p)^T(M_{1}+M_{j})},\norm{\term{D^\prime V^T}^{-1}(M_{1} + M_{j})}_2}  ,
\end{split}
\end{equation}
where the equality holds by definition of $M$, the first inequality holds by Lemma~\ref{thm:equivalenceNorms}, the second inequality holds by Lemma~\ref{lem:boundingMinOperator} and by Claim~\ref{clm:matProp}, and the last inequality follows from the fact that $\beta_z \leq 2d$.

By combining~\eqref{eq:restrictedLp_1} and \eqref{eq:restrictedLp_2}, we have that
\begin{equation}
\label{eq:restrictedLp_3}
\begin{split}
s(p) \leq 2d^{1 + \abs{\frac{1}{2} - \frac{1}{z}}}w(p)\sum\limits_{j=1}^d \min\br{\abs{U(p)^T(M_{*1}+M_{*j})},\norm{\term{D^\prime V^T}^{-1}(M_{*1} + M_{*j})}_z},
\end{split}
\end{equation}
where the inequality follows from invoking Lemma~\ref{thm:equivalenceNorms}.

Summing~\eqref{eq:restrictedLp_3} over every $p \in P$, we obtain that
\[
\sum\limits_{p \in P} s(p) \leq 2\gamma d^{1 + \abs{\frac{1}{2} - \frac{1}{z}}} \sum_{j=1}^d \norm{M_{1} + M_{j}}_2 \leq 4\gamma d^{2+\abs{\frac{1}{2} - \frac{1}{z}}},
\]
where the first inequality is by~\eqref{eq:semiLowner} and the second inequality holds since $\norm{x + y}_2 \leq 2$ for any pair of unit vectors $x,y \in \REAL^d$.
\end{proof}

\setcounter{theorem}{\getrefnumber{thm:restrictedLp}-1}
\setcounter{myCounter}{34}

\restrictedLp*
\begin{proof}
First, observe that by Lemma~\ref{lem:sensBoundRestrictedLp} the total sensitivity of the query space $(P,w,\REAL^d,\fmestimator)$ is bounded by $\bigO{\gamma d^{2+\abs{\frac{1}{2} - \frac{1}{z}}}}$. Let $s(p)$ be the upper bound on the sensitivity of each point $p \in P$ as in Lemma~\ref{lem:sensBoundRestrictedLp}, and let $t=\sum_{q\in P}s(q)$. Let $S$ be an i.i.d random sample of size $\bigO{\frac{\gamma d^{2+\abs{\frac{1}{2} - \frac{1}{z}}}}{\eps^2} \term{d \log{\term{\gamma d^{2+\abs{\frac{1}{2} - \frac{1}{z}}}}} + \log{\frac{1}{\delta}}}}$, where each point is sampled with probability $\frac{s(p)}{t}$, and let $v(P)=\frac{w(p)t}{s(p)|S|}$. Hence by Theorem~\ref{thm:coreset}, we get that with probability at least $1-\delta$, $(S,v)$ is an $\eps$-coreset for the query space $(P,w,\REAL^d,\fmestimator)$.

\end{proof}

\section{\say{Easy} examples covered by our framework}
\label{sec:easy}

\subsection{$\ell_z$-Regression for $z \in [1,\infty)$}

\setcounter{theorem}{\themyCounter}

\begin{lemma}
\label{lem:sensBoundLPReg}
Let $z\in[1,\infty)$, $(P,w,\REAL^d,\flz)$ be a query space, such that fo every $x \in \REAL^d$ and $p \in P$  the loss function $\flz:P\times \REAL^d \to [0,\infty)$ is defined to be
$\flz(p,x) = \abs{p^Tx}^z.$
Let $(U,D,V)$ be the $f$-SVD of $(P,w)$ with respect to $\flz$ (see Definition~\ref{def:svdWRTF}). Then, claims (i) -- (ii) hold as follows:

\begin{enumerate}[label=(\roman*)]
\item \label{claim:lzreg1}for every $p\in P$, the sensitivity of $p$ with respect to the query space $(P,w,\REAL^d,\flz)$ is bounded by
\[
s(p) \leq
\begin{cases}
w(p)\norm{U(p)}_z^z & z \in [1,2] \\
\sqrt{d^z} w(p)\norm{U(p)}_z^z & \text{otherwise}
\end{cases} ,
\]

\item \label{claim:lzreg2}and the total sensitivity is bounded by
\[
\sum\limits_{p \in P} s(p) \leq
\begin{cases}
d^{\frac{z}{2} + 1} & z \in [1,2) \\
d & z = 2 \\
d^{z+1} & \text{otherwise}
\end{cases}.
\]
\end{enumerate}
\end{lemma}

\begin{proof}

Note the following:
\begin{enumerate}[label=(\alph*)]
    \item For every $q \in P$, $ x \in \Q$ and $b \geq  0$ we have $\abs{q^T b x} = b \abs{q^T x}$.
    \item Since $\abs{q^T x}$ is convex function, it also holds that $\sum\limits_{q \in P} w(q) \abs{q^T x}^z$ is convex due to the fact that sum of convex functions is also convex,
    \item The level set $\br{x \middle| x \in \Q, \sum\limits_{q \in P} w(q) \abs{q^Tx}^z \leq 1}$ is convex and is centrally symmetric.
    \item For any unit vector $x \in \REAL^d$ and $q \in P$,
\[
\abs{U(q)^T x}^z \leq \norm{U(q)}_2^z \norm{x}_{2}  = \norm{U(q)}_2^z \leq \begin{cases}
\norm{U(q)}_z^z & z \in [1, 2]\\
d^{\frac{1}{2} - \frac{1}{z}} \norm{U(q)}_z^z & z > 2
\end{cases},
\]
where the first inequality holds by Cauchy Schwartz's inequality, the equality is by the assumption that $x$ is a unit vector, and the last inequality holds by plugging $a := 2$ and $b := z$ for $z > 2$ and $a := z$ and $b := 2$ for $z \in [1,2]$ into Lemma~\ref{thm:equivalenceNorms}.
\end{enumerate}

Hence, plugging
\begin{itemize}
    \item $f(p,x) := \flz[f](p,x)$, $g(p,x):= \flz[f](p,x)$, $h(p,x) := 0$ for every $p \in P$, and $x \in \Q$,
    \item $c_i := 1$ for every $i \in [5]$,
    \item $z := z$,
    \item $\alpha := \sqrt{d}$ for $z \neq 2$ and $\alpha := 1$ for $z = 2$,
    \item $v_i := e_i$ where $e_i$ denotes a vector which at its $i$th entry there is $1$, and $0$'s elsewhere,
    \item and $c := \begin{cases}
    1 & z \in [1,2] \\
    d^{\frac{1}{2} - \frac{1}{z}} & z > 2
    \end{cases}$,
\end{itemize}
into Lemma~\ref{thm:sensitivityBound}, yields that
\[
s(p) = w(p)\sum\limits_{i \in [d]} \abs{U(p)^T D V^T e_i}^z \cdot \begin{cases}
1 & z \in [1,2] \\
d^{\frac{1}{2} - \frac{1}{z}} & z > 2
\end{cases}
\]

This satisfies (\romannumeral 1) as
\[
\sum\limits_{i \in [d]}\abs{U(q)^T e_i}^z = \norm{U(q)}_z^z,
\]
holds for every $q \in P$ by definition of norms.

As for the sum of sensitivities, Claim (\romannumeral 2) follows from Lemma~\ref{thm:sensitivityBound}.
\end{proof}

\begin{restatable}{corollary}{LPReg}
\label{coro:LPReg}
Let $(P,w, \REAL^d, \flz)$ be a query space, such that for every $x \in \REAL^d$, and $p \in P$, the loss function $\fnclz$ is defined to be
\[
\flz(p,x) = \abs{p^Tx}^z.
\]

Let $\eps, \delta \in (0,1)$, and let $(S,v)$ be the output of a call to $\coreset\term{P,w,\flz, \eps, \delta}$. Then, with probability at least $1 - \delta$, $(S,v)$ is an $\eps$-coreset for the query space $(P,w, \REAL^d, \flz)$, and the size of the coreset is
\[
\abs{S}\in \begin{cases}
\bigO{\frac{d^{\frac{z}{2} + 1}}{\eps^2} \term{d \log{\term{d^{\frac{z}{2} + 1}}} + \log{\term{\frac{1}{\delta}}}}} & z \in [1, 2)\\
\bigO{\frac{d}{\eps^2} \term{d \log{\term{d}} + \log{\term{\frac{1}{\delta}}}}} & z = 2\\
\bigO{\frac{d^{z + 1}}{\eps^2} \term{d \log{\term{d^{z + 1}}} + \log{\term{\frac{1}{\delta}}}}} & z \in (2, \infty)
\end{cases}.
\]
\end{restatable}
\begin{proof}
First, observe that by Lemma~\ref{lem:sensBoundLPReg}, the total sensitivity is bounded by $t := \begin{cases}
d^{\frac{z}{2} + 1} & z \in [1,2) \\
d & z = 2 \\
d^{z+1} & \text{otherwise}
\end{cases}
$.

Plugging $s(p)$ for every $p \in P$ from Lemma~\ref{lem:sensBoundLPReg}, $t:=t$, $\eps:=\eps$ and $\delta:=\delta$ into Theorem~\ref{thm:runTime}, yields that with probability at least $1 - \delta$, $(S,v)$ is an $\eps$-coreset of size $\bigO{\frac{t}{\eps^2} \term{d \log{\term{t}} + \log{\term{\frac{1}{\delta}}}}}$.% Such coreset can be constructed by Algorithm~\ref{alg:mainAlg} using the variables defined in the proof of Lemma~\ref{lem:sensBoundLPReg}.
\end{proof}

\subsection{Least squared errors}
\begin{restatable}{lemma}{lse}
\label{lem:sensBoundLSE}
Let $(P,w, \REAL^d, \flse)$ be a query space, such that for every $x \in \REAL^d$ and $p \in P$, the loss function $\flse$ is defined to be $\flse(p,x) = \norm{p - x}_2^2.$
Let $P^\prime = \br{ p^\prime = \begin{bmatrix} \norm{p}_2^2 \\ 2p \\ 1 \end{bmatrix} \middle| p \in P}$ and let $\flse[g]: P'\times \REAL^{d+2} \to [0,\infty)$ such that for every $y \in \REAL^{d+2}$ and $p \in P^\prime$, $\flse[g](p,y) = \abs{p^Ty}$. Let $(U,D,V)$ be the $f$-SVD of $P'$ with respect to $\flse[g]$. Then, claims (i) -- (ii) hold as follows:
\begin{enumerate}[label=(\roman*)]
\item for every $p \in P$, the sensitivity of $p$ with respect to the query space $(P,w, \REAL^{d},\flse)$ is bounded by $s(p) \leq  w(p) \norm{U(p^\prime)}_1,$
\item and the total sensitivity is bounded by  $\sum\limits_{p \in P} s(p) \in \bigO{d^{1.5}}.$
\end{enumerate}
\end{restatable}

\begin{proof}
Put $p \in P$, and observe that for every $y \in \REAL^d$, $\norm{p - y}_2^2 = \norm{p}_2^2 - 2p^Ty + \norm{y}_2$, which enables us to rewrite the problem by reformulating the query space and the input space ($\REAL^d$ and $P$ respectively). Let $X' = \br{\begin{bmatrix} 1 \\ -x \\ \norm{x}_2^2 \end{bmatrix} \middle| x \in \Q}$. Then, we obtain that for every $x \in X^\prime$
\[
\frac{w(p) \flse(p,x)}{\sum\limits_{q \in P} w(q) \flse(q,x)} = \frac{w(p)\abs{{p^\prime}^T x}}{\sum\limits_{q \in P} w(q) \abs{{q^\prime}^T y}} \leq \sup_{y \in \REAL^{d+2}} \frac{w(p) \abs{{p^\prime}^T y}}{\sum\limits_{q \in P} w(q) \abs{{q^\prime}^Ty}},
\]
where the second inequality is by rewriting the cost function and setting $y \in X'$ and the last inequality follows from $\sup$ operator.

Finally, the upper bound on the sensitivity of each point $p \in P$ and an upper bound on the total sensitivity follows from plugging $P'$, $\REAL^{d+2}$ as the query space, and $z:=1$ into Corollary~\ref{lem:sensBoundLPReg}.
\end{proof}

\begin{corollary}
Let $(P,w, \REAL^d, \flse)$ be a query space, such that for every $x \in \REAL^d$, and $p \in P$, the loss function $\flse$ is defined to be
\[
\flse(p,x) = \norm{p - x}_2^2.
\]

Let $P^\prime = \br{ \begin{bmatrix} p^\prime = \norm{p}_2^2 \\ 2p \\ 1 \end{bmatrix}  \middle| p \in P}$ and let $\flse[g]: P'\times \REAL^{d+2} \to [0,\infty)$ such that for every $x \in \REAL^{d+2}$ and $p \in P^\prime$, $\flse[g](p,x) = \abs{p^Tx}$. For every $p\in P$ and $p' = (\norm{p}^2_2 \mid 2p \mid 1)$ we define $w'(p') = w(p)$. Let $\eps, \delta \in (0,1)$, and let $(S',v')$ be a coreset for the query space $\term{P^\prime, w^\prime, \REAL^{d+2}, \flse[g]}$ by Corollary~\ref{coro:LPReg}. Let $S = \br{p \mid (\norm{p}^2_2 \mid 2p \mid 1)\in S'}$, and for every $p\in S$, and $p'= (\norm{p}^2_2 \mid 2p \mid 1)\in S'$ let $v(p) = v'(p')$. Then, with probability at least $1 - \delta$, $(S,v)$ is an $\eps$-coreset for the query space $(P,w, \REAL^d, \flse)$, and the size of the coreset is $\abs{S}\in \bigO{\frac{\term{d+2}^{2.5}}{\eps^2} \term{d \log{\term{\term{d+2}^{2.5}}} + \log{\term{\frac{1}{\delta}}}}}$.
\end{corollary}

\begin{proof}
First, observe that by Lemma~\ref{lem:sensBoundLSE}, the total sensitivity is bounded by $t := \term{d+2}^{1.5}$ of $\term{P^\prime, w^\prime, \REAL^{d+2}, \flse[g]}$. Plugging $P:=P^\prime$, $t:=t$ , $\eps:=\eps$ and $\delta:=\delta$ into Corollary~\ref{coro:LPReg}, yields that $S',v'$ is an $\eps$-coreset of size $\bigO{\frac{\term{d+2}^{2.5}}{\eps^2} \term{d \log{\term{\term{d+2}^{2.5}}} + \log{\term{\frac{1}{\delta}}}}}$ for the query space $\term{P^\prime, w^\prime, \REAL^{d+2}, \flse[g]}$.

By construction of $P^\prime$, it holds that for every $p^\prime \in S^\prime$ and $x^\prime = \begin{bmatrix} 1\\ -x \\ \norm{x}_2^2\end{bmatrix}$ where $x \in \REAL^d$,
\[
v\term{p^\prime}\flse[g](p^\prime,x^\prime) = v\term{p^\prime} \abs{p^\prime x^\prime} = v\term{p^\prime} \norm{p - x}_2^2 = v(p) \norm{p - x}_2^2 = v(p) \flse(p,x).
\]

Thus we obtain that for every $x \in \REAL^d$
\[
\abs{\sum\limits_{p \in P} w(p)\norm{p - x}_2^2 - \sum\limits_{p \in S} w(p)\norm{p - x}_2^2} \leq \eps \sum\limits_{p \in P} w(p)\norm{p - x}_2^2,
\]
hold with probability at least $1 - \delta$, i.e., $(S,v)$ is an
$\eps$-coreset for the query space $\term{P, w, \Q, \flse}$.
\end{proof}

\section{Experimental setup}
\label{sec:sup_experimental}
% \paragraph{Datasets.} The following datasets were used for our experiments and were downloaded from UCI machine learning repository~\cite{Dua:2019}:
% \begin{enumerate*}[label=(\roman*)]
%     \item \textit{HTRU~\cite{Dua:2019}} --- $17,898$ radio emissions of the Pulsar star each consisting of $9$ features.
%     \item \textit{Skin~\cite{Dua:2019}} --- $245,057$ random samples of R,G,B from face images consisting of $4$ dimensions.
%     \item \textit{Cod-rna~\cite{uzilov2006detection}} --- consists of $59,535$ samples, $8$ features, which has two classes (i.e. labels), describing RNAs.
%     \item \textit{Web dataset~\cite{CC01a}} -- $49,749$ web pages records where each record is consists of $300$ features.
%     \item \textit{3D spatial networks~\cite{Dua:2019}} -- 3D road network with highly accurate elevation information (+-20cm) from Denmark used in eco-routing and fuel/Co2-estimation routing algorithms consisting of $434,874$ records where each record has $4$ features.
% \end{enumerate*}

\paragraph{Preprocessing step.} We applied a standardization step, i.e., each input point has zero mean and unit variance. In addition, specifically for the problem of \emph{SVM} and \emph{Logistic regression}, the points were normalized such that the maximal norm of a point in the dataset will be $1$.

\paragraph{Faster algorithms for computing the $f$-SVD} Problems which can be reduced to the $\ell_2$-regression problem, are easier to deal with, since the $f$-SVD can be computed using the SVD factorization which is can be computed in $O\term{n^2d}$, e.g., we showed that both logistic regression and SVM can be reduced to $\ell_2$-regression as discussed in Lemma~\ref{lem:boundingLog2} and Lemma~\ref{lem:boundingSVMs}.

As for our aforementioned problems, we shown a reduction to $\ell_1$ regression, which using~\cite{clarkson2017low}, we can compute the $f$-SVD in roughly $O\left(nd + poly(d) \right)$ time (worst case scenario).

Note that~\cite{clarkson2017low} can accelerate the computation time of the $f$-SVD if the problem can be reduced to $\ell_z$ regression for any $z \geq 1$, due to the fact that it computes an approximated L\"{o}wner ellipsoid using randomized algorithm. For other problems, the time needed for computing the $f$-SVD is mentioned at Theorem~\ref{thm:runTime}.

\end{document}